\documentclass[lettersize,journal]{IEEEtran}
\usepackage{amsmath,amsfonts}
\usepackage{algorithmic}
\usepackage{algorithm}
\usepackage{array}
\ifCLASSOPTIONcompsoc
\usepackage[caption=false, font=normalsize, labelfont=sf, textfont=sf]{subfig}
\else
\usepackage[caption=false, font=footnotesize]{subfig}
\usepackage{textcomp}
\usepackage{stfloats}
\usepackage{url}
\usepackage{verbatim}
\usepackage{graphicx}
\usepackage{cite}
\usepackage{indentfirst}
\usepackage{multirow}
\usepackage{makecell}
\usepackage[table,xcdraw]{xcolor}
\usepackage{cite}
\usepackage{amsmath}
\usepackage{amsthm}
\newtheorem{theorem}{Theorem}

\newtheorem{remark}{Remark}

\begin{document}

\title{A Graph Transformer-Driven Approach for Network Robustness Learning}

\author{Yu Zhang, Jia Li, Jie Ding,~\IEEEmembership{Member,~IEEE,} Xiang Li,~\IEEEmembership{Senior Member,~IEEE}
\thanks{This work was supported by National Natural Science Foundation of China under Grants 62103107, National Key Research and Development
Program of China under Grant 2021YFE019330, and Shanghai Sailing Program under Grant 21YF1403000.}
\thanks{Yu Zhang, Jia Li, and Jie Ding are with the Adaptive Networks and Control Lab, Department of Electronic Engineering, School of Information Science and Engineering, Fudan University, Shanghai 200433, China (e-mail: dingjie@fudan.edu.cn).}
\thanks{X. Li is with the Institute of Complex Networks and Intelligent Systems, Shanghai Research
Institute for Intelligent Autonomous Systems, Tongji University, Shanghai 201210, China. (e-mail: lix2021@tongji.edu.cn.)}
}

\markboth{}
{Shell \MakeLowercase{\textit{et al.}}: A Sample Article Using IEEEtran.cls for IEEE Journals}

\maketitle

\begin{abstract}
Learning and analysis of network robustness, including controllability robustness and connectivity robustness, is critical for various networked systems against attacks. Traditionally, network robustness is determined by attack simulations, which is very time-consuming and even incapable for large-scale networks. Network Robustness Learning, which is dedicated to learning network robustness with high precision and high speed, provides a powerful tool to analyze network robustness by replacing simulations. In this paper, a novel versatile and unified robustness learning approach based on the customized graph transformer (NRL-GT) is proposed,  which consists of a particularly designed backbone and three branches to accomplish the tasks of robustness curve learning, overall robustness learning and synthetic network classification simultaneously. Numerous experiments show that: 
1) NRL-GT is a unified learning framework for controllability robustness and connectivity robustness, demonstrating a strong generalization ability to ensure high precision when training and test sets are distributed differently;
2) It is theoretically and experimentally demonstrated the proposed graph transformer layer is capable of outperforming classical graph neural networks. Compared to the cutting-edge methods in network robustness learning, NRL-GT can simultaneously perform network robustness learning from multiple aspects and obtains superior results in less time on both synthetic and real-world networks, especially circuit networks and power networks; 
3) It is worth mentioning that with the transferability of the proposed backbone, NRL-GT is able to deal with complex networks of different sizes for different tasks efficiently and effectively.
\end{abstract}


\section{Introduction}
\IEEEPARstart{R}EASEARCH related to controllability and connectivity as well as the robustness of complex networks has been widely applied in natural and engineering systems\cite{1,2}. Specifically, network controllability\cite{8,9,11,12,59,95}, which refers to the ability of a network to be driven from any initial state to any target state with permissible control inputs in a finite amount of time\cite{67}, is receiving great attention over the past decade\cite{89,90,91,92}. Connectivity, as one of the most important metrics to measure network performance, has a fundamental impact on network controllability\cite{67}. Recent studies also show that degree distribution has a significant impact on both network controllability and network connectivity. 
In \cite{8}, the driver nodes are found not to prefer high-degree nodes for either synthetic networks or real systems. In \cite{22}, it is observed that if the minimum in-degree and out-degree of a random network are both greater than two, then it can be controlled by an extremely small number of driving nodes. In \cite{23}, it shows that degree correlation instead of clustering and modularity has a significant effect on the controllability of a network. For complex networks with a power-law degree distribution, hubs play a critical role in maintaining the connectivity of networks. Removing several hubs can break a network into small disconnected parts\cite{69}.

The robustness of complex networks refers to the ability of complex networks to maintain controllability and connectivity in response to attacks, including network controllability robustness\cite{29,24,70} and connectivity robustness\cite{71,72}. With the growing ubiquity of sophisticated attacks on circuit networks\cite{99}, power networks\cite{100}, internet\cite{3}, and so on, robustness analysis is becoming increasingly important in real applications. For a specific example, it helps to improve the power grid robustness against cascading failures\cite{97,98}. The main strategies to attack a network include malicious attacks and random attacks\cite{25,26}. Whether malicious or random attacks, they occur in the form of node- or edge-removals, causing significant damage to the networks\cite{105}. Under continuous node- or edge-removals, controllability robustness, and connectivity robustness can be measured by the density variation of driver nodes\cite{24}, and the variation of the proportion of nodes in the largest connected component (LCC)\cite{73}, respectively. Typical malicious attacks are targeted, namely purposefully attacking nodes or edges with high degrees or high betweenness\cite{105}. In \cite{30}, it is revealed that a network shows the best robustness to random attacks when its in-degree and out-degree tend to be homogeneous.

Attack simulation is a general method for determining network robustness. However, it has the disadvantage of being severely time-consuming. Therefore, the research of \emph{Network Robustness Learning} including \emph{Network Controllability Robustness Learning} and \emph{Network Connectivity Robustness Learning} receives increasing attention, aiming to design a general learning framework for realizing high-speed and high-precision learning of network controllability robustness and connectivity robustness. Our 
task has three particular challenges compared to traditional prediction tasks 
using deep learning methods, such as time series 
forecasting. The first challenge is how to generate 
robustness-related dense feature matrices from the provided topology of complex 
networks. Secondly, due to the wide range of topological variations in complex 
networks, our prediction requires high precision and high generalization  
simultaneously. Thirdly, we need to design a unified framework to accomplish 
network robustness learning from multiple perspectives. In \cite{31,32}, a single CNN-based learning model (PCR) and a composite model integrating multiple CNN modules (iPCR) are proposed, which achieve high-precision controllability robustness learning when the training and test sets maintain the same distribution. In iPCR, prior knowledge of the topology type of the network is used to improve the learning precision, while this improvement relies on a complex integrated model, which is more difficult to train and more time-consuming compared to PCR. There are two analogous approaches in the field of connectivity robustness learning: CNN-RP~\cite{68} and mCNN-RP~\cite{77}, which demonstrate the effectiveness of CNN-based models for learning connectivity robustness. In \cite{33}, LFR-CNN is proposed for general robustness learning, which incorporates CNNs and PATCHY-SAN~\cite{50}. Due to the embedding mechanism of PATCHY-SAN, LFR-CNN can handle cases when the network size slightly varies. Although LFR-CNN has superior performance in networks with random size and random average degree, it costs 20 times more time than PCR for the same input. Moreover, the generalization ability is not reflected in LFR-CNN when the training and test sets are inconsistent, which is a common shortcoming of CNN-based models~\cite{33}. Therefore, a unified framework for high-speed and high-generalization robustness learning requires further investigation.

Complex networks are graph-structured data, which cannot be directly 
manipulated by traditional deep neural models. For previous CNN-based models, the adjacency matrix of a network is treated as an image, which restricts the 
ability of obtaining the neighborhood information for each node during the
convolution operation. In recent years, there has been a great interest in graph neural networks (GNNs) which are shown to be efficient and effective on graph datasets and are widely applied in various fields~\cite{107,108,109,110}, such as natural language processing (NLP)~\cite{78}, recommendation systems~\cite{79}, computer vision~\cite{43}, and so on. Different from CNN-based methods, GNNs accurately obtain the neighborhood information of each node in the process of computing. The birth of the Transformer~\cite{80} has made a great contribution to the development of NLP, which is currently still cutting-edge architecture for handling long-term sequential data. Since the embedding generation mechanism of the transformer is considered equivalent to that of GNN applied to fully connected graphs in terms of message passing mechanism~\cite{81}, more and more research tends to combine the transformer with GNN to build a more powerful embedding generation algorithm, which is referred to as Graph Transformer~\cite{101,102,103}.

In this work, we propose an improved version of Graph Transformer by incorporating the properties of network robustness and introduce an efficient and effective scheme for \emph{Network Robustness Learning} via the proposed Graph Transformer (NRL-GT).  As a member of GNNs, the proposed Graph Transformer of NRL-GT is capable of 
accurately capturing the information about neighboring nodes compared to 
CNN-based models, thus better mining the topology information of complex 
networks. More critically, given the  powerful inductive learning 
capability~\cite{45} on complex networks of GNNs that CNNs do not have, NRL-GT can 
generate high-quality 
complex network embeddings even for inconsistent distributions between the 
training and test sets, resulting in strong generalization performance. Compared to the classic transformer~\cite{80} and Graph Transformer~\cite{101,102,103}, our customized Graph Transformer fuses two types of multi-head attention mechanisms: inner head and outer head, implying its stronger feature mining capabilities. Meanwhile, it encodes the edge features and the dependencies between nodes and connecting edges in the attention mechanism, which motivates the design of inner head to outperform the learning ability of classical GNNs~\cite{45,46,47}. Besides, compared to classical Transformer-based models, our proposed graph transformer restricts nodes to attend only to their local neighbors in the graph, which adds a good exploitation of the sparse 
message passing process to help inductive learning on graphs~\cite{103}. 
Moreover, NRL-GT is the first framework that accomplishes the task of robustness learning in three aspects, which can simultaneously conduct robustness curve learning, overall robustness learning, and synthetic network classification. Specifically, there is a backbone and three downstream modules in NRL-GT. The backbone is responsible for extracting features related to the robustness of complex networks of any size, while the three downstream modules are respectively responsible for the aforementioned tasks. The main contributions of this work are summarized as follows.
\vspace{-0.4cm}
\begin{itemize}
\item
We propose an improved version of Graph Transformer by introducing a multi-head attention mechanism with a combination of inner and outer heads. It has been experimentally demonstrated to obtain higher-quality graph representations than CNNs, classical GNNs, and classical transformer-base methods. Based on this, a novel unified learning framework for network controllability robustness and connectivity robustness is proposed named NRL-GT. With a carefully designed backbone and three branches, NRL-GT is capable of accomplishing the tasks of robustness curve learning, overall robustness learning, and synthetic network classification simultaneously. 

\item  

Having been experimentally verified, NRL-GT overcomes the issue that the training and test sets are with different distributions, such as unweighted networks for training and weighted networks for testing as well as directed networks for training and undirected networks for testing, demonstrating the strong generalization ability of NRL-GT.
\item
It is worth mentioning that NRL-GT has significant advantages in precision and speed compared with leading-edge methods in \cite{31,32,68,77}, both in synthetic and real-world networks, especially circuit networks and power networks, and provides a better indicator for connectivity robustness compared to spectral measures.

\item
The backbone of NRL-GT has strong transferability. To be specific, the backbone trained on the task of robustness curve learning can be directly transferred to the task of overall robustness learning and synthetic network classification without retraining and the backbone pre-trained on networks with nodes $N = 1000$
can directly extract the features of complex networks of other sizes without retraining.
\end{itemize}


The rest of this paper is as follows. Section II introduces the preliminary knowledge related to network robustness and network robustness learning. The specific structure of NRL-GT is described in Section III. In Section IV, various experiments have been conducted to evaluate the different functions of NRL-GT. Section V summarizes this work.
\section{Robustness analysis of complex networks}
Basic concepts and measures of controllability robustness and connectivity robustness of complex networks are first introduced in this section. Then, the overall network robustness is presented as an extension of the robustness measure. Finally,  error measures and evaluation metrics for network robustness learning are demonstrated in the last part.
\vspace{-0.4cm}
\subsection{Controllability Robustness}
In terms of a linear time-invariant networked system  with $N$ nodes ${\dot x_t} = A{x_t} + B{u_t}$ , where $A$ and $B$ are the transposed adjacency matrix and input matrix, and ${{x_t}}$ and ${{u_t}}$ are the state vector and input vector. The full rank of the controllability matrix $[B, AB,{A^2}B,...,{A^{N - 1}}B]$ is a necessary and sufficient condition for state controllability, where $N$ represents the dimension of $A$.  Based on the framework of exact controllabilty~\cite{9}, the minimum number of driver nodes $N_D$ can be determined by ${N_D} = \max \{ N - rank(A),1\} $. For  structural controllability where  $A$ and $B$ are treated as free parameters,   ${N_D} = \max \{ N - \left| {{M^*}} \right|,1\} $, where $\left| {{M^*}} \right|$  is the number of elements in the  maximum matching ${{M^*}}$ according to minimum input theorem~\cite{8}.   

The controllability robustness can be measured by the controllability curve. \emph{Network Controllability Robustness Learning} is to predict such a curve through the designed model. For node attacks considered in this work, a sequence of values on the curve is defined as follows:\[\tag{1}{n_D}(i) = \frac{{{N_D}(i)}}{{N - i}}, i = 1,2,..., N - 1,\]
where ${{N_D}(i)}$ is the number of driver nodes needed to retain the network controllability after removing $i$ nodes, and $N-i$ is the number of remaining nodes after the $i$-$th$ attack.
\vspace{-0.4cm}
\subsection{Connectivity Robustness}
For an undirected network, network connectivity is measured by connectedness. If an undirected network is connected, there is a path between any two nodes. While for a directed network, network connectivity is measured by weak connectedness. A directed network is weakly connected if it remains to be connected by ignoring the edge directions.

The connectivity robustness can be measured by the connectivity curve. This curve records the proportion of nodes in the LCC under continuous attack, which is mathematically given by:
\[\tag{2}s(i) = \frac{{{N_{LCC}}(i)}}{{N - i}},i = 1,2,...,N - 1,\]
where $N_{LCC}$ is the number of nodes in LCC after removing $i$ nodes. \emph{Network Connectivity Robustness Learning} is to predict such a curve through the designed model.
\vspace{-0.4cm}
\subsection{Overall Robustness}
The overall measure of the network controllability robustness and connectivity robustness ${R_c}$ can be calculated by:

\[\tag{3}{R_c} = \left\{ \begin{array}{l}
{\rm{   }}\frac{1}{{N - 1}}\sum\limits_{i = 1}^{N - 1} {{n_D}(i)} {\rm{     }}\quad Controllability\quad Robustness\\
{\rm{   }}\frac{1}{{N - 1}}\sum\limits_{i = 1}^{N - 1} {s(i)} {\rm{        }} \quad Connectivity\quad Robustness
\end{array} \right.\]

Given two networks under the same sequential attack, for connectivity robustness, the one with the higher value of $R_c$ has better connectivity robustness, while for controllability robustness, the one with the smaller value of $R_c$ performs better.

\begin{remark}
Note that the robustness curve (including controllability curve and connectivity curve)   and overall robustness $R_c$ defined in (3) indicate the robustness of a network from different perspectives.  Specifically, robustness curve reflects the variation of network robustness during the attack process while $R_c$ measures the robustness of a network as a whole. 
\end{remark}
\vspace{-0.4cm}
\subsection{Error Measures}
The error in the network robustness learning process can be visualized by $Er$ curve. This curve represents the deviation of the true and predicted values of the controllability curve or connectivity curve. Each value on this curve is calculated by:
\[\tag{4}Er(i) = abs(pv(i) - tv(i)),i = 1,2,...,N - 1\]
where $pv(i)$ and $tv(i)$ represents the predicted value and true value of $n_D(i)$ or $s(i)$. The average value of the error curve $\overline {Er}$ is calculated by\[\tag{5}\overline {Er} = \frac{1}{{N - 1}}\sum\limits_{i = 1}^{N - 1} {Er(i)} \]
which implies the precision of robustness learning. When comparing the performance of different algorithms in \emph{Network Robustness Learning}, the highest precision corresponds to the lowest $\overline {Er}$.

\section{The framework of NRL-GT}
In this section, we will introduce detailed information about NRL-GT. As shown in Fig. 1, in NRL-GT, the degree distribution is encoded by a degree centrality encoder into initial node features. A backbone consisting of multiple graph transformer layers generates more informative node representations with initial node features and network structure information and sends them to the three downstream modules to accomplish different tasks. First, we describe the degree centrality encoder of NRL-GT. We then describe the graph transformer layer (GT layer), which is the core of the backbone in NRL-GT. Next, we provide different downstream modules of NRL-GT. Finally, we present the training strategies and loss functions of this work.
\begin{figure*}
\centering
\includegraphics[width=6in]{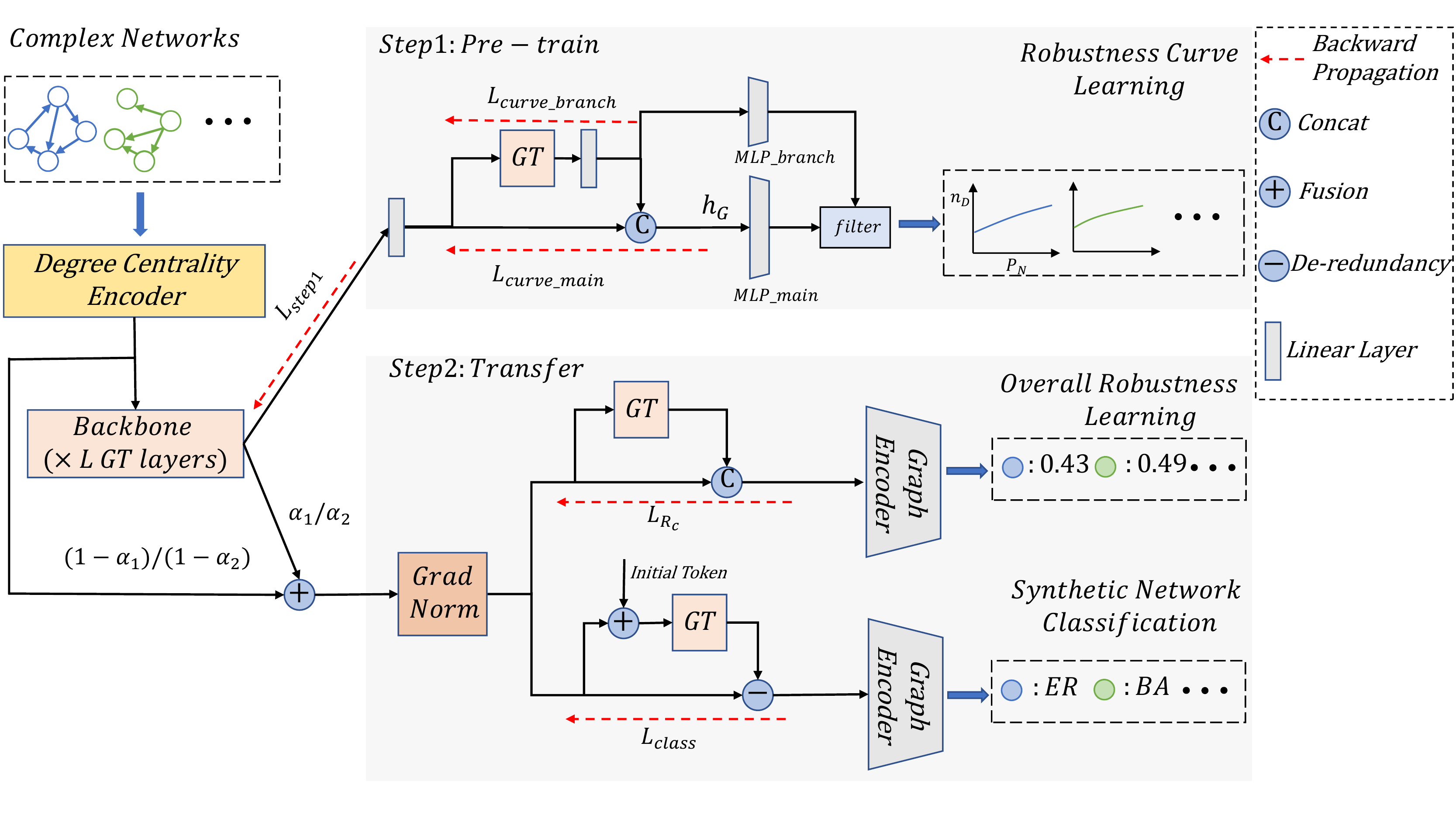}
\begin{center}
\caption{The overall architecture of NRL-GT. There is a backbone and three downstream modules in NRL-GT including robustness curve learning, overall robustness learning, and synthetic network classification.}
\label{Fig. 1}
\end{center}
\vspace{-0.6cm}
\end{figure*}

\vspace{-0.3cm}
\subsection{Degree Centrality Encoding}
As discussed in the Introduction, degree-related node attributes play an important role in measuring network robustness. Based on this, in NRL-GT, we specially design a degree centrality encoder to encode the degree distribution as the initial input of GT layers, as illustrated in Fig. 1. Specifically, for each node $i$, the degree centrality ($D{C_i}$) of the GT layer is defined as
\[\tag{6}D{C_i} = x_{{{\deg }^ - }({v_i})}^ - ||x_{{{\deg }^ + }({v_i})}^ + \]
where ${x^ - }$ and ${x^ + }$ are learnable embedding vectors assigned by in-degree ${\deg ^ - }({v_i})$ and out-degree ${\deg ^ + }({v_i})$, respectively. $||$ means
concatenation operation. For directed networks, two different centrality encoders are employed for in- and out-degree. And for undirected networks, only one common centrality encoder is shared.
\vspace{-0.5cm}
\subsection{Graph Transformer Layer}\label{Graph Transformer Layer}
Fig. 2 shows the overall architecture of the proposed Graph Transformer Layer, which is built on the framework of the classical transformer. The goal of the GT layer is to capture topological structure information of complex networks and generate robustness-related node representations. For node $i$, assuming that the representation of layer $l$ is $h_i^l \in {R^d}$ and the representation of corresponding neighbor nodes $\{ j|j \in N(i)\} $ is $\{ h_j^l\}$, the updating process to obtain the representation of layer $l+1$ ($h_i^{l + 1}$) can be described as:
\[\tag{7}h_i^{l + 1} = \sum\limits_{m = 1}^M {(head_{outer}^{m,l}(h_i^l,\{ h_j^l\} ))} \]
where
\[\tag{8}\begin{array}{l}
head_{outer}^{m,l}(h_i^l,\{ h_j^l\} ) = h_i^l + {W^{m,l}}(\sigma (\mathop {||}\limits_{s} (head_{inner}^{^{s,m,l}}(h_i^l,\{ h_j^l\} )))\\
head_{inner}^{^{s,m,l}}(h_i^l,\{ h_j^l\} )  = \!\! \!\!\sum\limits_{j \in {N(i)}} \!\! \!\! {f(\!\frac{{{Q_i^{s,m,l}} \cdot W_e^{s,m,l}{K_j^{s,m,l}}}}{{\sqrt {\left\lfloor {d/S} \right\rfloor} }})} (W_v^{s,m,l}{V_j^{s,m,l}})\
\end{array}\]
$\left\lfloor \right\rfloor $ denotes the floor function. ${{Q_i^{s,m,l}}}$, ${{K_j^{s,m,l}}}$, ${V_j^{s,m,l}} \in {R^{{\left\lfloor {d/S} \right\rfloor}}}$ corresponds to query, key, value in general transformer, which is calculated by:
\[\tag{9}\left\{ \begin{array}{l}
Q_i^{s,m,l} = linea{r_Q^{s,m,l}}(h_i^l)\\
K_j^{s,m,l} = linea{r_K^{s,m,l}}(h_j^l)\\
V_j^{s,m,l} = linea{r_V^{s,m,l}}(h_j^l)
\end{array} \right.\]
Message passing between central and neighboring nodes for both the classical Transformer~\cite{80} and classical Graph Transformer~\cite{101,102,103} can be defined as 
\[\tag{10}h_i^{l + 1} = h_i^l + {W^{l}}\mathop {||}\limits_s (\sum\limits_{j \in N(i)}  f(\frac{{Q_i^{s,l} \cdot K_j^{s,l}}}{{\sqrt {\left\lfloor {d/S} \right\rfloor } }})(V_j^{s,l}))\]
The comparison reveals that, unlike the classical Transformer-based methods that contain only one type of multi-heads, we design a multi-head attention mechanism that combines inner and outer heads to mine richer graph representations, and the final output is a summation of the information from $M$ outer heads. We apply a linear projection to update the concatenation of the output of $S$ inner heads, followed by residual connection to characterize an outer head operation. Both $M$ and $S$ are adjustable hyperparameters. Instead of directly multiplying the query vector and key vector to measure the similarity, we add a trainable weight $W_e^{s,m,l}$ that characterizes the edge information. Another trainable weight $W_v^{s,m,l}$ is used to incorporate the dependency of nodes and connected edges in the message passing process.

\begin{theorem}
The proposed graph transformer layer is able to outperform the classical GNN models: GCN~\cite{46}, GAT~\cite{47}, and GraphSAGE~\cite{45} in network robustness learning.
\end{theorem}

\begin{proof} We divide the proof into two parts.
	
 i) For GAT, since graph attention can be considered as a special case of self-attention~\cite{85}, the single multi-inner-head mechanism is comparable to the operations of the general graph attention network (GAT). Considering that the proposed GT incorporates both the multi-inner-head and multi-outer-head mechanisms, GT can capture richer features compared with GAT.
 
ii) Note that the aggregation process of both GCN and GraphSAGE can be treated as summing neighboring nodes with certain weights. For GCN and GraphSAGE, the information passed from source node $j$ to target node $i$ during the message propagation can be expressed as $\frac{1}{{\sqrt {{d_i}{d_j}} }}{h_j}$ and $\frac{1}{{{d_i}}}{h_j}$, respectively, where ${{d_i}}$ and ${{d_j}}$ represent the in-degree of the central and neighbor nodes. For GT layer, it can be denoted as $f(\frac{{{Q_i} \cdot W_e{K_j}}}{{\sqrt {\left\lfloor {d/S} \right\rfloor} }})(W_v{V_j})$ if only one attention inner head is applied. Actually, ${{Q_i}}$, ${{K_j}}$ contain the linear transformations of ${{d_i}}$ and ${{d_j}}$ since the input is characterized by degree centrality and $W_v{V_j}$ is $W_v(linea{r_V}({h_j}))$. Relying on the degree characteristics and the universal approximation theorem of the MLP~\cite{93}, The aggregation weights of both GCN and GraphSAGE can be approximated by GT, which shows the aggregation of GCN and GraphSAGE are special forms of the outer head of GT layer. From this perspective, the structural feature extraction ability of GT layer is also superior to that of GCN and GraphSAGE.	
\end{proof}
\vspace{-0.2cm}
Apart from above theoretical analysis, we also provide specialized ablation experiments in Section IV-G to demonstrate the significant advantages of the proposed GT layer over GCN, GAT, and GraphSAGE in network robustness learning.

\begin{figure}[h]
\centering
\includegraphics[width=2.2in]{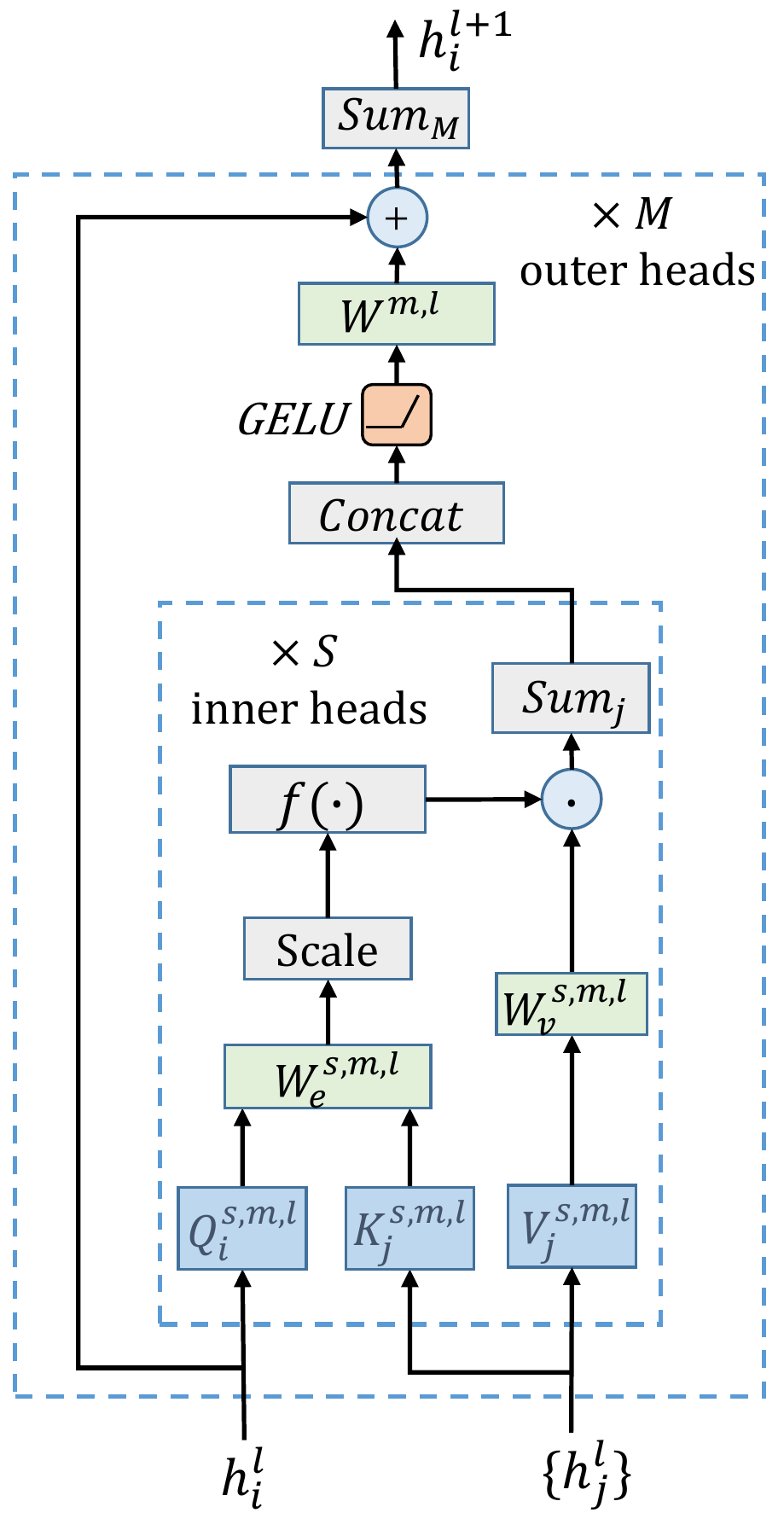}%
\begin{center}
\caption{The architecture of the proposed Graph Transformer Layer, which generates robustness-related node representations.}
\label{Fig. 2}
\end{center}
\vspace{-0.8cm}
\end{figure}

\vspace{-0.4cm}
\subsection{Downstream Modules of NRL-GT}
NRL-GT is a general framework for network robustness learning that not only enables the learning of network robustness curves and overall measures of robustness but also the classification of synthetic networks. In the backbone, the initial degree distribution is encoded as degree centrality and then subjected to $L$ layers of GT operations to generate robustness-related node representations for the following downstream modules: robustness curve learning module, overall robustness learning module, and synthetic network classification module.

{\bf{Module \#1: Robustness Curve Learning}}

The module for robustness curve learning accomplishes the prediction of the network controllability curve and connectivity curve which contains two parts: the main stem and the branch, providing different depths of structural information. The whole network features ${h_G}$ are concatenated by the features of each node in complex network $G$ before $MLP_{main}$ operation in the network. The MLP and filter are employed to transform ${h_G}$ into the final learning results, denoted as:
\[\tag{11}Robustness\_curve = filter({MLP_{main}}({h_G}))\]
For controllability robustness learning, the essence of this filter is a trainable physical bias, which facilitates easy model training and avoids meaningless zero values in the learning results. While for the connectivity robustness, we replace the physical bias with a filter presented in \cite{68} to better fit the properties of the connectivity curve.

{\bf{Module \#2: Overall Robustness Learning}}

As introduced in Section II, the overall robustness, which is another indicator for network robustness, can be measured by a single value ${R_c}$. Generally, the predicted value of overall robustness $R_c$ is usually obtained by averaging the learning results of the robustness curve. However, for networks with different sizes, due to the inconsistent size of the robustness curve, it is necessary to regenerate a large amount of data to train models to obtain the overall network robustness based on the learned robustness curve, which causes a high training cost. Motivated by this, we design an overall robustness learning module in NRL-GT dedicated to predicting the overall robustness of complex networks of different sizes without redundant training. The input node feature $\hat h^0$ to the ${R_c}$ learning module is fused from two components:
\[\tag{12}\hat h^0 = {\alpha _1} \cdot h^L + (1 - {\alpha _1}) \cdot h^0\]
where $h^0$ and $h^L$ represent the characteristics before and after the backbone operation, respectively, and ${\alpha _1}$ is a trainable weighting coefficient. 

There is also an additional GT layer in this module which makes $\hat h^0$ become $\hat h^1$. The computation of ${R_c}$ can be generalized as:
\[\tag{13}{R_c} = Graph\_Encoder([{\hat h^0}||{\hat h^1}])\]
Graph Encoder is an encoder consisting of the graph average pooling layer and the linear layer, which transforms the received node features into graph embedding to complete the prediction of $R_c$.

{\bf{Module \#3: Synthetic Network Classification}}

A unique advantage of NRL-GT is the existence of a network classification module, making it applicable to network classification of any size. The input components of the classification module is similar to that of ${R_c}$ learning module, with the summed weights being ${\alpha _2}$. For classification module, the input $\tilde h^0$ is calculated by:
\[\tag{14}\tilde h^0 = {\alpha _2} \cdot h^L + (1 - {\alpha _2}) \cdot h^0\]
Inspired by~\cite{94}, we introduce a special initial token ${x_{It}}$ and an additional GT to deeply capture graph topology features as well as perform simple feature de-duplication operations on top of that. Through this path, $\tilde h^0$ becomes $\tilde h^1$. The classification can be denoted as:
\[\tag{15}{P_G} = Graph\_Encoder(\tilde h^1)\]
We also apply a Graph Encoder for category prediction and ${P_G}$ represents the probability of each category.
\vspace{-0.4cm}
\subsection{Training Strategy and Loss Function}\label{Training Strategy and Loss Function}
As illustrated in Fig. 1, the trainable modules in NRL-GT consist of a backbone network and three downstream branches, with corresponding loss function available for each branch. And the complete training process for NRL-GT involves two steps: pre-training and transfer learning, as presented below.

{\bf{Step 1}} Firstly, the centrality encoder, backbone, and robustness curve learning module are trained as a whole to generate a pre-trained backbone. Meanwhile, the task of learning robustness curve is solved in this step. We design reweighted mean-squared error loss (RMSE) to ensure a low relative error even when the prediction value is small, which is given by:
\[\tag{16}RMSE(pv,tv) = \frac{1}{{N - 1}}\sum\limits_{i = 1}^{N - 1} {{W_{RMSE}}(i){{(pv(i) - tv(i))}^2}} \]
where $W_{RMSE}$ is a hyperparameter. The final loss of this step is given by:
\[\tag{17}L_{step1} = {L_{curve\_main}} + \rho  \times {L_{curve\_branch}}\]
where ${L_{curve\_main}}$ and ${L_{curve\_branch}}$ indicate the RMSE loss of the main stem and branch shown in Fig. 1 and $\rho$ is an adjustable hyperparameter. Branch loss promotes intermediate features more applicable to our task.

{\bf{Step 2}} After updating the parameters in Step 1, we obtain a pre-trained backbone and the centrality encoder that can be transferred directly to other tasks. The network classification module and the $R_c$ learning module are trained simultaneously in Step 2. To balance the gradient norms in multi-task learning on graphs, we introduce Grad-Norm~\cite{86} to tune the weights of network classification loss and $R_c$ learning loss. Let $L_{class}$ denote the network classification loss, which is determined by:
\[\tag{18}L_{class} = crossentropy({P_G},{Y_G}) =  - \sum\limits_{i = 1}^C {Y_G^i\log (P_G^i)} \]
crossentropy loss is employed for classification, where $C$ represents the total number of categories, ${P_G}$ represents the predicted probability, and ${Y_G}$ represents the true label. The loss of the $R_c$ learning is specified as:
\[\tag{19}L_{R_c} = {({{\hat R}_c}(G) - {R_c}(G))^2}\]
where ${{\hat R}_c}(G)$ indicates the predicted value while ${R_c}(G)$ denotes the true value of overall robustness for network G.

As illustrated in Fig. 1, the total loss in this step is the weighted combination of $L_{class}$ and $L_{R_c}$ where the weights $w_c$ and $w_R$ are tuned to deal with the imbalanced gradient norms based on Grad-Norm, that is,
\[\tag{20}L_{step2} = {w_c} \cdot L_{class} + {w_R} \cdot L_{R_c}\]

\section{Experiments}
In this section, we experimentally evaluate the performance of NRL-GT. We first present the details of the datasets and default model parameter settings in this work. Then a series of experiments are conducted to evaluate the ability of the NRL-GT in terms of network robustness performance learning on various types of networks. After that, we carry out transfer learning experiments to verify the transferability of the backbone in NRL-GT and to evaluate the capabilities of NRL-GT for overall network robustness learning and synthetic network classification. And the inference time of different methods is compared. Finally, the performance comparison with the classical GNN and transformer-based models is carried out in ablation studies to further validate the effectiveness of the proposed GT.
\vspace{-0.4cm}
\subsection{Preparation}
{{\bf Datasets}} We evaluate the performance of NRL-GT on five synthetic network models, including the  Erdös-Rényi (ER) \cite{54}, generic scale-free (SF)~\cite{55}, q-snapback(QSN)~\cite{29}, Newman–Watts small-world (NW)~\cite{56}, Barab\'asi–Albert (BA) scale-free networks~\cite{87}. The difference between BA and SF in this work is that for BA, $\gamma = 3 $ with regard to the power law distribution $k^{-\gamma}$, while for SF, $\gamma = 2.001$.

Various types of node-removal attacks are considered in the experiment, including random attacks,  targeted degree-based attacks (TDA), and targeted betweenness-based attacks (TBA). 

All experiments are carried out on a computer with a 64-bit operating system, installed with Intel i7-11700K (3.6GHz) and NVIDIA Geforce RTX 3060.

{{\bf Default parameter settings}} The results reported in this paper are based on the following parameter settings of NRL-GT: \emph{dimension of degree centrality}: $10$; \emph{GT layer number in backbone}: $L=2$; \emph{inner head number in GT layer}: $S=2$; \emph{outer head number in GT layer}: $M=3$; $\rho$ of $loss_{step1}$: 0.5; \emph{Optimizer}: Adam; \emph{learning rate}: 1e-4; \emph{weight decay}: 5e-5; \emph{the weight vector ${{W_{RMSE}}}$ of RMSE} is initialized as:\[\tag{21}{W_{RMSE}}(i) = \left\{ \begin{array}{l}
{\rm{2}}\quad i \le \left\lfloor {\frac{N}{2}} \right\rfloor\\
{\rm{1}}\quad \left\lfloor {\frac{N}{2}} \right\rfloor < i \le N - 1
\end{array} \right.\]
 For experiments on networks with 1000 nodes, the detailed configuration of NRL-GT is shown in Table I.
\begin{table}[!h]
\centering
\caption{Detailed network architecture\label{tab:table1}}
\begin{tabular}{|c|c|c|c|}
\hline
Groups                                   & Layers              & \begin{tabular}[c]{@{}c@{}}Layer \\ size\end{tabular} & \begin{tabular}[c]{@{}c@{}}Output \\ Feature size\end{tabular} \\ \hline
\multirow{2}{*}{Backbone}                & GT layer            & 10x10      & (1000,10)           \\ \cline{2-4} 
                                         & GT layer            & 10x10      & (1000,10)           \\ \hline
\multirow{5}{*}{\begin{tabular}[c]{@{}c@{}}Robustness Curve \\Learning Module\end{tabular}} & Linear & 10x2       & (1000,2)            \\ \cline{2-4} 
                                         & GT layer            & 2x2        & (1000,2)            \\ \cline{2-4} 
                                         & Linear  & 2x1        & (1000,1)            \\ \cline{2-4} 
                                         & MLP\_branch         & 1000x999   & (999,)              \\ \cline{2-4} 
                                         & MLP\_main           & 3000x999   & (999,)              \\ \hline
\multirow{2}{*}{\begin{tabular}[c]{@{}c@{}}Overall Robustness \\Learning  Module\end{tabular}}    & GT layer            & 10x10      & (1000,10)           \\ \cline{2-4} 
                                         & Graph Encoder              & 20x1       & (1,)                \\ \hline
\multirow{2}{*}{\begin{tabular}[c]{@{}c@{}}Synthetic Network \\Classification Module\end{tabular}}  & GT layer            & 10x10      & (1000,10)           \\ \cline{2-4} 
                                         & Graph Encoder              & 10x5       & (5,)            \\ \hline
\end{tabular}
\vspace{-0.4cm}
\end{table}
\vspace{-0.6cm}
\subsection{Network Controllability Robustness Learning}
The precision and generalization of different models in network controllability robustness learning under RA and TDA will be demonstrated in this subsection. Two state-of-the-art models in the field of network controllability robustness learning named PCR~\cite{31} and iPCR~\cite{32}, are involved in the comparison with NRL-GT. All the models are trained on directed unweighted networks with size $N=1000$. For each synthetic topology, 2000 instances are randomly generated for training. Thus, there are 2000 × 5 = 10000 training samples in total. The range of average degree $\left\langle k \right\rangle $ of each synthetic topology is set as [1,10] (overall average degree is 5.5). The test set covers all types of synthetic complex networks, including directed and undirected, unweighted and weighted networks, with 200 samples of each type with an average degree belonging to [1,10] (overall average degree is 5.5).

Fig. 3 shows the test results of RA. The first row shows the test precision of different models when the training and test sets are consistent. As can be seen from Figs. 3(a), 3(b), and 3(e), every method maintains comparable results on BA, ER, and SF when the training and test sets are kept consistent. In NW and QSN, the error curve of NRL-GT is significantly lower than that of PCR and iPCR, demonstrating higher precision in these two topologies (see Figs. 3(c) and 3(d)). The second to fourth rows show the generalization ability of different methods when the distributions of training and test sets are inconsistent. For NRL-GT, the model trained on unweighted networks can achieve high precision learning directly on weighted networks, which is not reflected in PCR and iPCR (see Figs. 3(f)-3(j) and Figs. 3(p)-3(t)). Moreover, NRL-GT trained on directed networks can be directly applied to undirected networks (see Figs. 3(k)- 3(t)). However, this ability is limited for PCR and iPCR. The generalization ability of NRL-GT is attributed to the powerful inductive learning ability of the proposed Graph Transformer and the fact that it can capture the topological information that determines the controllability robustness of the complex network with no edge weight influence. As shown in  Figs. 3(k), 3(l), and 3(o), PCR and iPCR can get comparable generalization performance with NRL-GT on undirected unweighted BA, ER, and SF networks,  but on undirected NW and QSN networks, PCR and iPCR almost completely lose their generalization ability (see Figs. 3(m) and 3(n)). The average learning errors of different methods for all types of networks are shown in Table II. In most cases, NRL-GT obtains the smallest mean error, demonstrating the higher precision and superior generalization ability of NRL-GT compared to the other two cutting-edge methods in network controllability robustness learning. Fig. S1 and Table S1 describe the learning results of controllability robustness of NRL-GT and PCR under TDA. It can be found that under TDA, NRL-GT trained by unweighted networks can still be generalized to weighted networks for high-precision learning while PCR lacks this capability.
\begin{figure*}[h]
\centering
\subfloat[]{\includegraphics[width=0.2\textwidth]{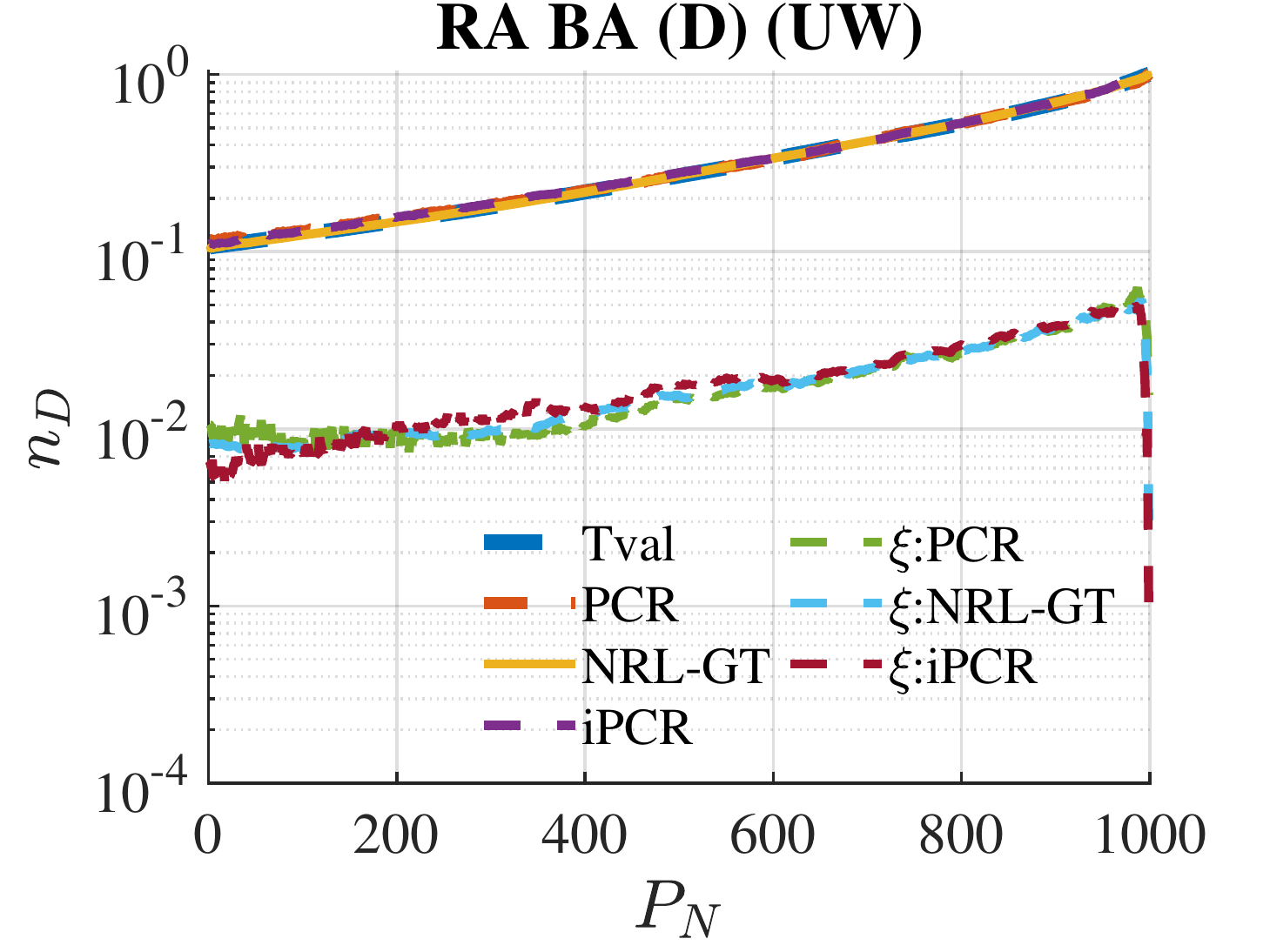}%
\label{a}}
\subfloat[]{\includegraphics[width=0.2\textwidth]{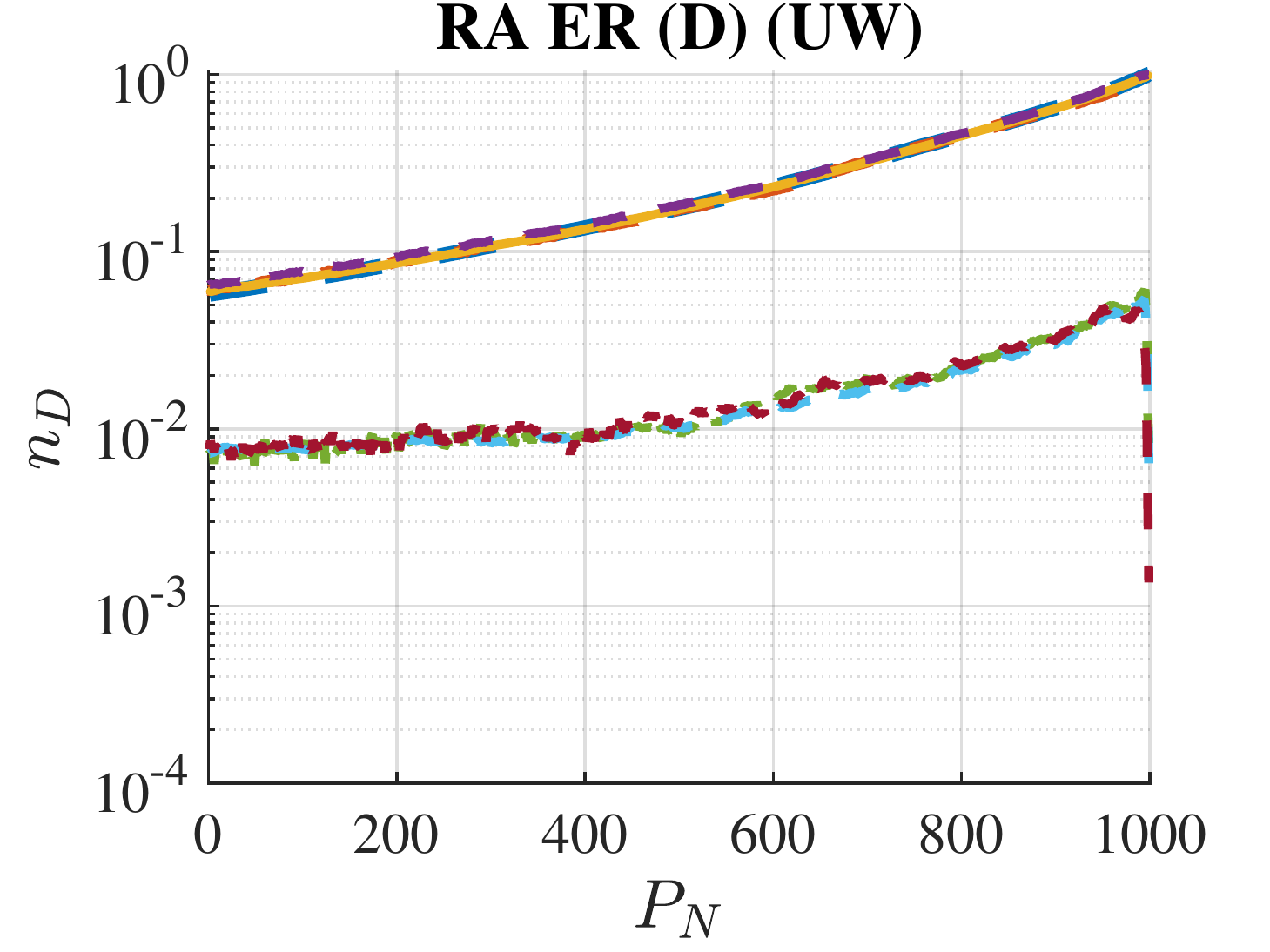}%
\label{b}}
\subfloat[]{\includegraphics[width=0.2\textwidth]{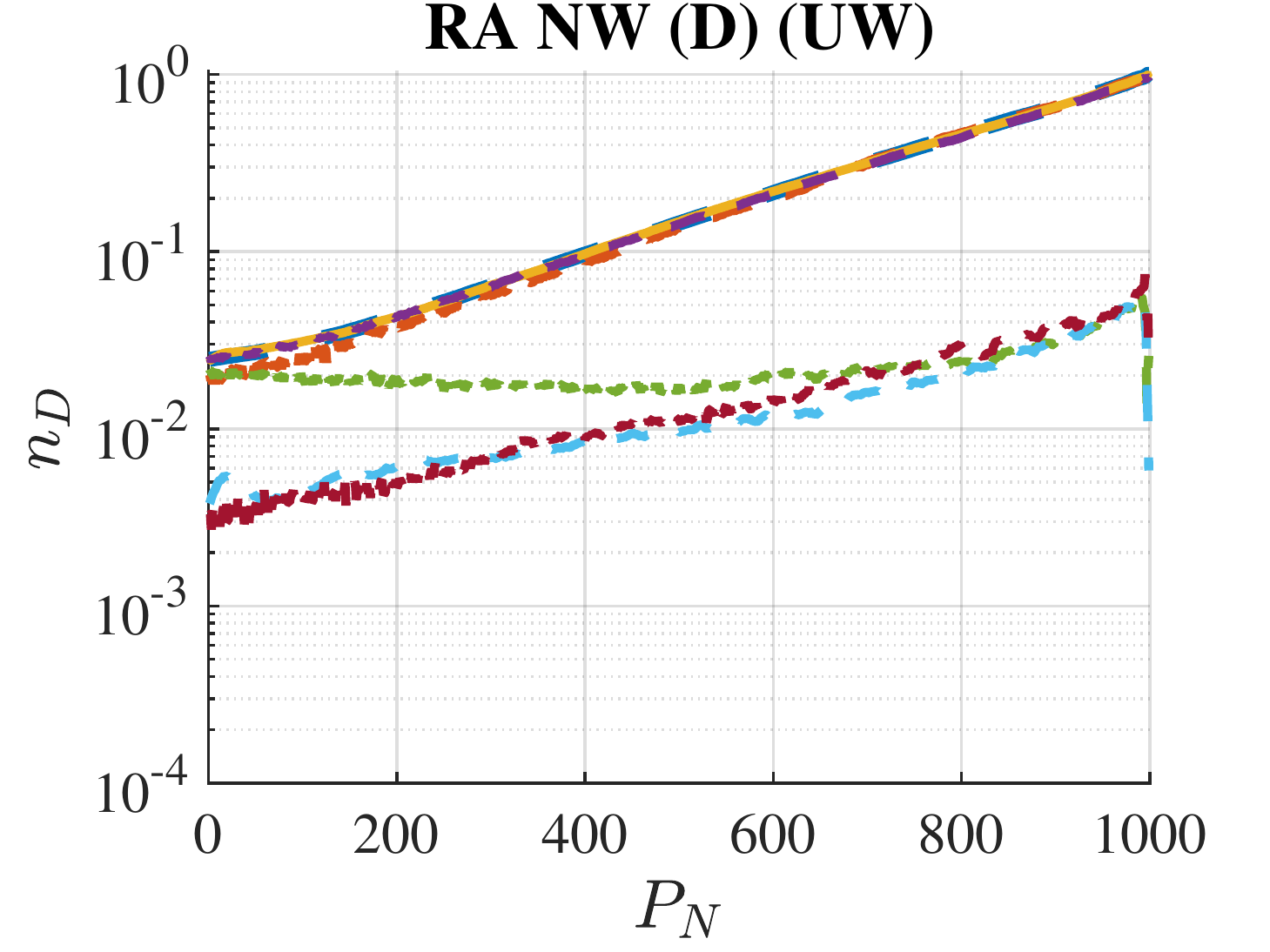}%
\label{c}}
\subfloat[]{\includegraphics[width=0.2\textwidth]{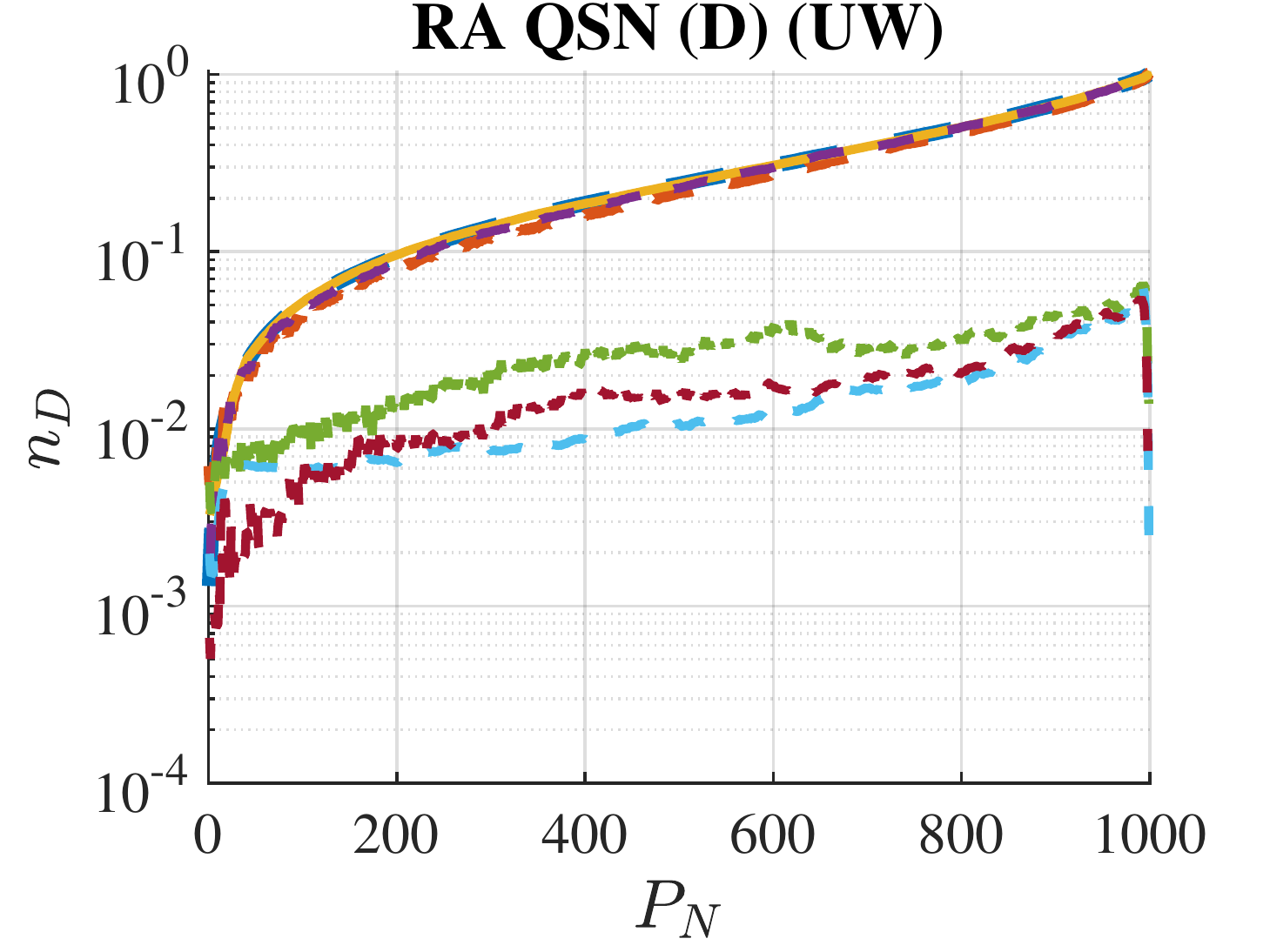}%
\label{d}}
\subfloat[]{\includegraphics[width=0.2\textwidth]{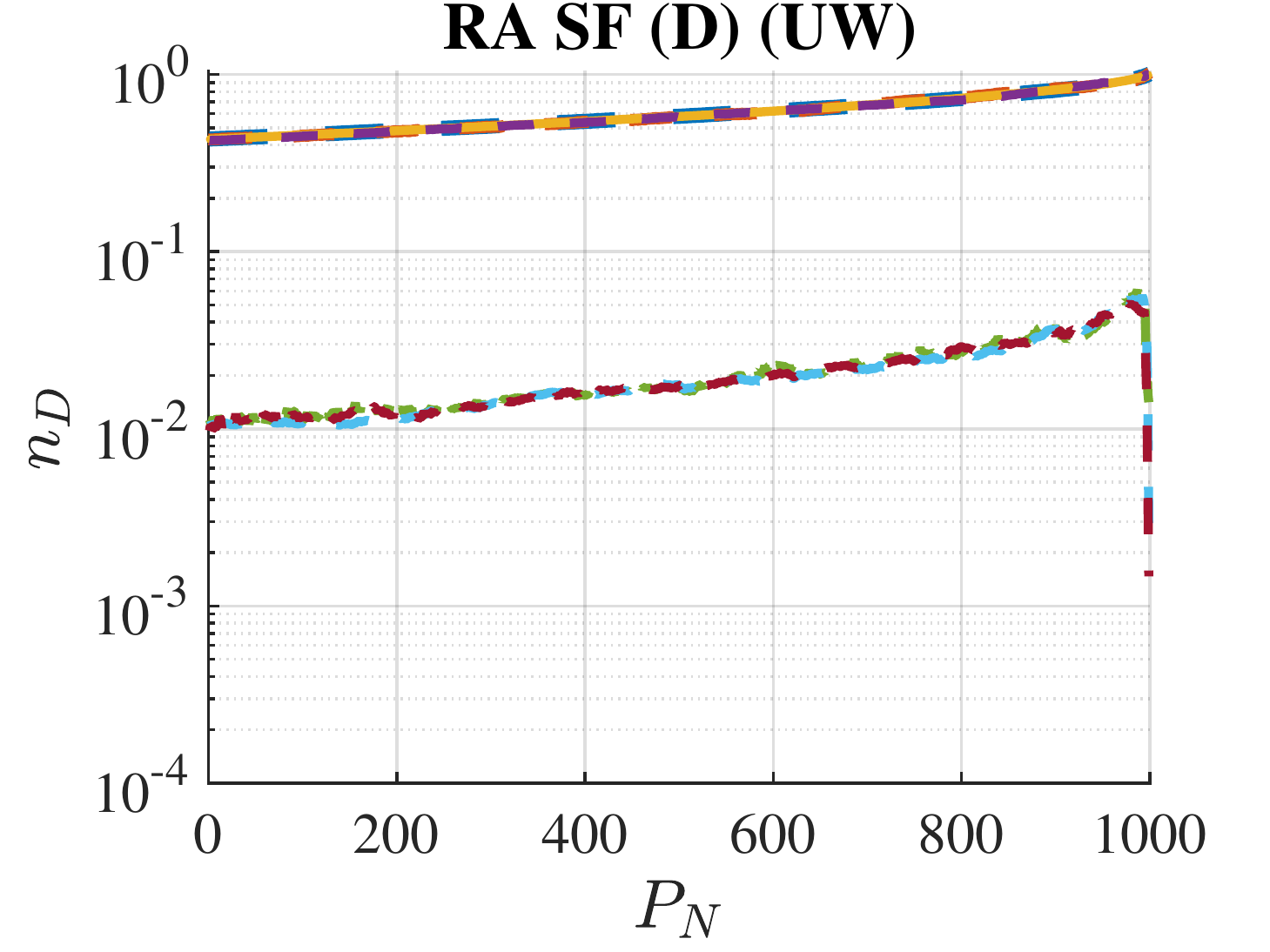}%
\label{d}}

\subfloat[]{\includegraphics[width=0.2\textwidth]{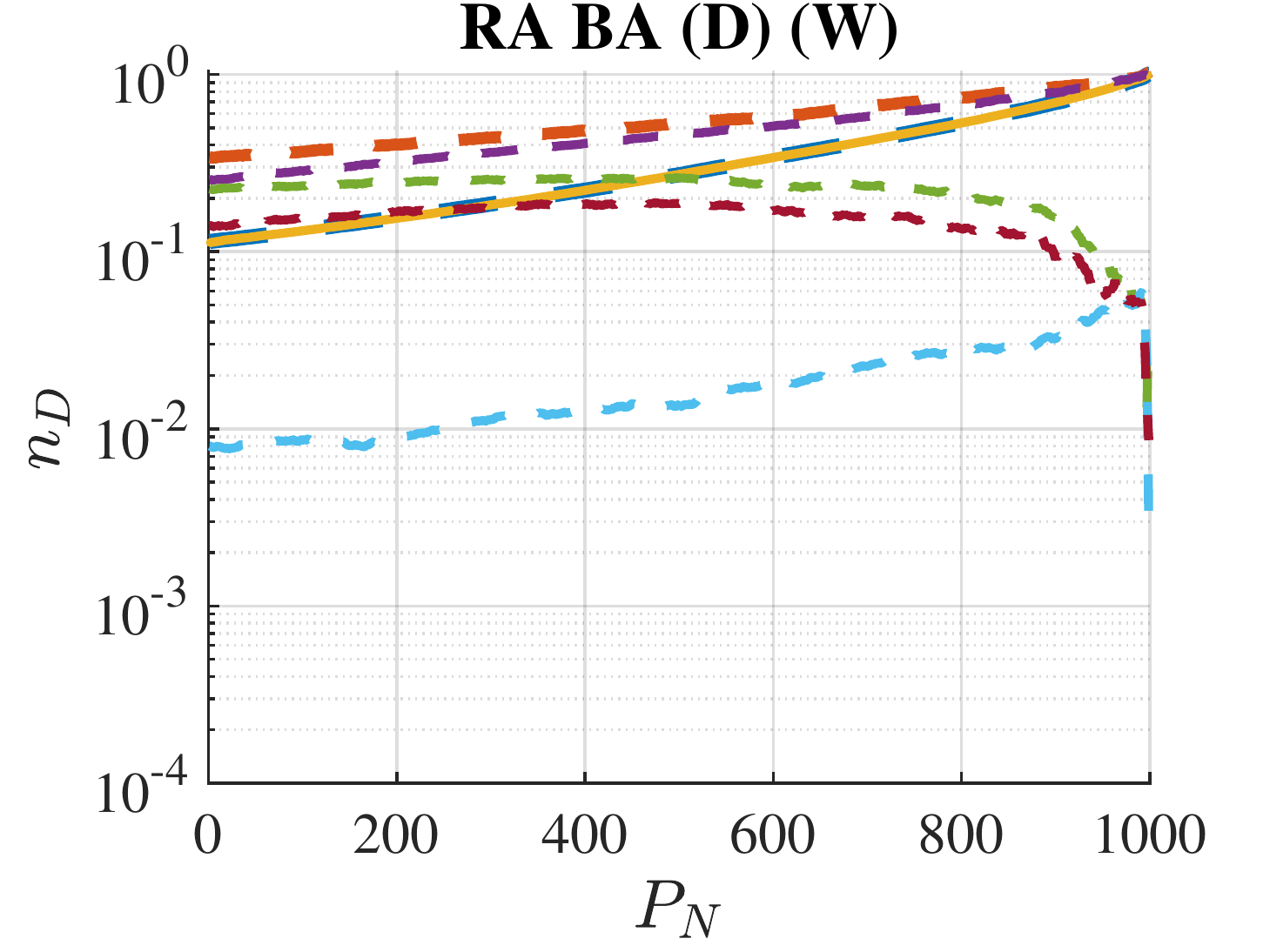}%
\label{a}}
\subfloat[]{\includegraphics[width=0.2\textwidth]{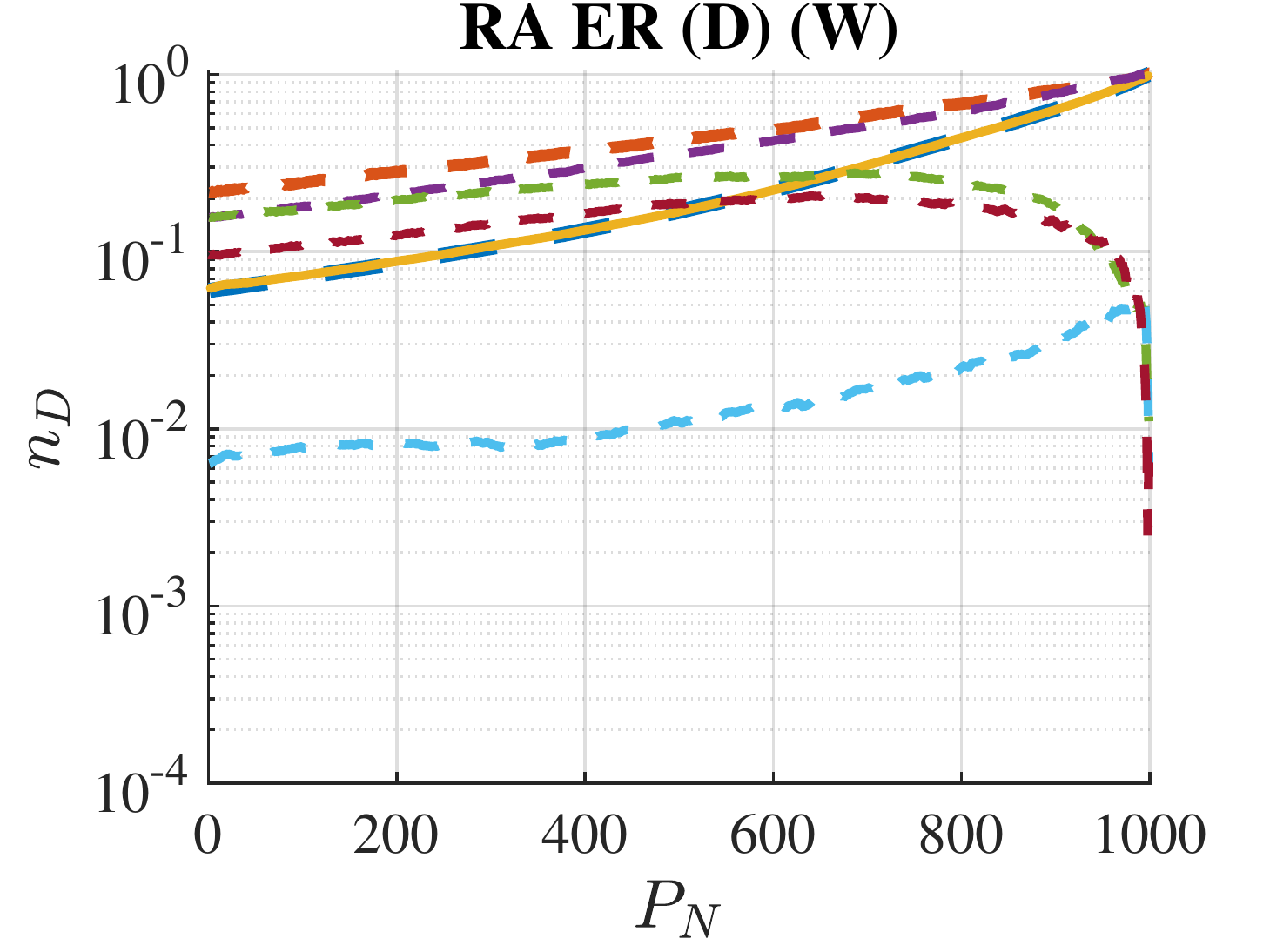}%
\label{b}}
\subfloat[]{\includegraphics[width=0.2\textwidth]{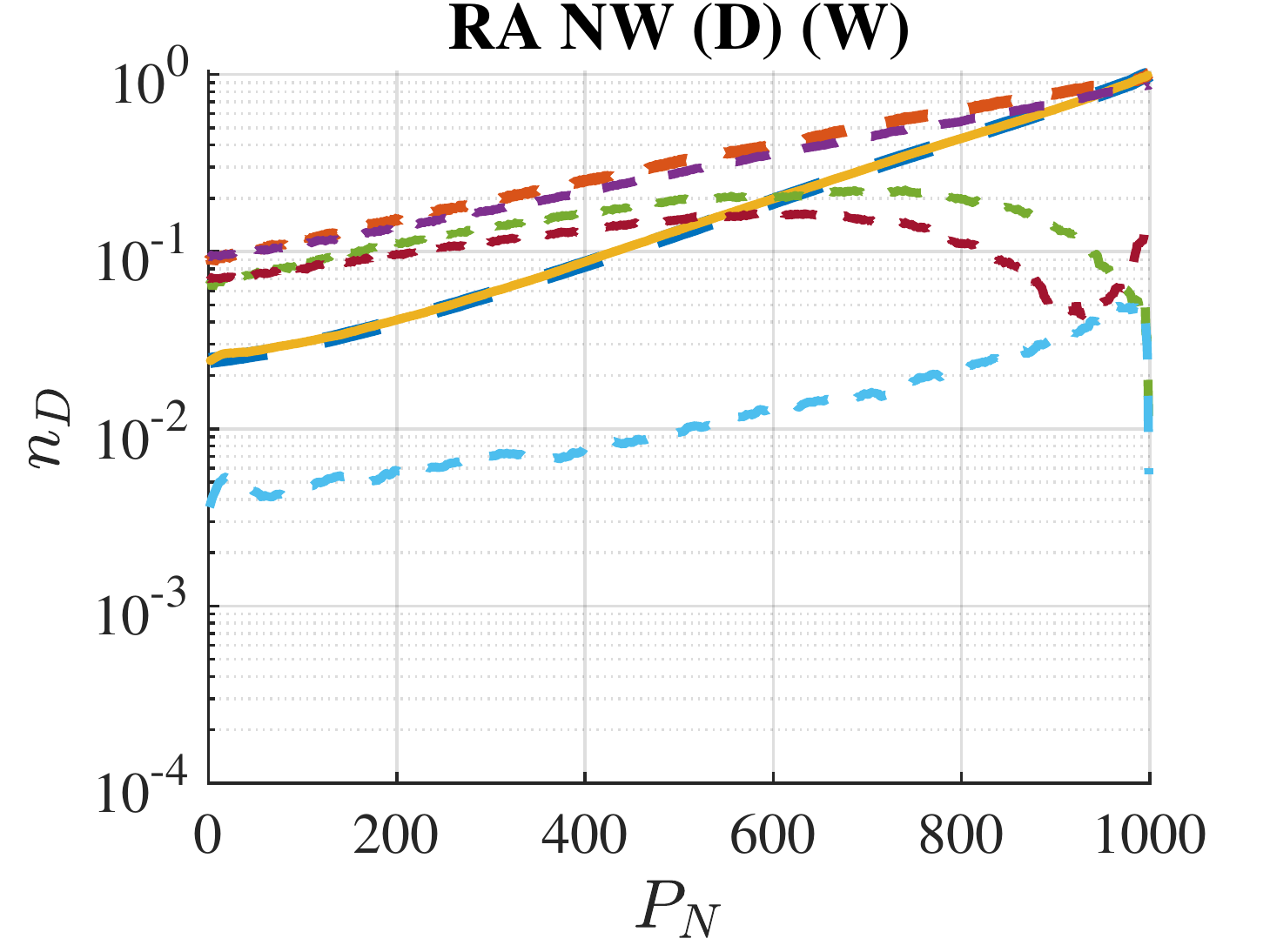}%
\label{c}}
\subfloat[]{\includegraphics[width=0.2\textwidth]{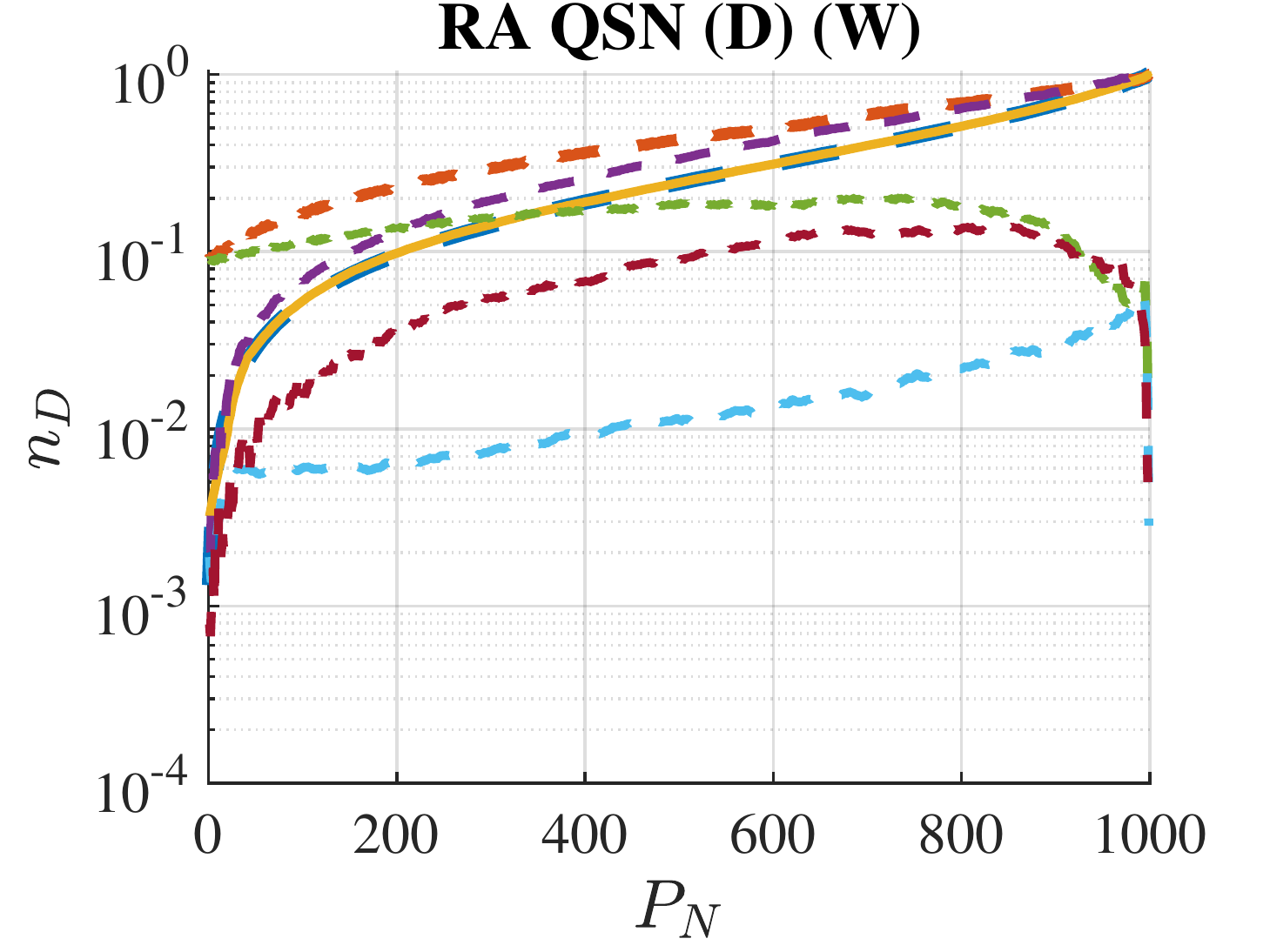}%
\label{d}}
\subfloat[]{\includegraphics[width=0.2\textwidth]{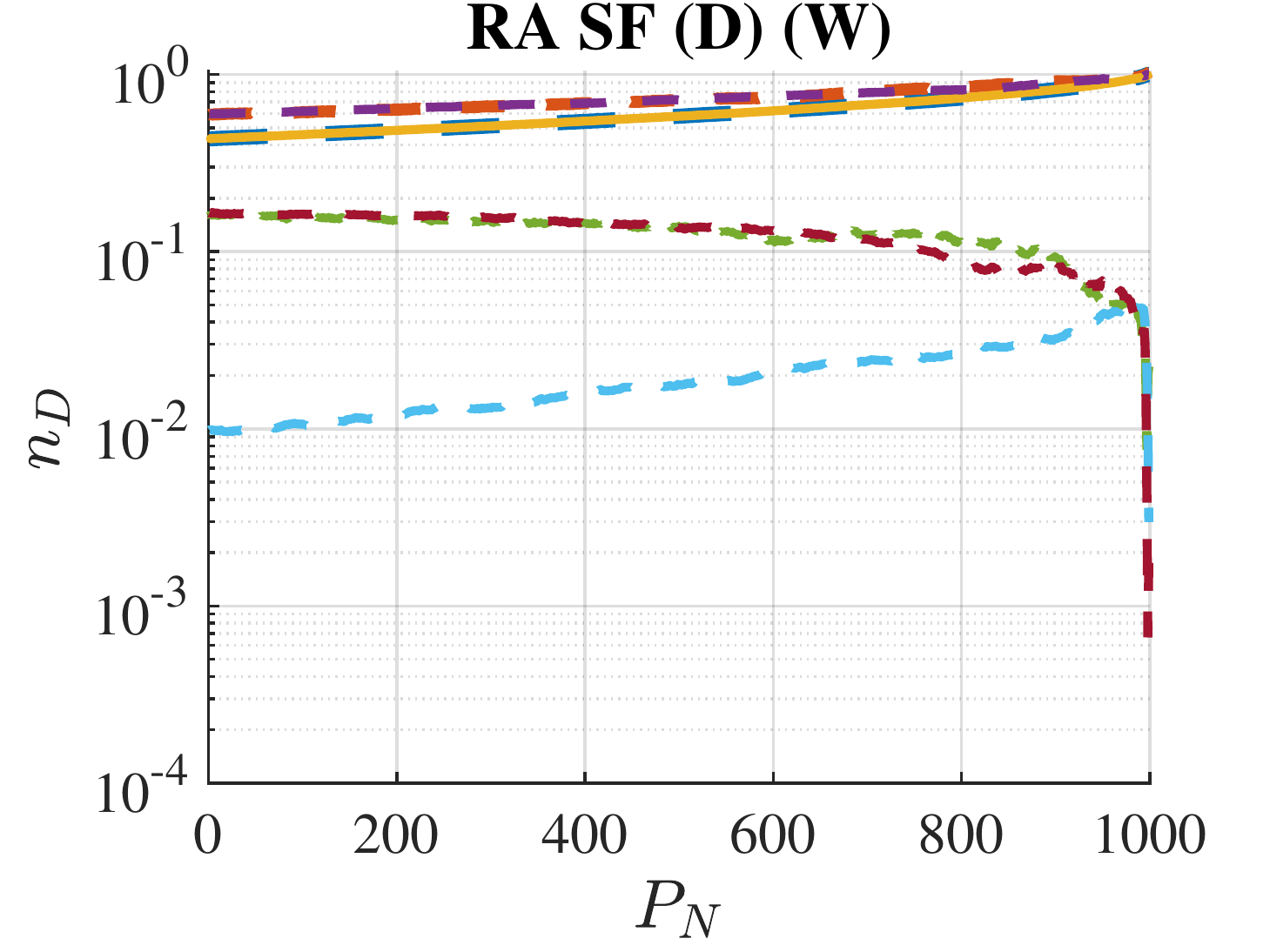}%
\label{d}}

\subfloat[]{\includegraphics[width=0.2\textwidth]{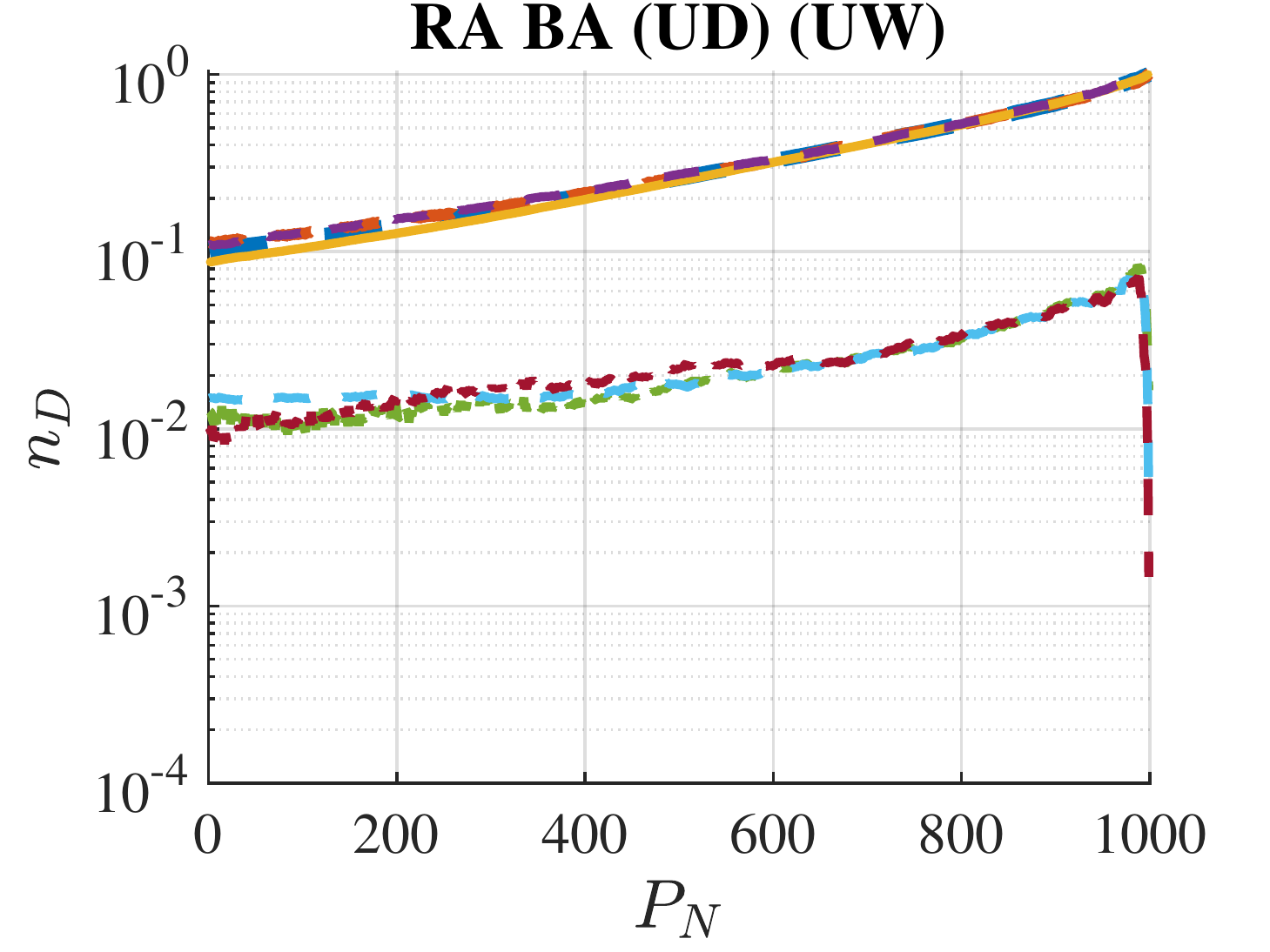}%
\label{a}}
\subfloat[]{\includegraphics[width=0.2\textwidth]{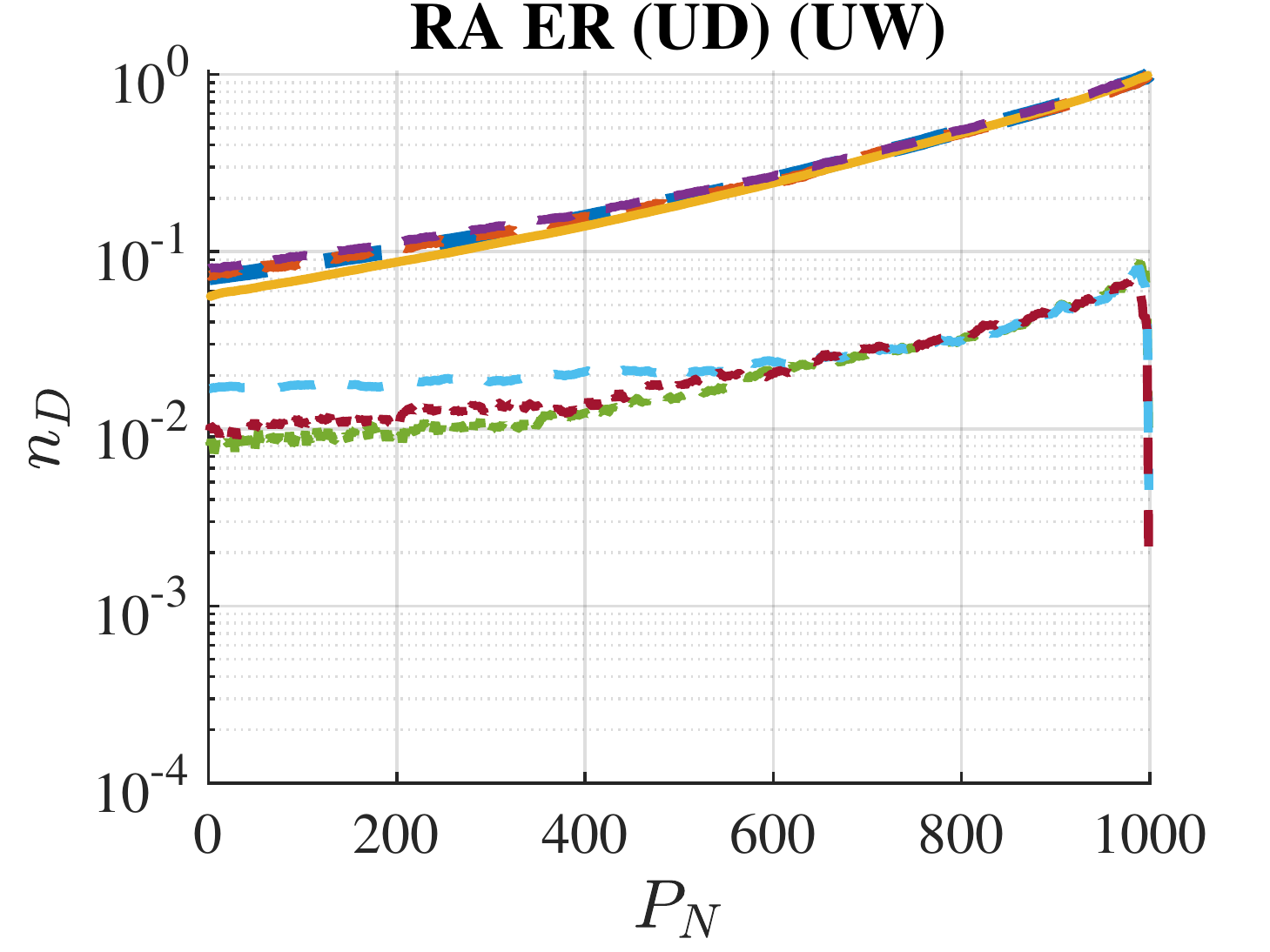}%
\label{b}}
\subfloat[]{\includegraphics[width=0.2\textwidth]{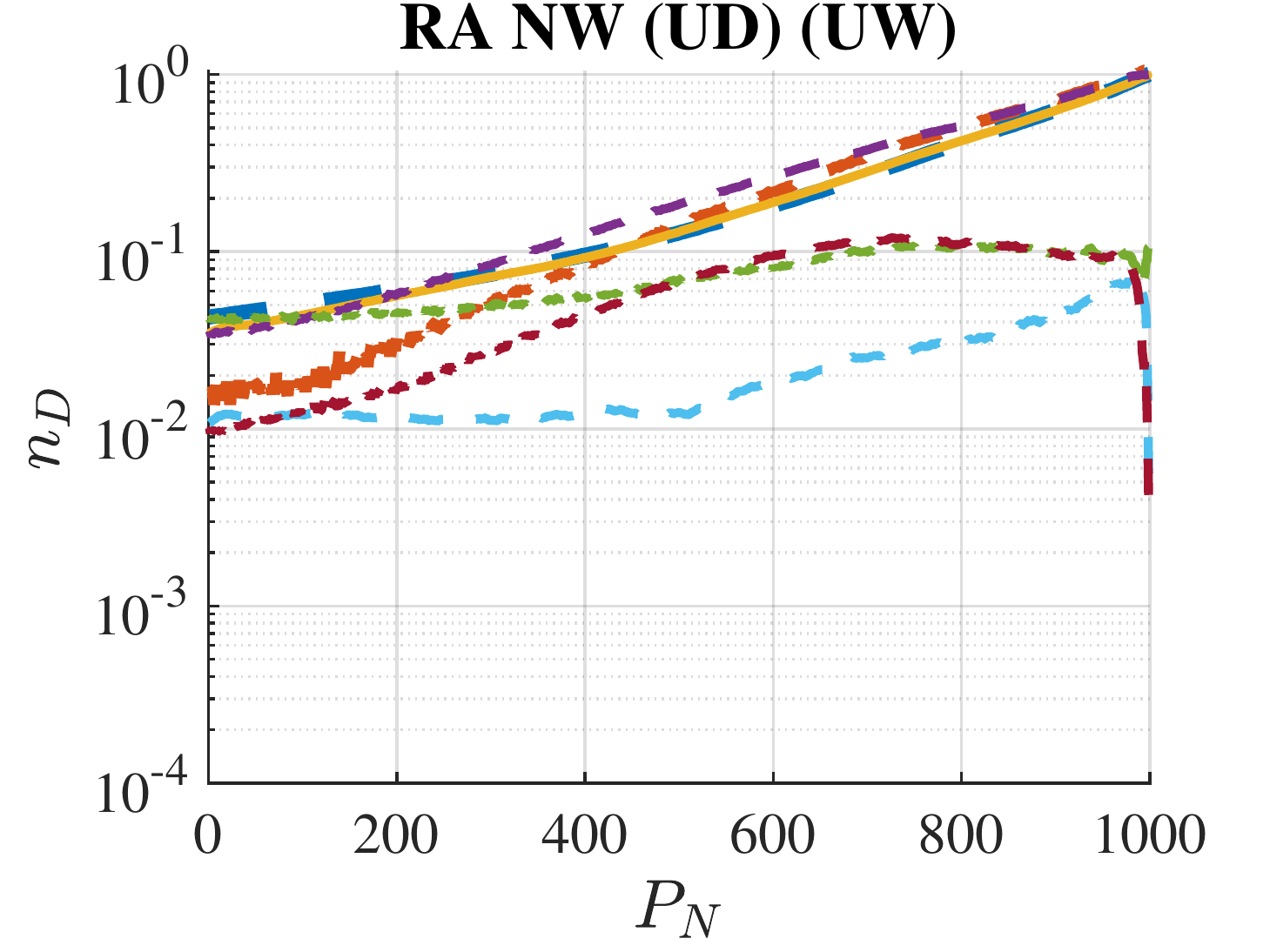}%
\label{c}}
\subfloat[]{\includegraphics[width=0.2\textwidth]{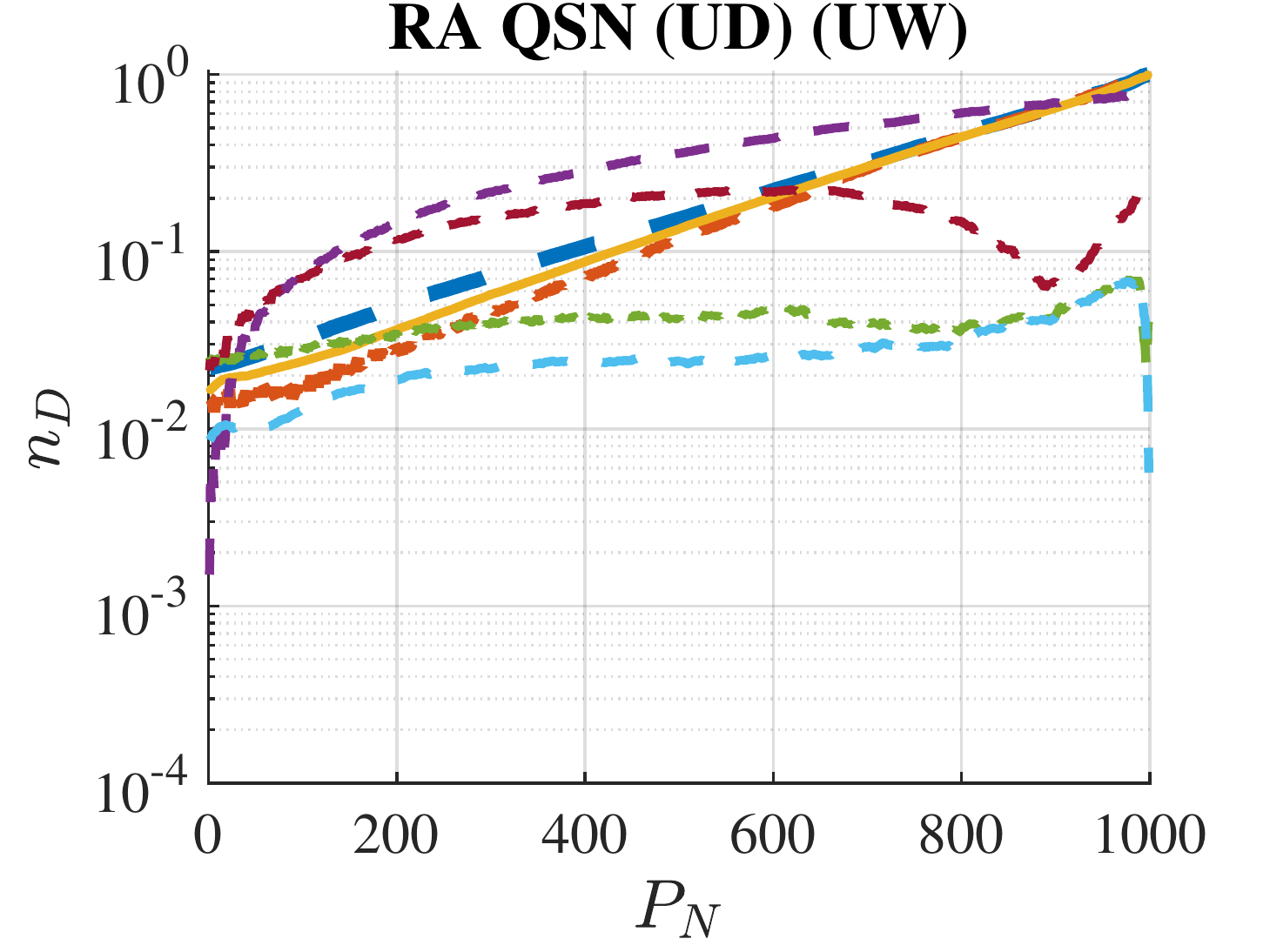}%
\label{d}}
\subfloat[]{\includegraphics[width=0.2\textwidth]{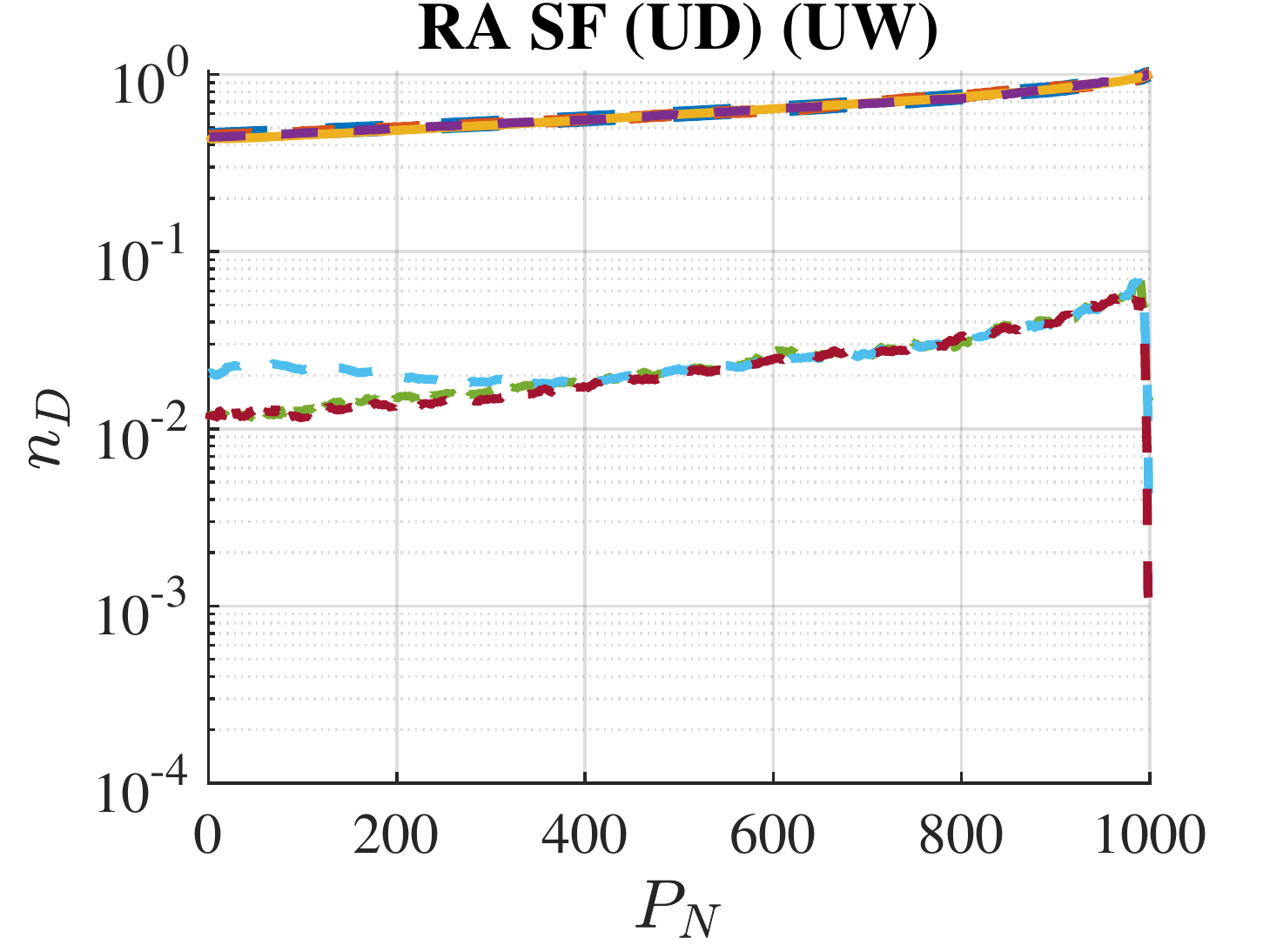}%
\label{d}}

\subfloat[]{\includegraphics[width=0.2\textwidth]{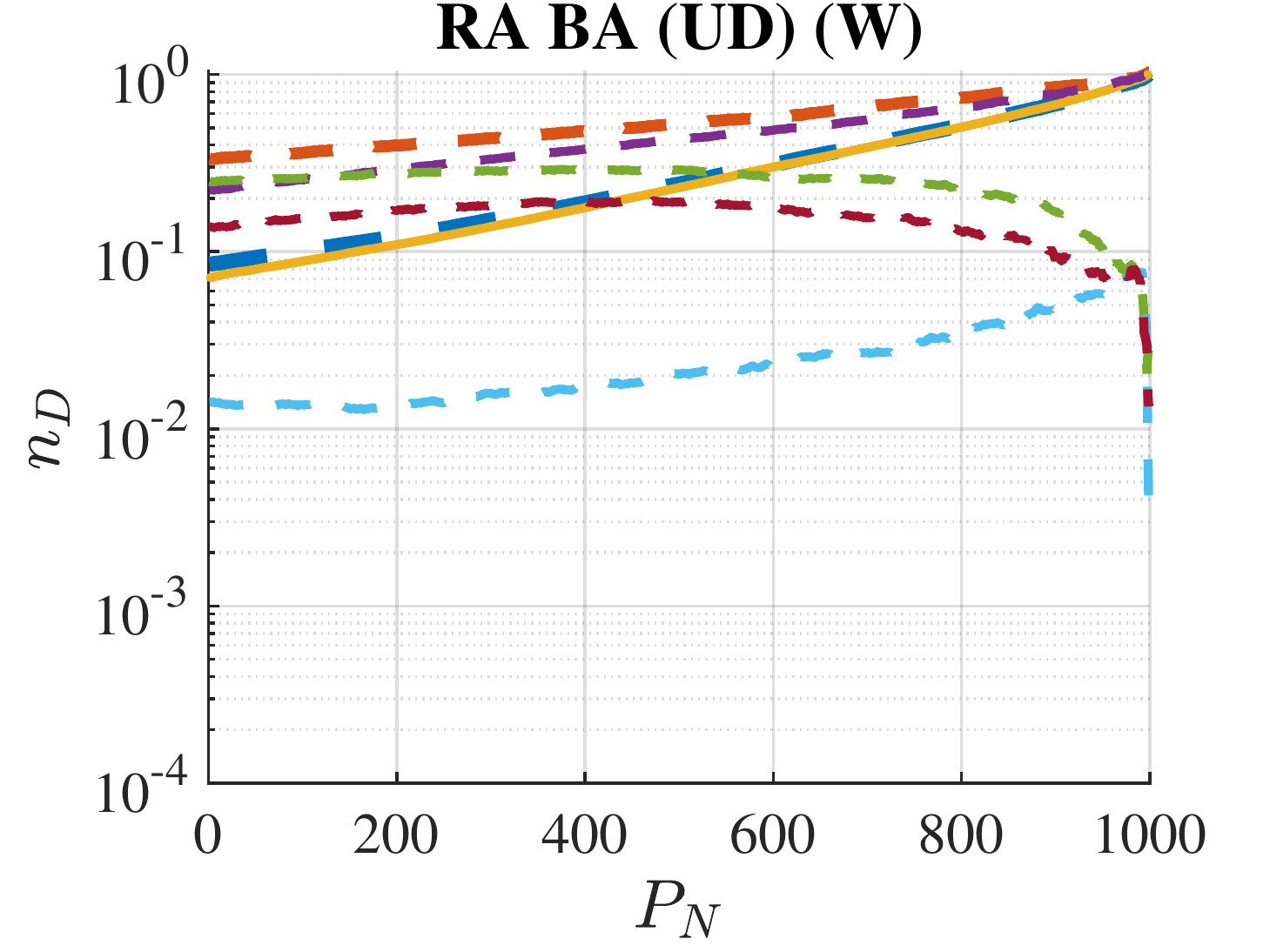}%
\label{a}}
\subfloat[]{\includegraphics[width=0.2\textwidth]{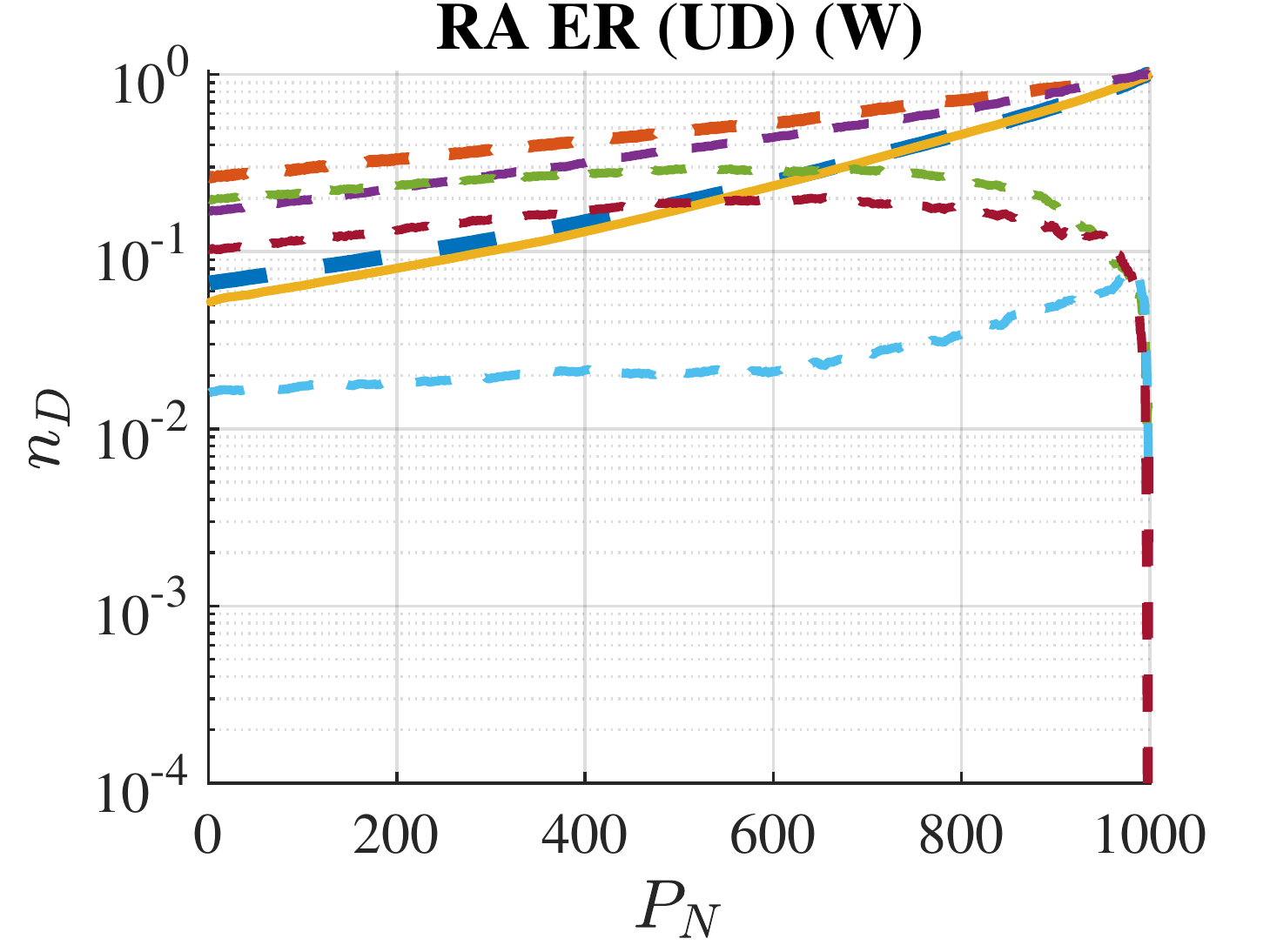}%
\label{b}}
\subfloat[]{\includegraphics[width=0.2\textwidth]{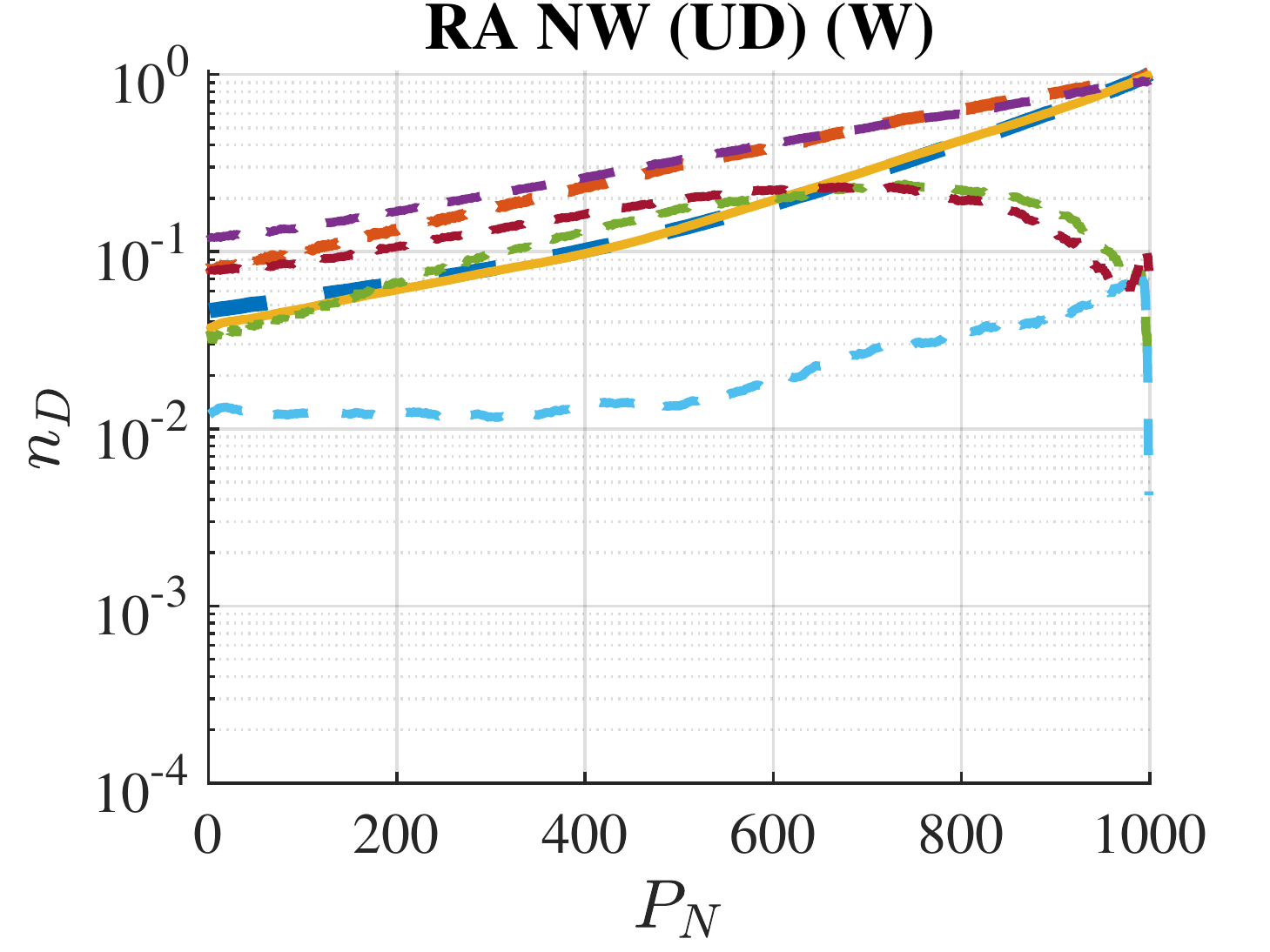}%
\label{c}}
\subfloat[]{\includegraphics[width=0.2\textwidth]{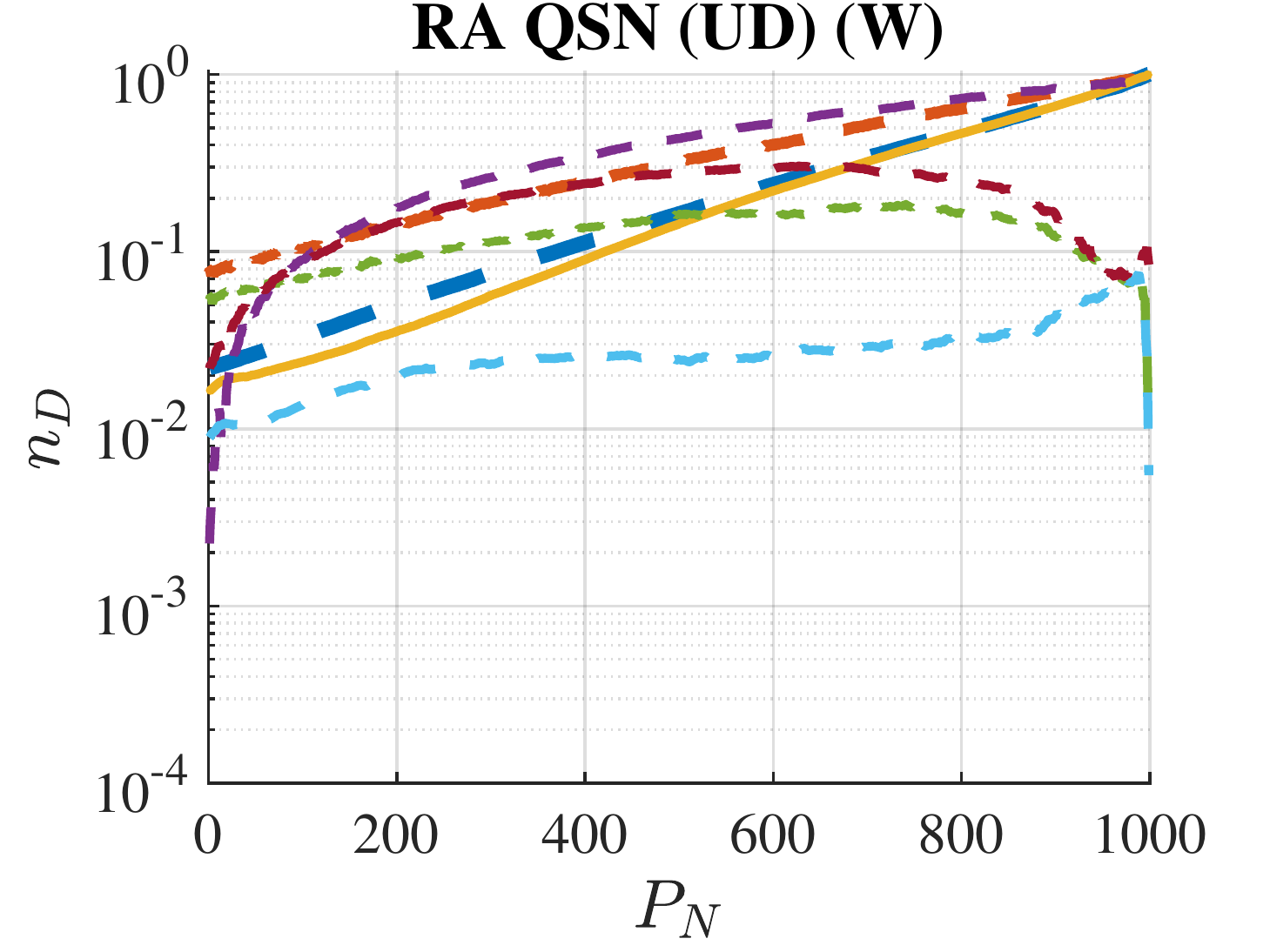}%
\label{d}}
\subfloat[]{\includegraphics[width=0.2\textwidth]{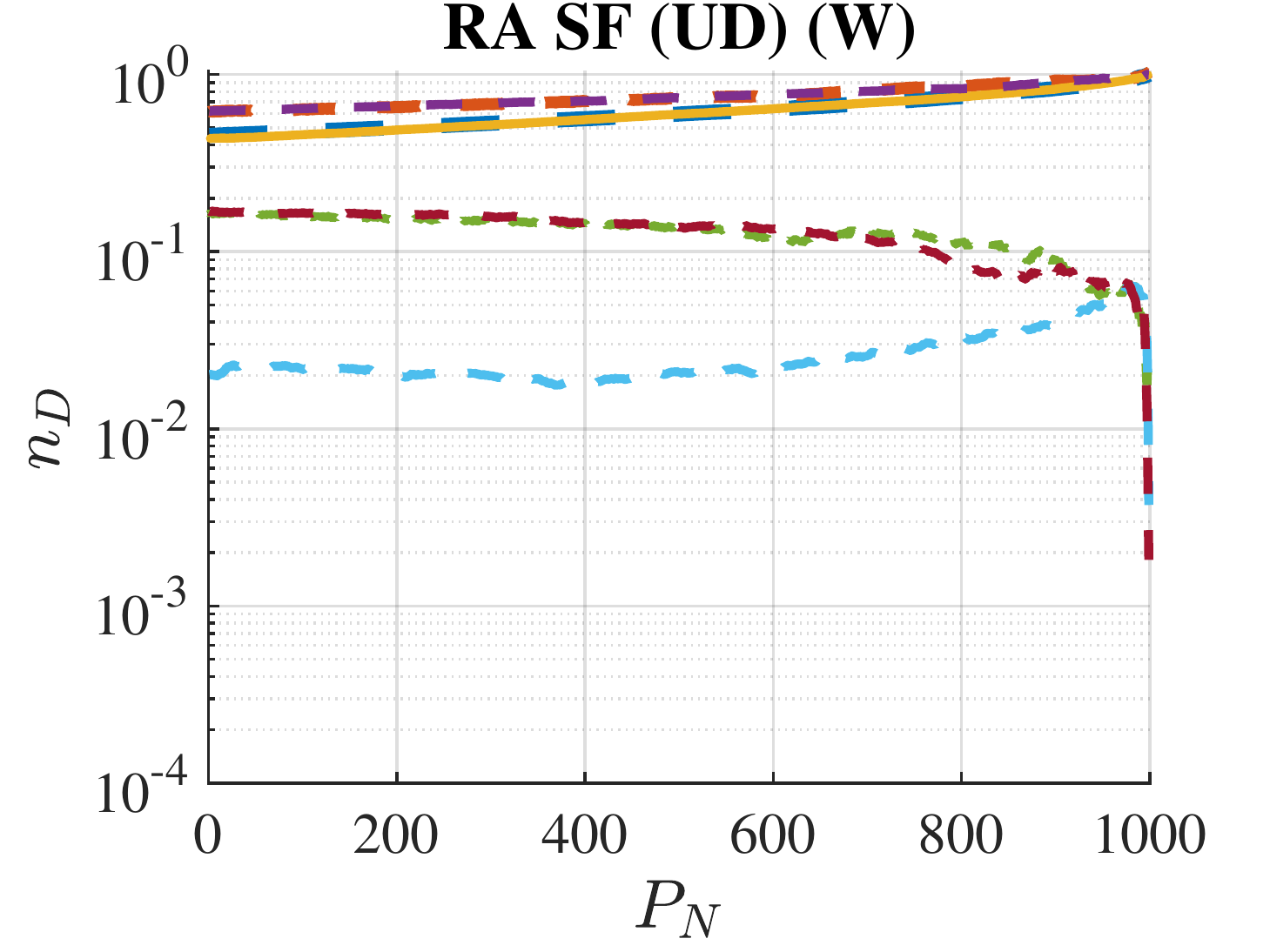}%
\label{d}}
\caption{Precision comparison and generalization ability evaluation of NRL-GT, PCR~\cite{31}, iPCR~\cite{32} for controllability robustness learning under RA. The experimental datasets include BA, ER, NW, QSN, and SF networks. In the title of each figure, D, UD, W, and UW indicate that the network is directed, undirected, weighted, and unweighted respectively. ${P_N}$ represents the number of nodes in the network that have been removed.
\label{fig_sim}}
\end{figure*}
\begin{table*}[!h]
\centering
\caption{Comparison of average errors of NRL-GT, PCR~\cite{31}, and iPCR~\cite{32} for controllability robustness learning under RA. The experimental datasets include directed and undirected, weighted and unweighted BA, ER, NW, QSN, and SF networks. \label{tab:table2}}
\begin{tabular}{|ccc|c|c|c|c|c|}
\hline
\multicolumn{3}{|c|}{Average Learning Error $\overline{\xi} $}                                                                                                                                                                                                                                                & BA             & ER             & NW             & QSN            & SF             \\ \hline
\multicolumn{1}{|c|}{\multirow{12}{*}{\begin{tabular}[c]{@{}c@{}}Controllability Robustness \\ of complex networks \\ under RA\end{tabular}}} & \multicolumn{1}{c|}{\multirow{3}{*}{\begin{tabular}[c]{@{}c@{}}Testing on directed\\ unweighted networks\end{tabular}}}                 & PCR    & \textbf{0.019} & 0.016          & 0.022          & 0.026          & 0.021          \\ \cline{3-8} 
\multicolumn{1}{|c|}{}                                                                                                                     & \multicolumn{1}{c|}{}                                                                                                                   & iPCR   & \textbf{0.019} & 0.016          & 0.016          & 0.017          & \textbf{0.020} \\ \cline{3-8} 
\multicolumn{1}{|c|}{}                                                                                                                     & \multicolumn{1}{c|}{}                                                                                                                   & NRL-GT & \textbf{0.019} & \textbf{0.015} & \textbf{0.014} & \textbf{0.014} & \textbf{0.020} \\ \cline{2-8} 
\multicolumn{1}{|c|}{}                                                                                                                     & \multicolumn{1}{c|}{\multirow{3}{*}{\begin{tabular}[c]{@{}c@{}}Generalization tests on \\ directed weighted networks\end{tabular}}}     & PCR    & 0.222          & 0.215          & 0.152          & 0.152          & 0.128          \\ \cline{3-8} 
\multicolumn{1}{|c|}{}                                                                                                                     & \multicolumn{1}{c|}{}                                                                                                                   & iPCR   & 0.153          & 0.153          & 0.114          & 0.079          & 0.128          \\ \cline{3-8} 
\multicolumn{1}{|c|}{}                                                                                                                     & \multicolumn{1}{c|}{}                                                                                                                   & NRL-GT & \textbf{0.019} & \textbf{0.015} & \textbf{0.014} & \textbf{0.015} & \textbf{0.020} \\ \cline{2-8} 
\multicolumn{1}{|c|}{}                                                                                                                     & \multicolumn{1}{c|}{\multirow{3}{*}{\begin{tabular}[c]{@{}c@{}}Generalization tests on \\ undirected unweighted networks\end{tabular}}} & PCR    & \textbf{0.023} & \textbf{0.022} & 0.072          & 0.040          & 0.024          \\ \cline{3-8} 
\multicolumn{1}{|c|}{}                                                                                                                     & \multicolumn{1}{c|}{}                                                                                                                   & iPCR   & 0.025          & 0.023          & 0.063          & 0.149          & \textbf{0.023} \\ \cline{3-8} 
\multicolumn{1}{|c|}{}                                                                                                                     & \multicolumn{1}{c|}{}                                                                                                                   & NRL-GT & 0.024          & 0.026          & \textbf{0.022} & \textbf{0.027} & 0.026          \\ \cline{2-8} 
\multicolumn{1}{|c|}{}                                                                                                                     & \multicolumn{1}{c|}{\multirow{3}{*}{\begin{tabular}[c]{@{}c@{}}Generalization tests on \\ undirected weighted networks\end{tabular}}}   & PCR    & 0.247          & 0.241          & 0.139          & 0.126          & 0.129          \\ \cline{3-8} 
\multicolumn{1}{|c|}{}                                                                                                                     & \multicolumn{1}{c|}{}                                                                                                                   & iPCR   & 0.156          & 0.154          & 0.154          & 0.201          & 0.129          \\ \cline{3-8} 
\multicolumn{1}{|c|}{}                                                                                                                     & \multicolumn{1}{c|}{}                                                                                                                   & NRL-GT & \textbf{0.025} & \textbf{0.027} & \textbf{0.023} & \textbf{0.027} & \textbf{0.026} \\ \hline
\end{tabular}
\vspace{-0.2cm}
\end{table*}

\vspace{-0.4cm}
\subsection{Network Connectivity Robustness Learning}
The robustness of connectivity for undirected networks under RA and the robustness of weak connectivity for directed networks under TBA is learned based on NRL-GT in this subsection. For different properties of connectivity robustness, the generalization ability is also tested from unweighted networks to weighted networks. We compare with CNN-RP~\cite{68} and mCNN-RP~\cite{77}, which are landmark models of network connectivity robustness learning. For each topology under RA or TBA, 2000 unweighted samples are generated for training, and 200 weighted and unweighted samples are generated for testing. The range of average degree $\left\langle k \right\rangle $ of each synthetic topology is set as [2,20] for undirected networks (with an overall average degree of 11)  and [1,10] for directed networks (with an overall average degree of 5.5).

The test results are presented in Fig. 4 and Table III. It can be seen that when the distributions of training and test sets are consistent, the error curves of NRL-GT are lower than those of CNN-RP and mCNN-RP for most topologies, both under RA and TBA (see Figs. 4(b)-4(d), Figs. 4(k)-4(m) and Fig. 4(o)). The corresponding positions in Table III show that the mean errors of NRL-GT are lower than CNN-RP and mCNN-RP, demonstrating the higher precision of NRL-GT in network connectivity robustness learning. Since the edge weights also have no effect on network connectivity robustness and NRL-GT is also independent of edge weights in capturing topological information that determines the robustness of network connectivity, NRL-GT trained on unweighted networks is still able to achieve high-precision performance on weighted networks, shown in Figs. 4(f)-4(j) and Figs. 4(p)-4(t). However, CNN-RP and mCNN-RP trained on unweighted networks are completely distorted on weighted networks, illustrating poor generalization of CNN-RP and mCNN-RP.
\begin{figure*}[h]
\centering
\subfloat[]{\includegraphics[width=0.2\textwidth]{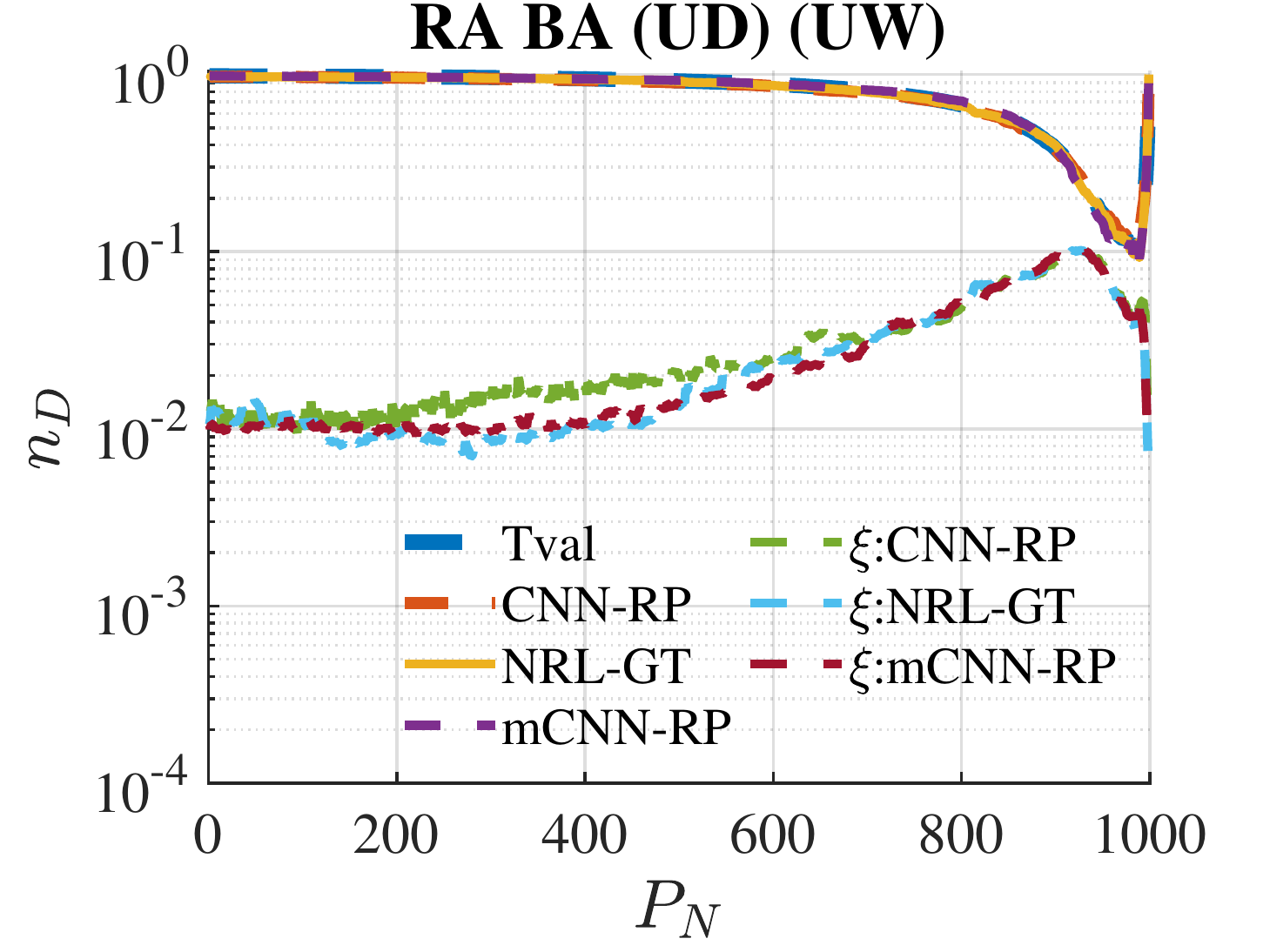}%
\label{a}}
\subfloat[]{\includegraphics[width=0.2\textwidth]{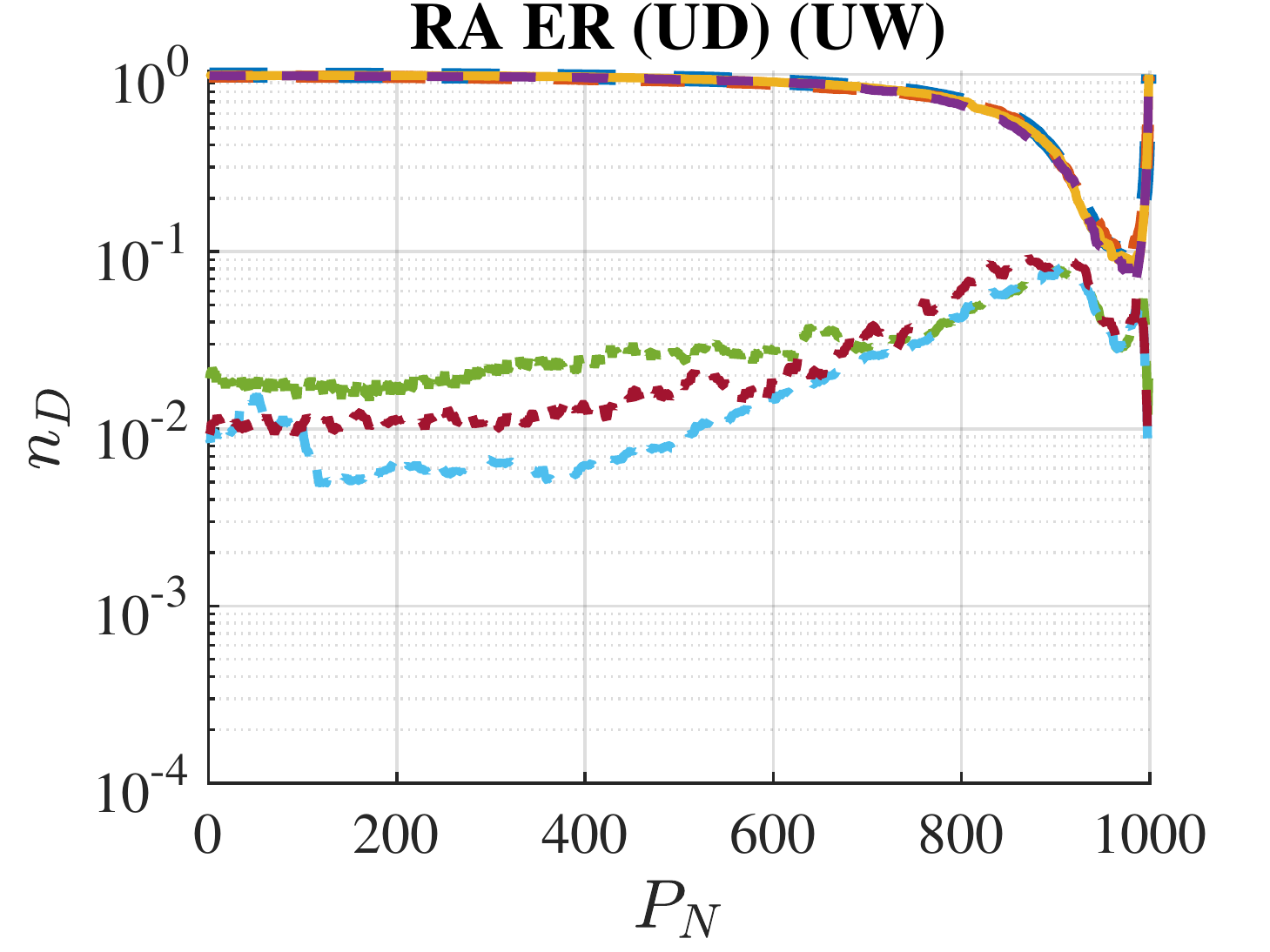}%
\label{b}}
\subfloat[]{\includegraphics[width=0.2\textwidth]{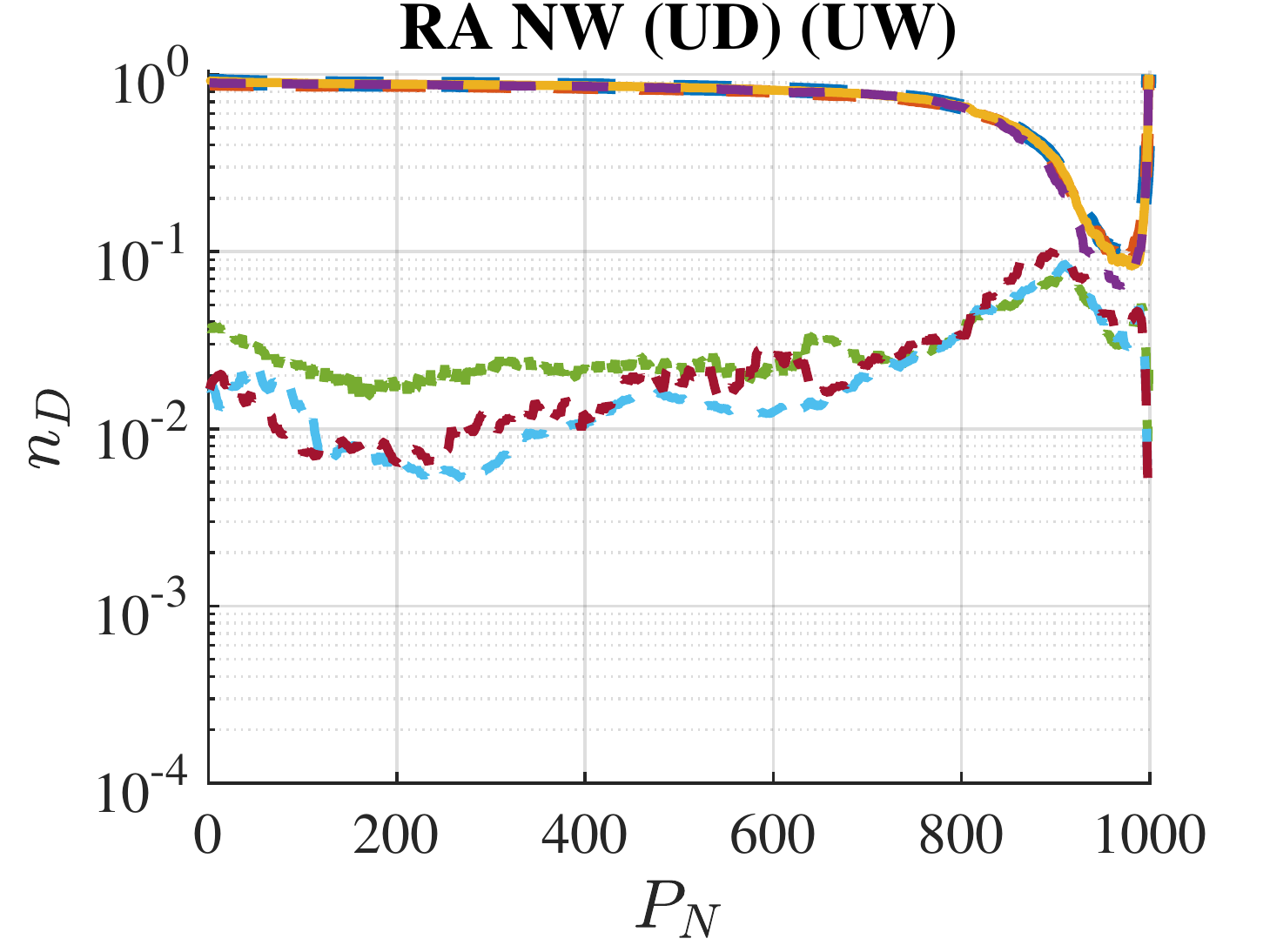}%
\label{c}}
\subfloat[]{\includegraphics[width=0.2\textwidth]{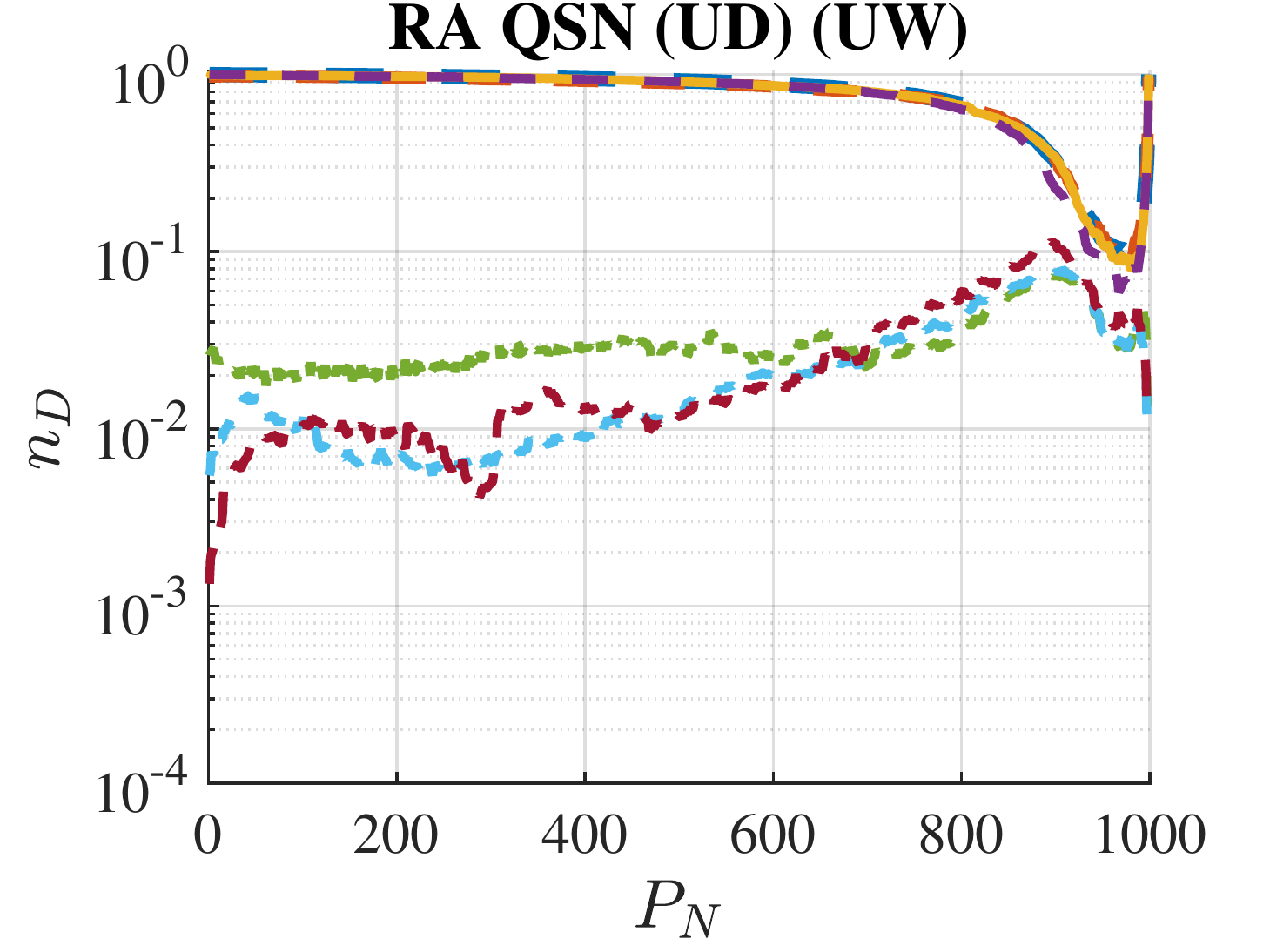}%
\label{d}}
\subfloat[]{\includegraphics[width=0.2\textwidth]{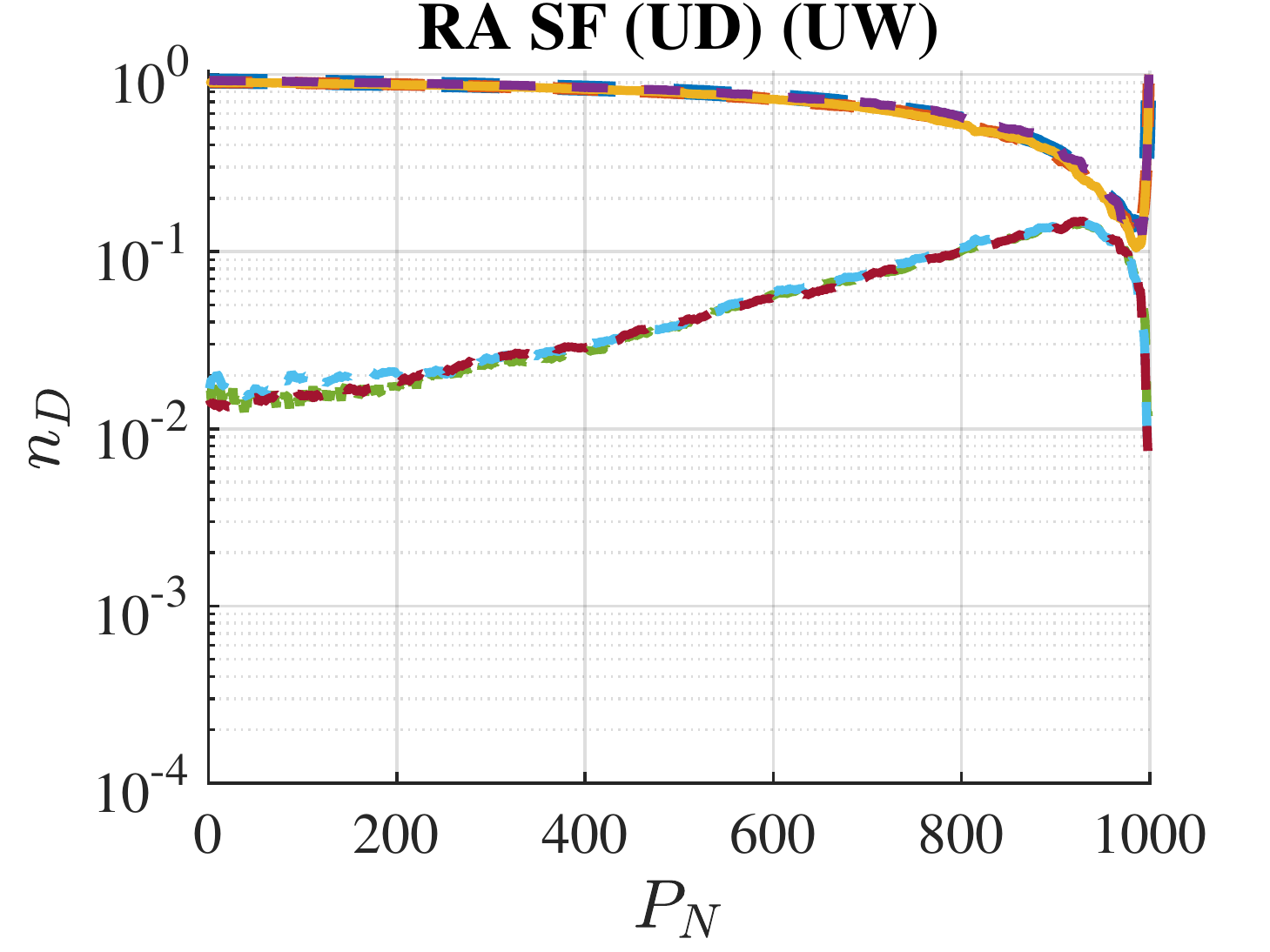}%
\label{d}}

\subfloat[]{\includegraphics[width=0.2\textwidth]{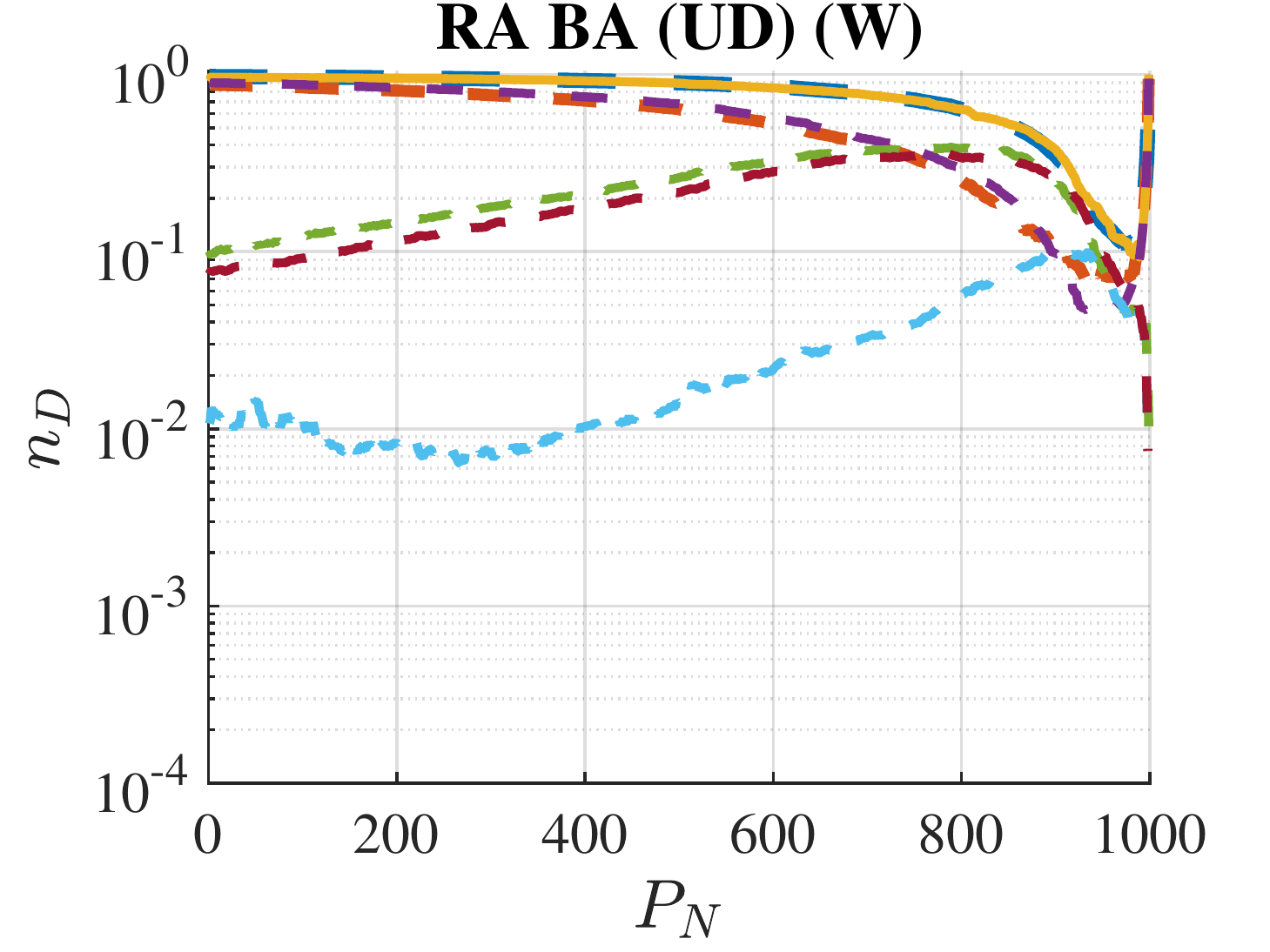}%
\label{a}}
\subfloat[]{\includegraphics[width=0.2\textwidth]{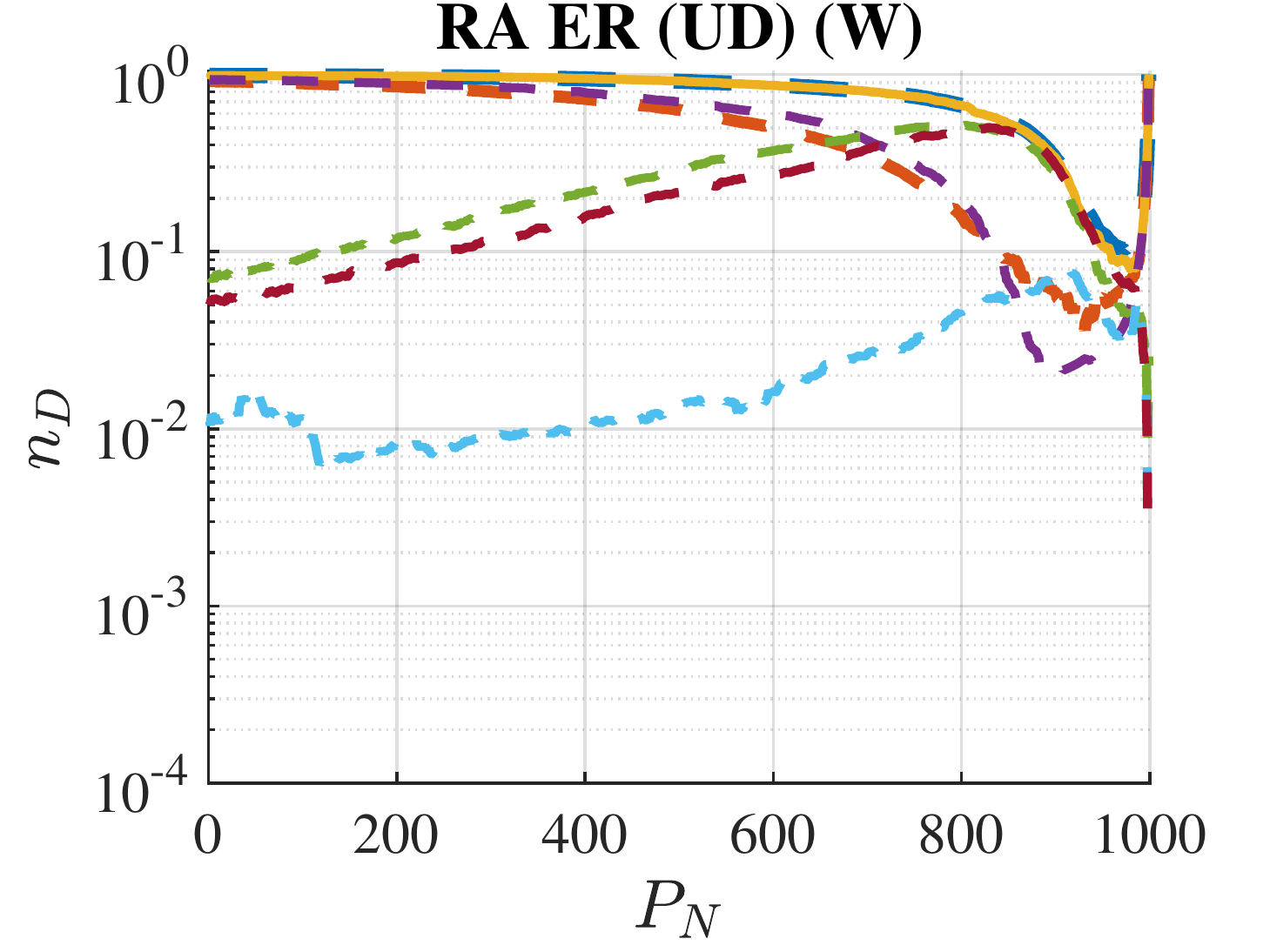}%
\label{b}}
\subfloat[]{\includegraphics[width=0.2\textwidth]{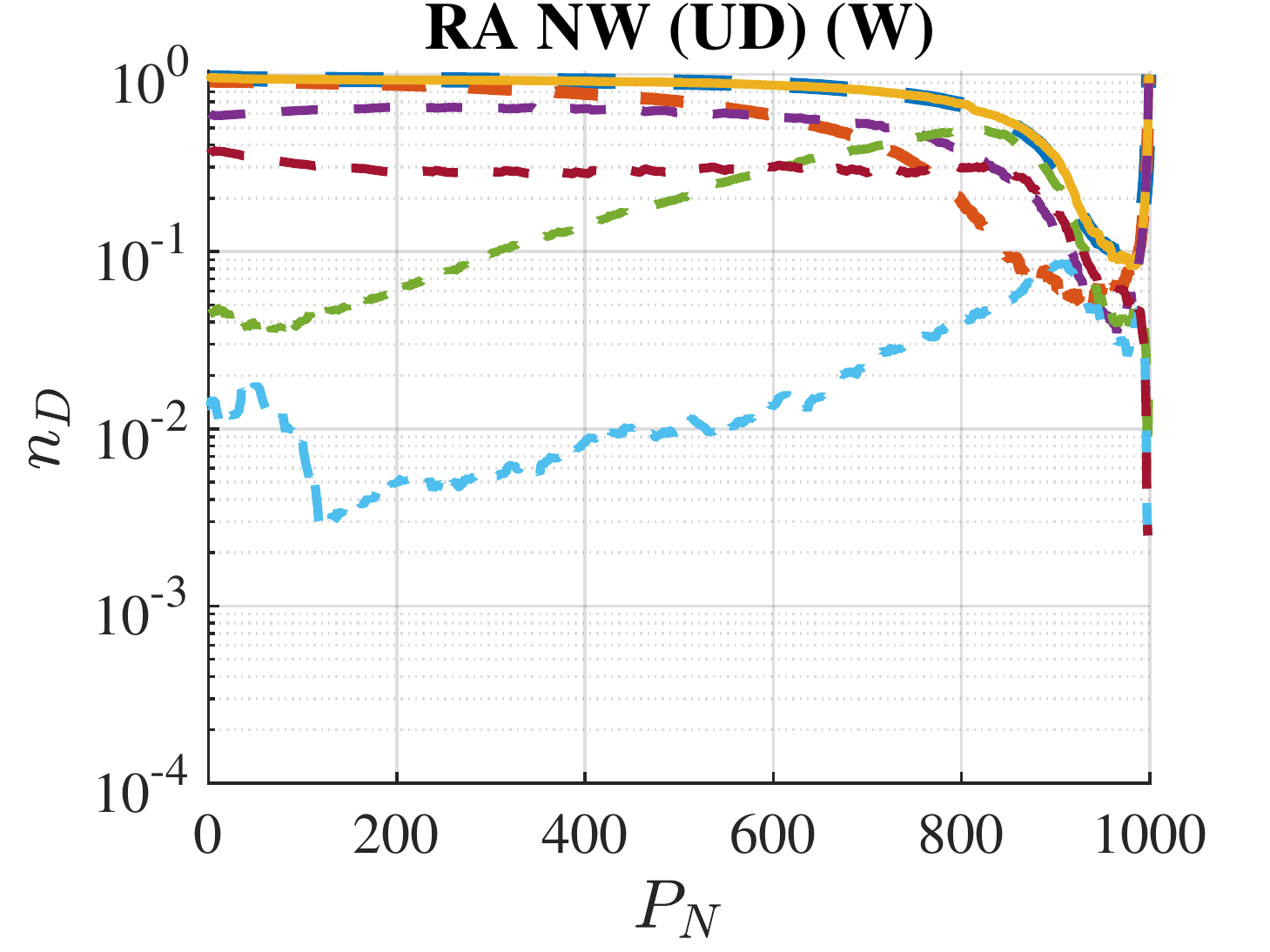}%
\label{c}}
\subfloat[]{\includegraphics[width=0.2\textwidth]{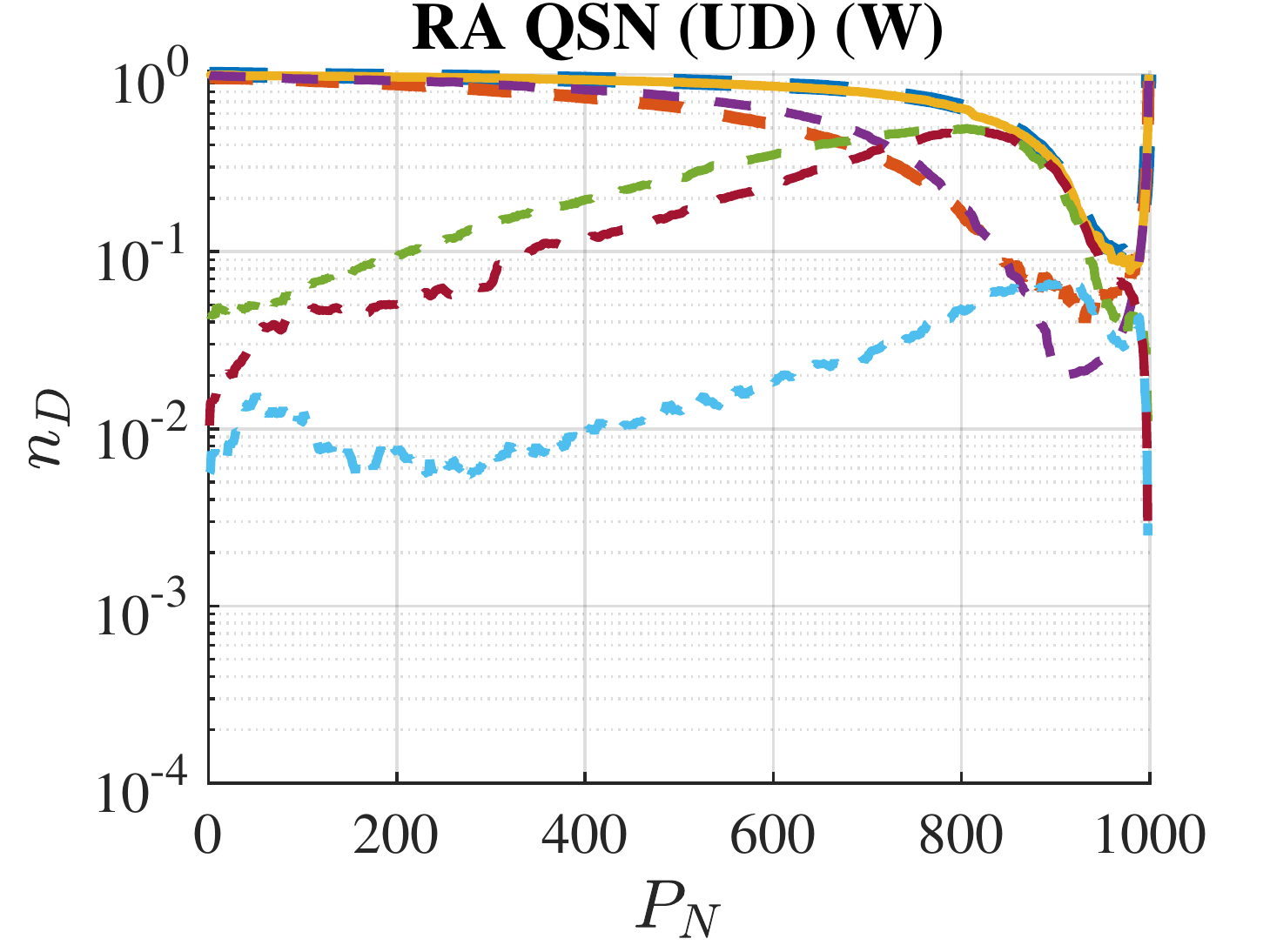}%
\label{d}}
\subfloat[]{\includegraphics[width=0.2\textwidth]{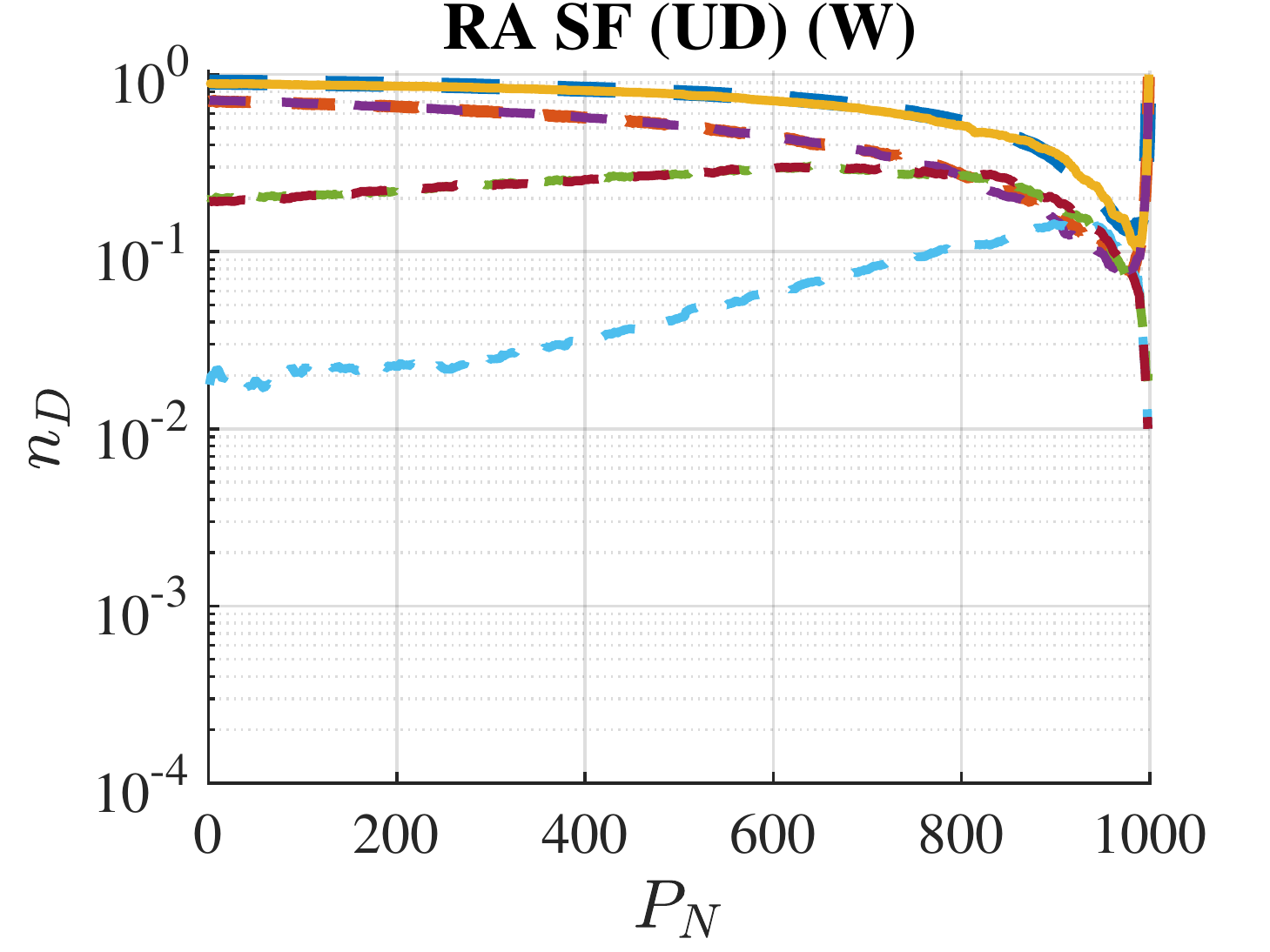}%
\label{d}}

\subfloat[]{\includegraphics[width=0.2\textwidth]{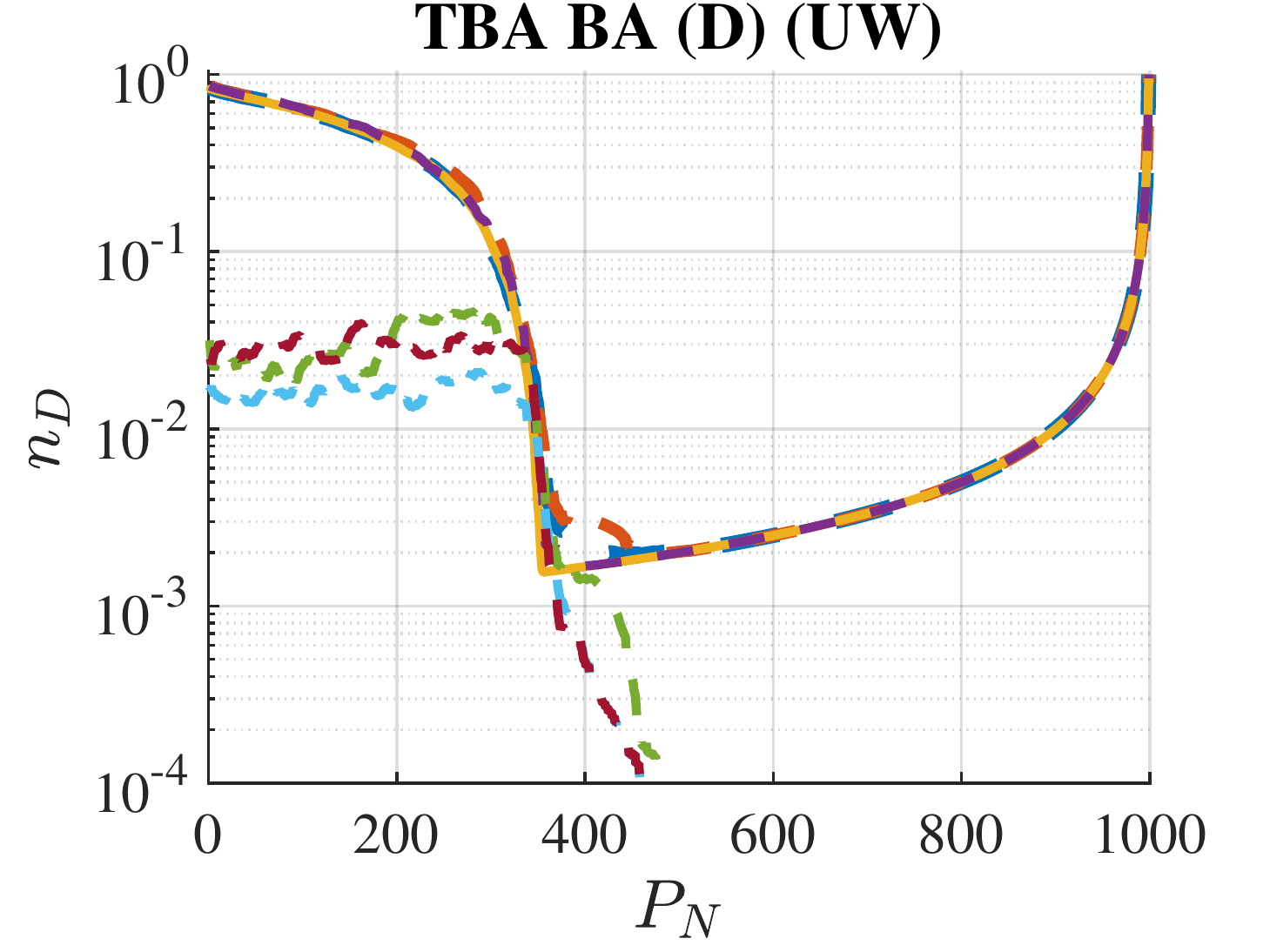}%
\label{a}}
\subfloat[]{\includegraphics[width=0.2\textwidth]{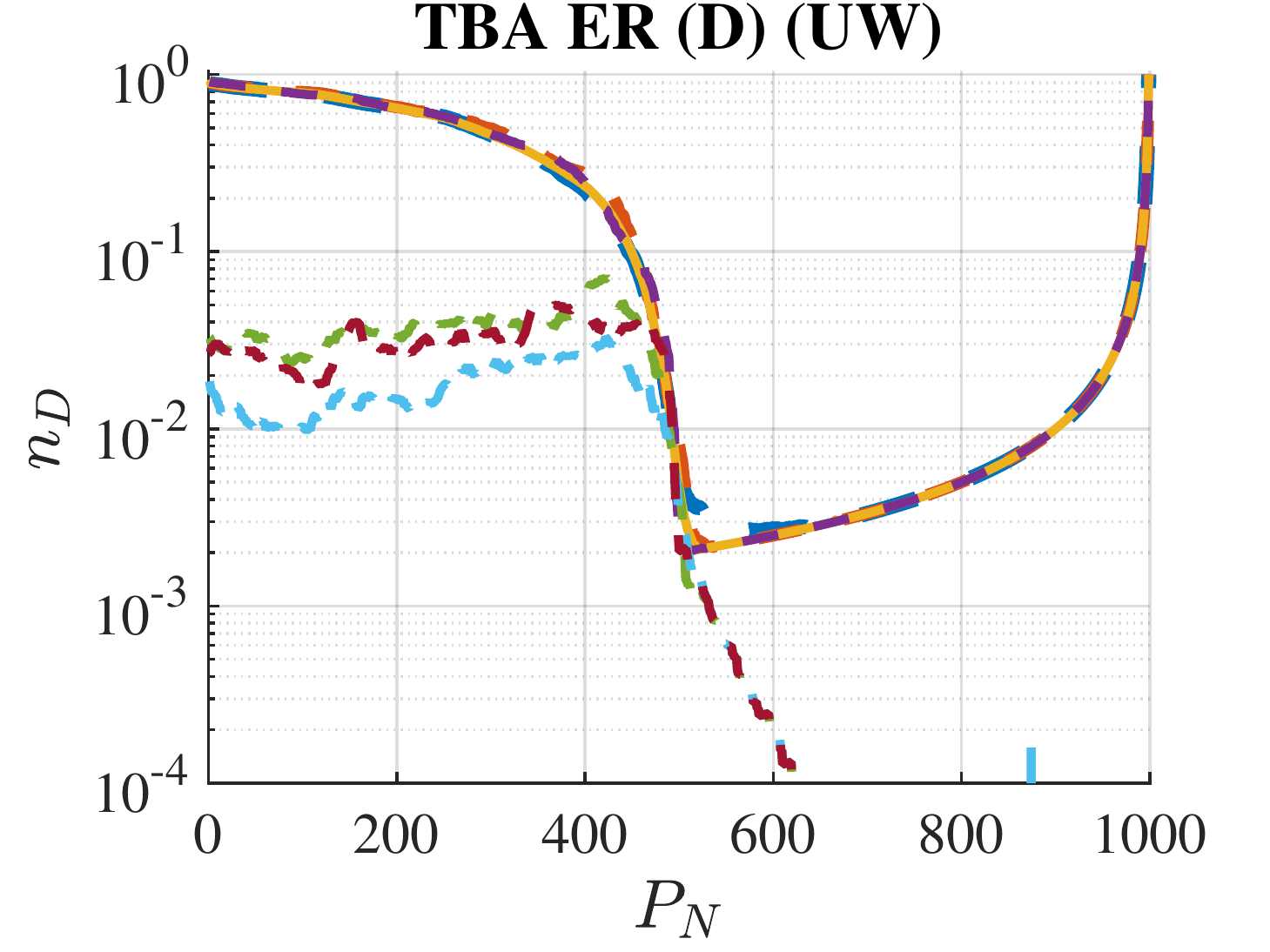}%
\label{b}}
\subfloat[]{\includegraphics[width=0.2\textwidth]{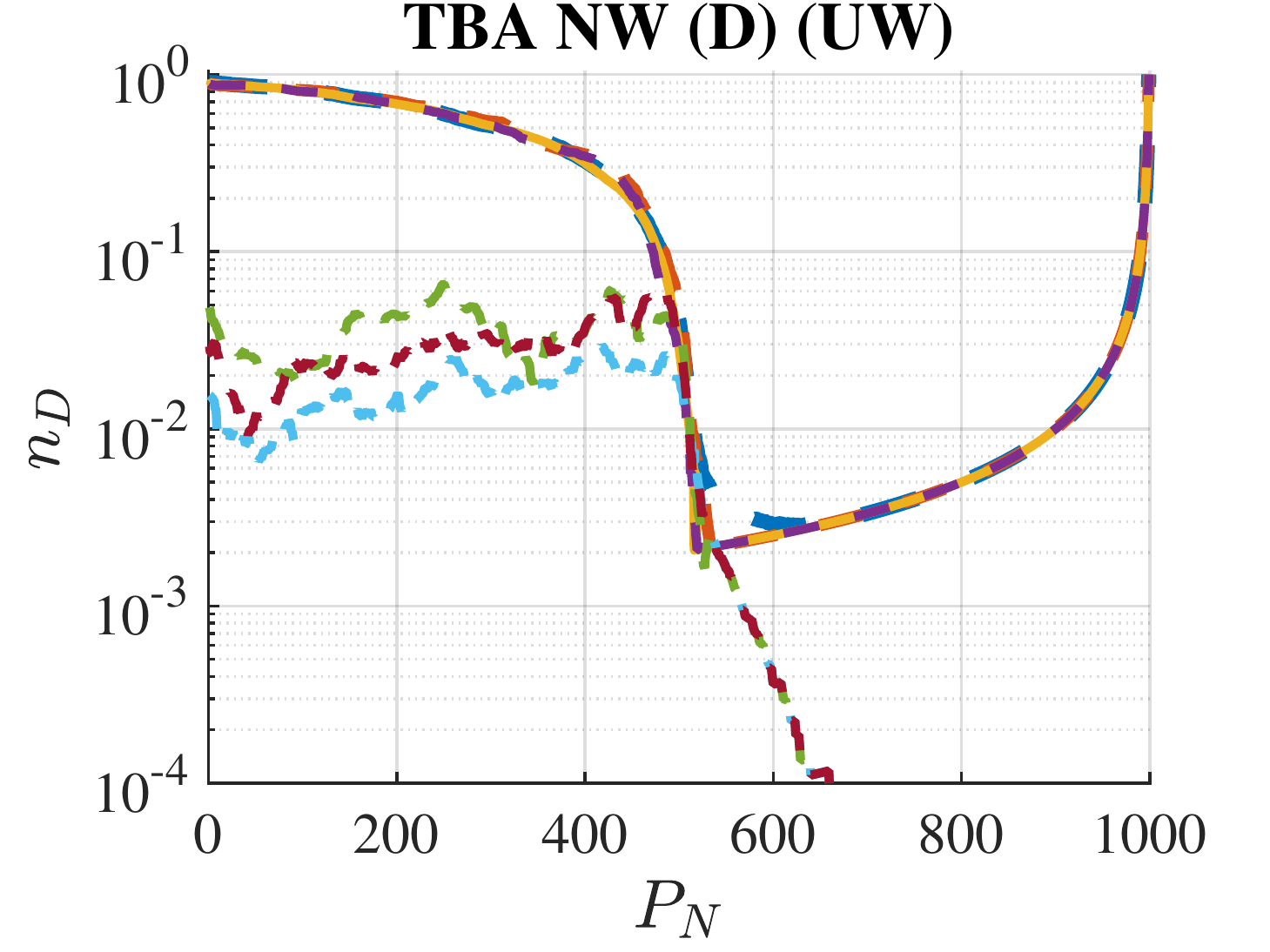}%
\label{c}}
\subfloat[]{\includegraphics[width=0.2\textwidth]{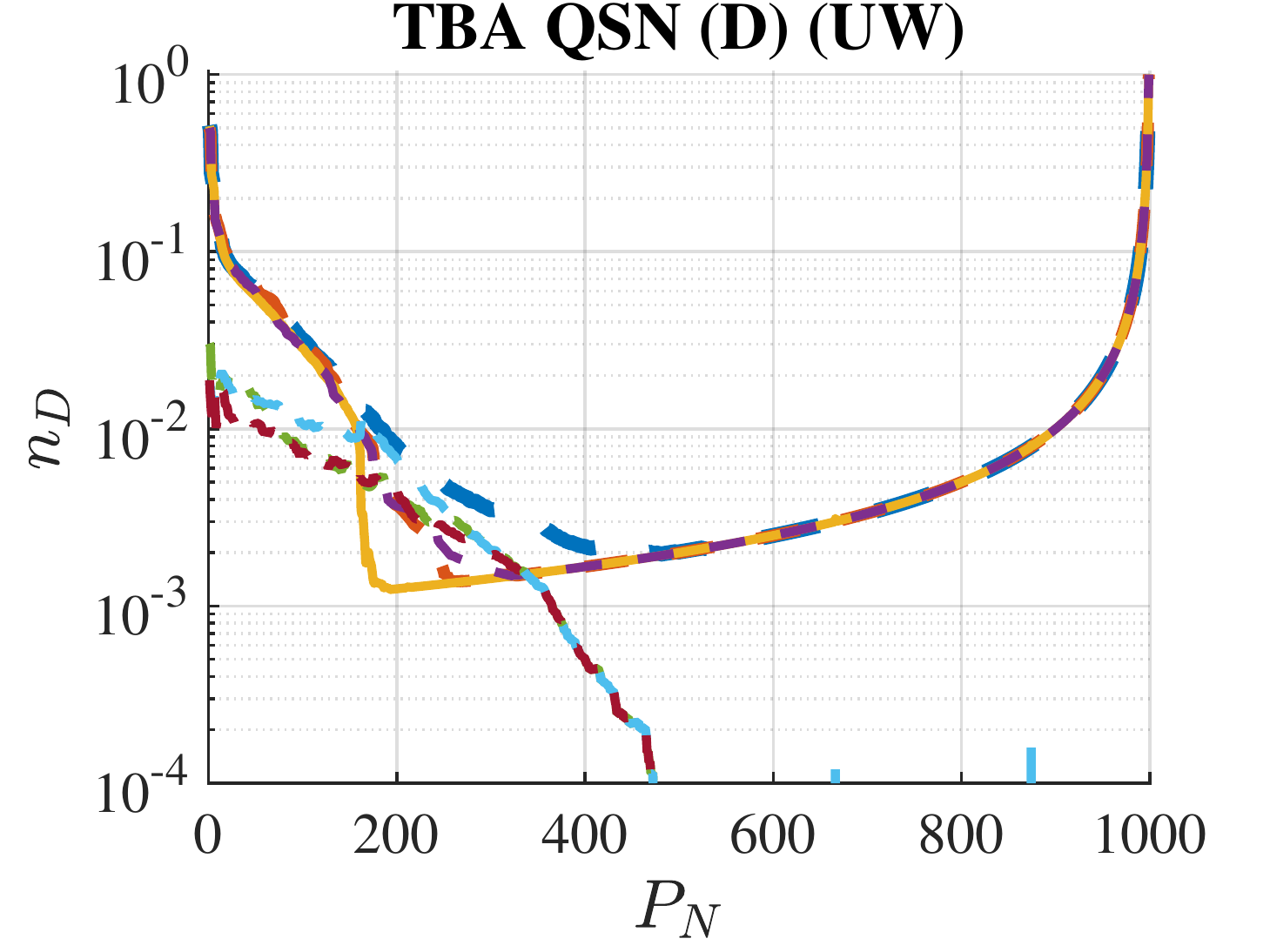}%
\label{d}}
\subfloat[]{\includegraphics[width=0.2\textwidth]{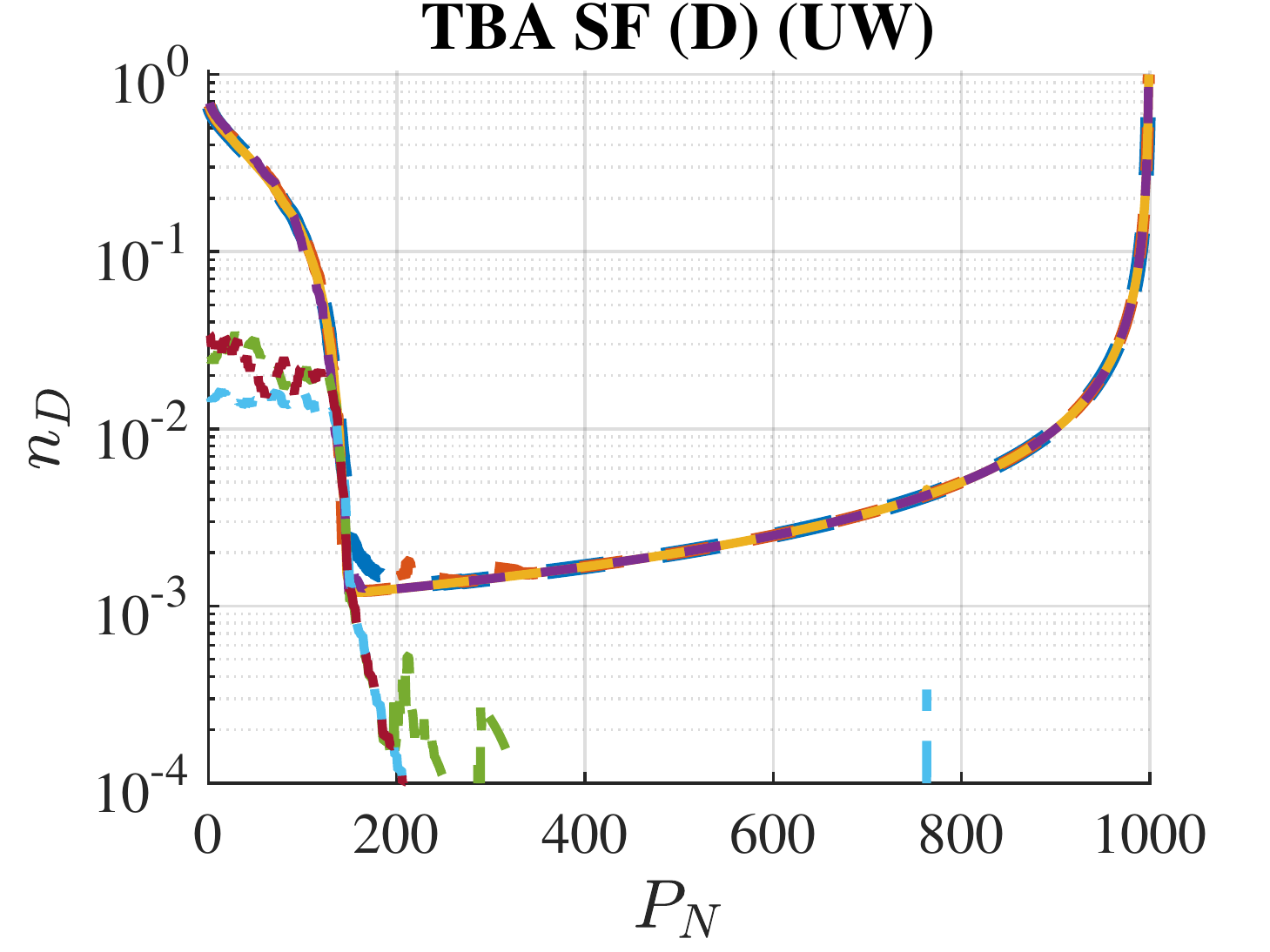}%
\label{d}}

\subfloat[]{\includegraphics[width=0.2\textwidth]{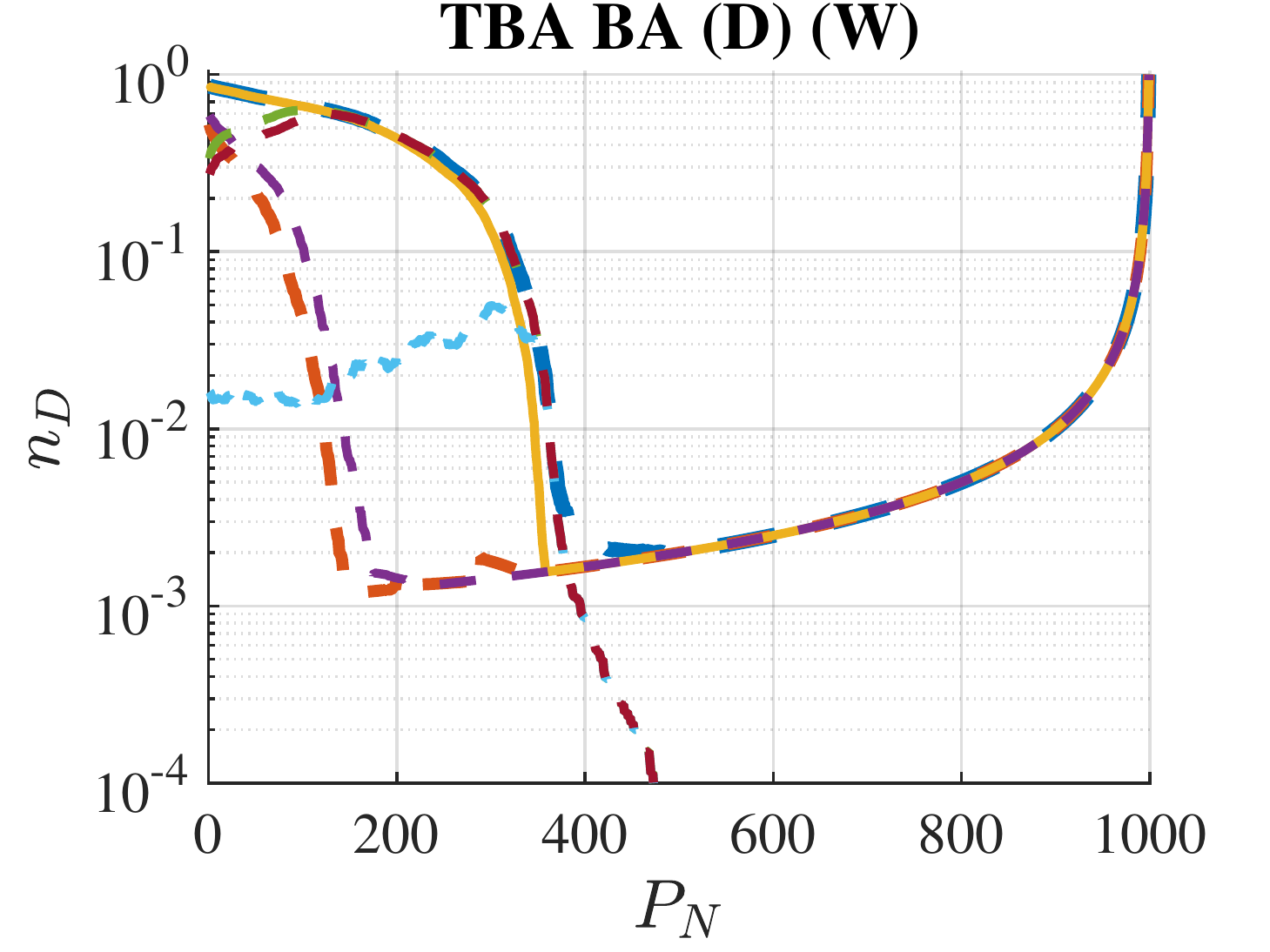}%
\label{a}}
\subfloat[]{\includegraphics[width=0.2\textwidth]{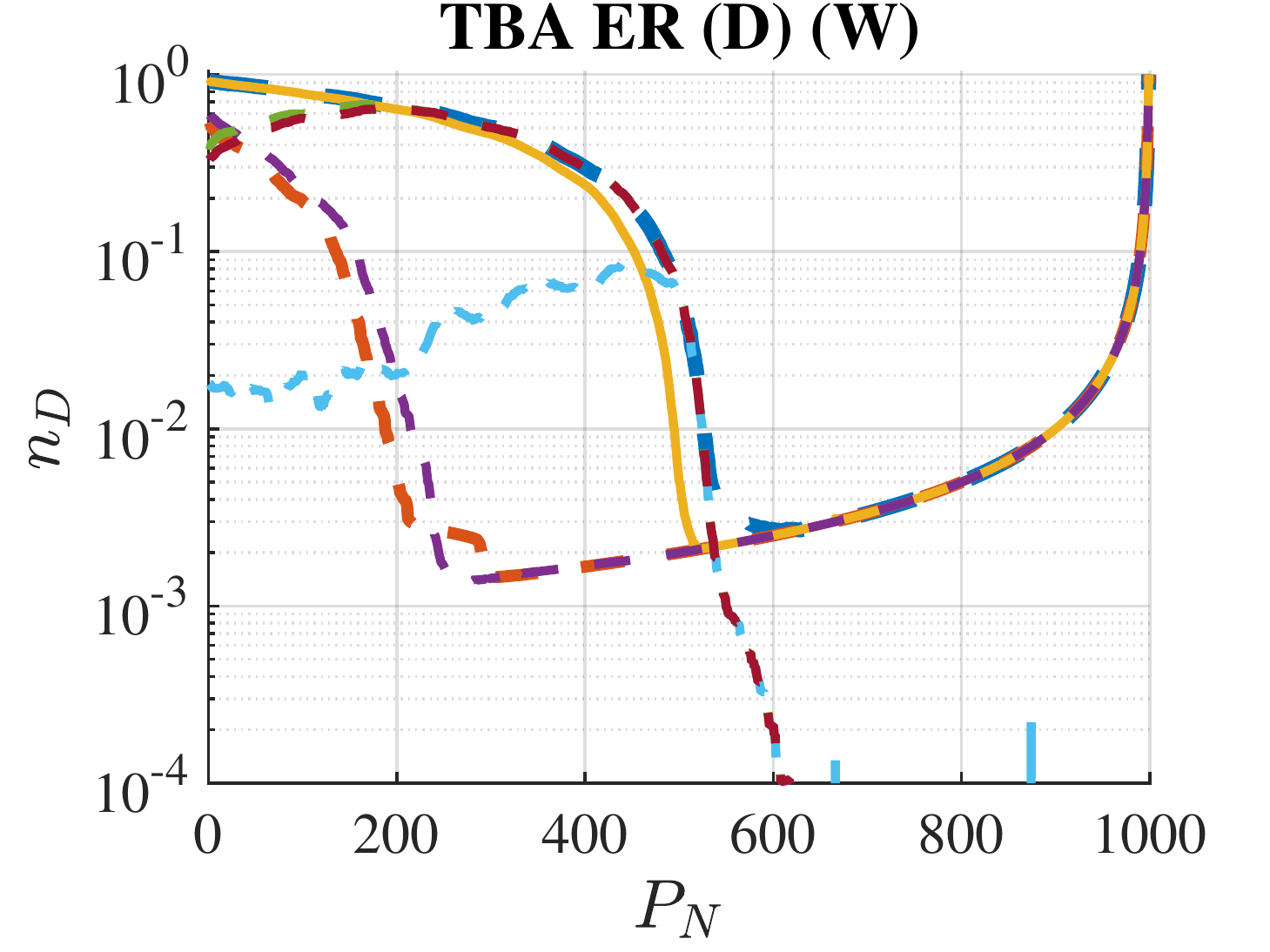}%
\label{b}}
\subfloat[]{\includegraphics[width=0.2\textwidth]{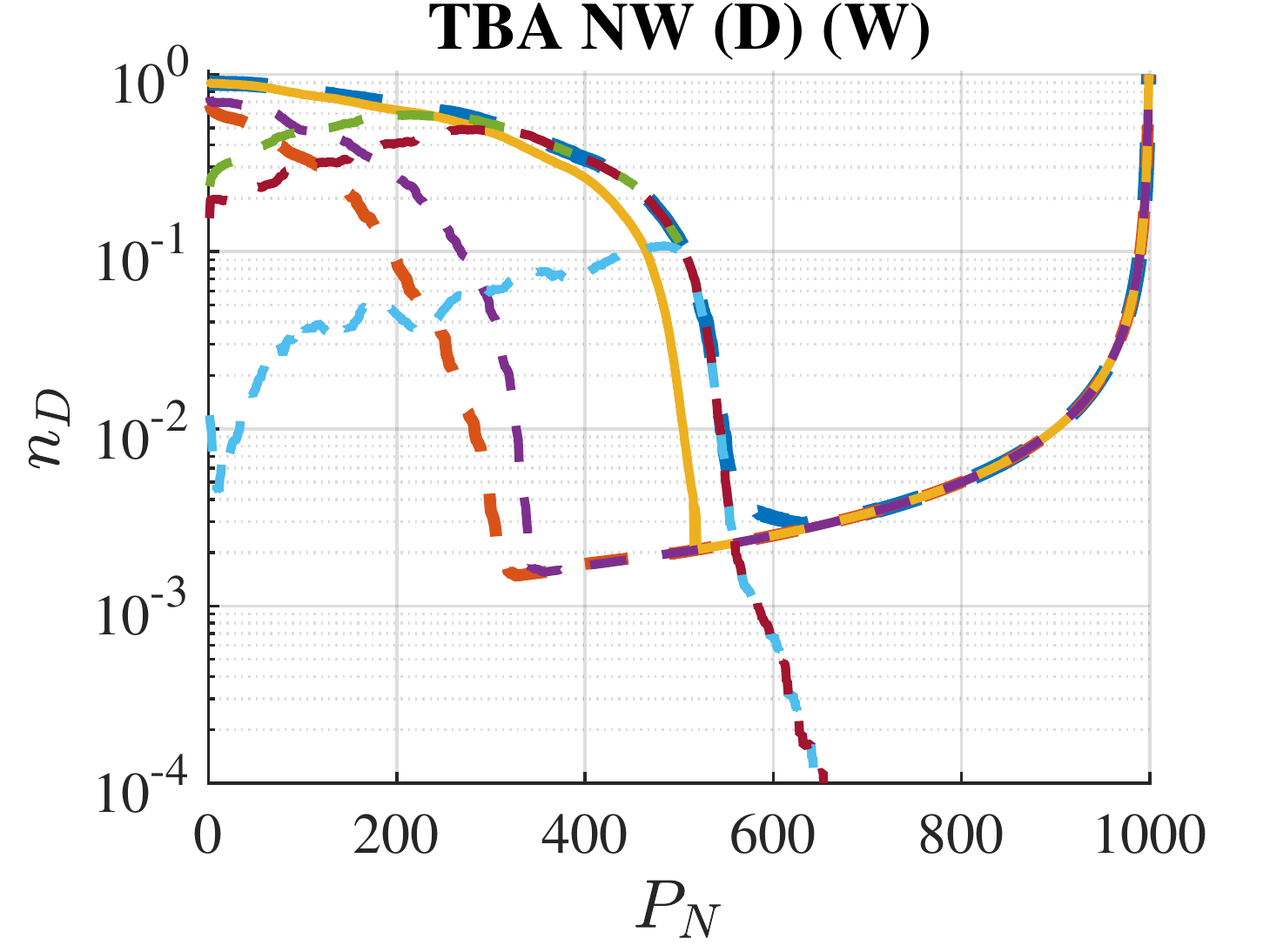}%
\label{c}}
\subfloat[]{\includegraphics[width=0.2\textwidth]{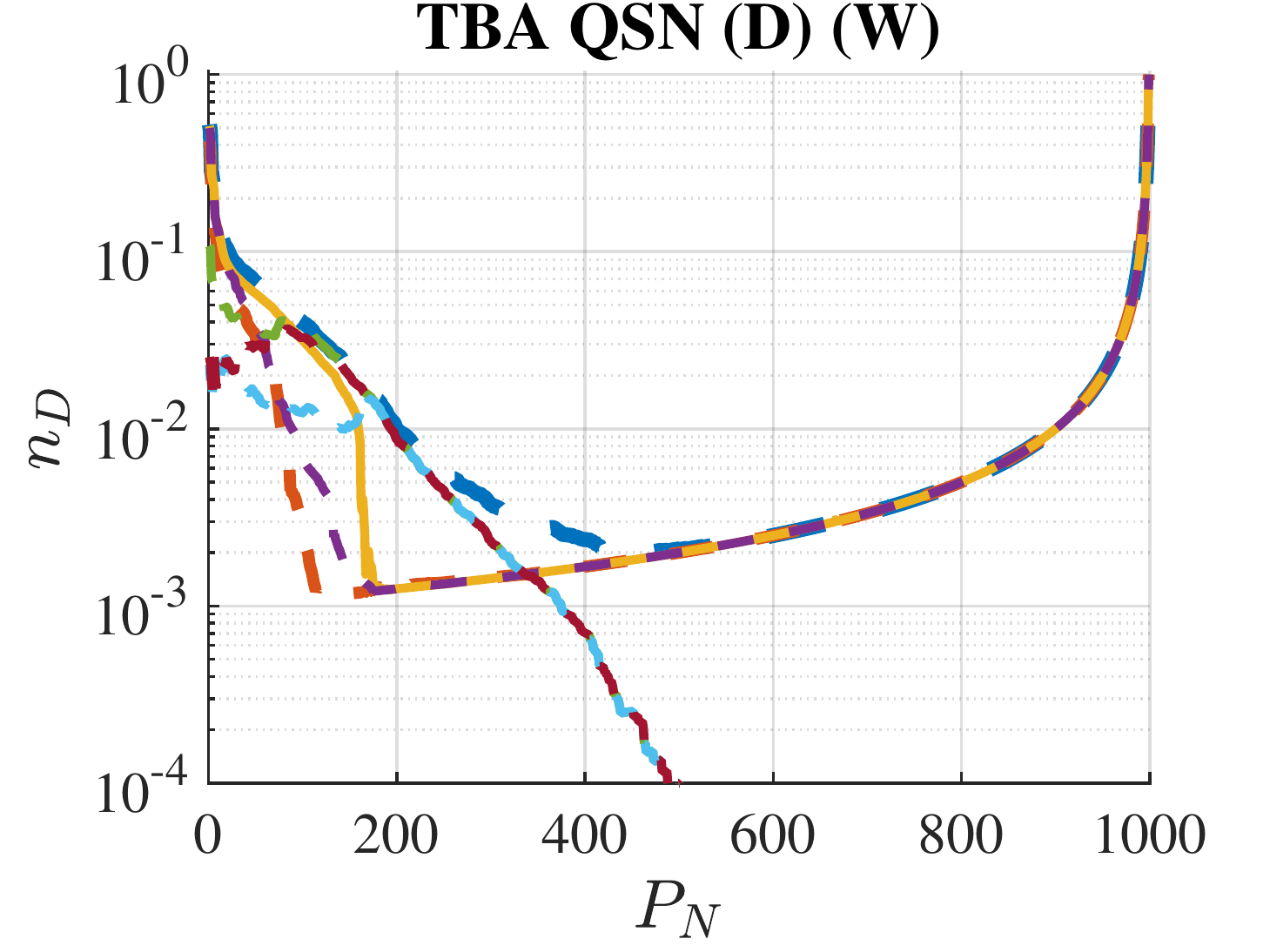}%
\label{d}}
\subfloat[]{\includegraphics[width=0.2\textwidth]{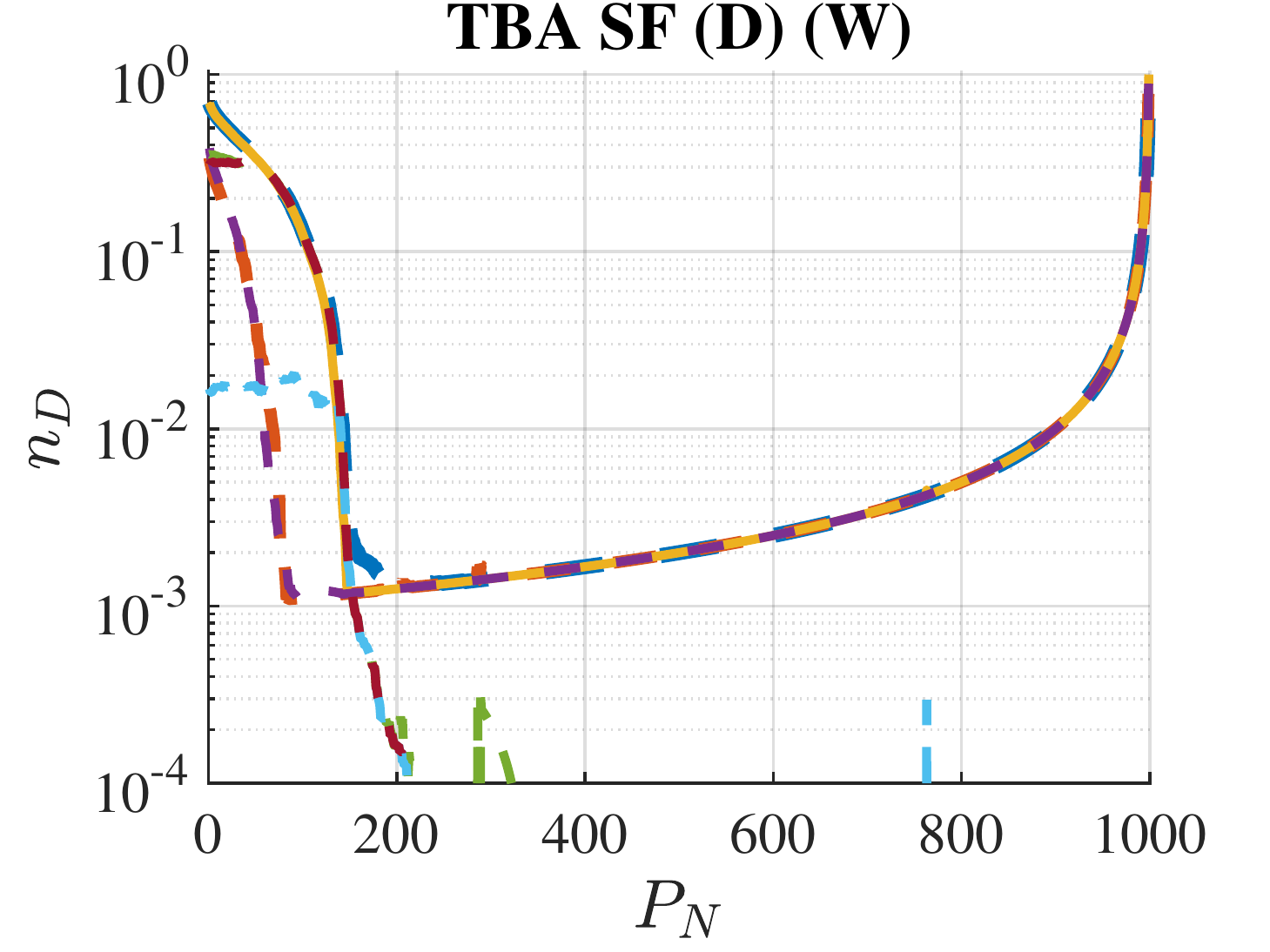}%
\label{d}}
\caption{Precision comparison and generalization ability evaluation of NRL-GT, CNN-RP~\cite{68}, and mCNN-RP~\cite{77} for connectivity robustness learning under RA and TBA. The experimental datasets include directed and undirected, weighted and unweighted BA, ER, NW, QSN, and SF networks.}
\label{fig_sim}
\end{figure*}

\begin{table*}[!h]
\centering
\caption{Comparison of average errors of NRL-GT, CNN-RP~\cite{68} and mCNN-RP~\cite{77} for connectivity robustness learning under RA and TBA. The experimental datasets include directed and undirected, weighted and unweighted BA, ER, NW, QSN, and SF networks.\label{tab:table2}}
\begin{tabular}{|ccc|c|c|c|c|c|}
\hline
\multicolumn{3}{|c|}{Average Learning Error $\overline{\xi} $}                                                                                                                                                                                                                            & BA             & ER             & NW             & QSN            & SF             \\ \hline
\multicolumn{1}{|c|}{\multirow{6}{*}{\begin{tabular}[c]{@{}c@{}}Connectivity Robustness\\ of complex networks\\ under RA\end{tabular}}}  & \multicolumn{1}{c|}{\multirow{3}{*}{\begin{tabular}[c]{@{}c@{}}Testing on undirected\\ unweighted networks\end{tabular}}}            & CNN-RP  & 0.030           & 0.031          & 0.029          & 0.031          & \textbf{0.054} \\ \cline{3-8} 
\multicolumn{1}{|c|}{}                                                                                                                   & \multicolumn{1}{c|}{}                                                                                                                & mCNN-RP & \textbf{0.028} & 0.028          &0.025          & 0.027          & \textbf{0.054} \\ \cline{3-8} 
\multicolumn{1}{|c|}{}                                                                                                                   & \multicolumn{1}{c|}{}                                                                                                                & NRL-GT  & \textbf{0.028} & \textbf{0.021} & \textbf{0.021} & \textbf{0.023} & 0.056          \\ \cline{2-8} 
\multicolumn{1}{|c|}{}                                                                                                                   & \multicolumn{1}{c|}{\multirow{3}{*}{\begin{tabular}[c]{@{}c@{}}Generalization tests on\\ undirected weighted networks\end{tabular}}} & CNN-RP  & 0.232          & 0.254          & 0.198          & 0.233          & 0.237          \\ \cline{3-8} 
\multicolumn{1}{|c|}{}                                                                                                                   & \multicolumn{1}{c|}{}                                                                                                                & mCNN-RP & 0.202          & 0.212          & 0.271          & 0.183          & 0.238          \\ \cline{3-8} 
\multicolumn{1}{|c|}{}                                                                                                                   & \multicolumn{1}{c|}{}                                                                                                                & NRL-GT  & \textbf{0.028} & \textbf{0.022} & \textbf{0.020} & \textbf{0.022} & \textbf{0.058} \\ \hline
\multicolumn{1}{|c|}{\multirow{6}{*}{\begin{tabular}[c]{@{}c@{}}Connectivity Robustness\\ of complex networks\\ under TBA\end{tabular}}} & \multicolumn{1}{c|}{\multirow{3}{*}{\begin{tabular}[c]{@{}c@{}}Testing on directed\\ unweighted networks\end{tabular}}}              & CNN-RP  & 0.011          & 0.019          & 0.019          & 0.003          & 0.003          \\ \cline{3-8} 
\multicolumn{1}{|c|}{}                                                                                                                   & \multicolumn{1}{c|}{}                                                                                                                & mCNN-RP & 0.010          & 0.016          & 0.015          & \textbf{0.002} & 0.003          \\ \cline{3-8} 
\multicolumn{1}{|c|}{}                                                                                                                   & \multicolumn{1}{c|}{}                                                                                                                & NRL-GT  & \textbf{0.006} & \textbf{0.009} & \textbf{0.009} & 0.003          & \textbf{0.002} \\ \cline{2-8} 
\multicolumn{1}{|c|}{}                                                                                                                   & \multicolumn{1}{c|}{\multirow{3}{*}{\begin{tabular}[c]{@{}c@{}}Generalization tests on\\ directed weighted networks\end{tabular}}}   & CNN-RP  & 0.146          & 0.235          & 0.217          & 0.007          & 0.031          \\ \cline{3-8} 
\multicolumn{1}{|c|}{}                                                                                                                   & \multicolumn{1}{c|}{}                                                                                                                & mCNN-RP & 0.136          & 0.225          & 0.173          & 0.006          & 0.030          \\ \cline{3-8} 
\multicolumn{1}{|c|}{}                                                                                                                   & \multicolumn{1}{c|}{}                                                                                                                & NRL-GT  & \textbf{0.009} & \textbf{0.021} & \textbf{0.030} & \textbf{0.003} & \textbf{0.002} \\ \hline
\end{tabular}
\vspace{-0.2cm}
\end{table*}
\vspace{-0.4cm}
\subsection{Robustness Learning for Real-World Networks}
Some instances are selected in Network Repository$\footnote{http://networkrepository.com/}$ to evaluate the performance of different models for robustness learning on real-world networks, including circuit networks, power networks, brain networks, and so on. NRL-GT, PCR, iPCR, CNN-RP, and mCNN-RP used here are trained on synthetic networks. The details of the real-world networks are shown in Table S2 of SI, where Circuits (A)-(C) and Power (A)-(B) are randomly sampled from the circuit-3 networks $\footnote{https://networkrepository.com/circuit-3.php}$ and the U.S. power networks $\footnote{https://networkrepository.com/power-bcspwr10.php}$, respectively. We reduce these network sizes to 1000 by randomly removing a few nodes. Fig. 5 and Fig. S2 in SI represent controllability robustness learning results and connectivity robustness learning results of different methods on real-world networks, respectively. The mean learning errors of different methods are summarized in Table IV. For controllability robustness learning of real-world networks, NRL-GT has obvious advantages. Except for 145, and G43, NRL-GT achieves the smallest mean error. Although PCR has a smaller average learning error in 145 and G43, the corresponding robustness curve of PCR fluctuates greatly around the true value, which cannot reflect the true controllability robustness (see Figs. 5(l)-5(m)). For connectivity robustness learning of real-world networks, NRL-GT also has significant advantages. Except for Circuits(A), DDG, DW5, and DW7, NRL-GT achieves the most satisfactory performance. It is worth mentioning that on almost all circuit networks and power networks, NRL-GT outperforms CNN-based models both in terms of controllability robustness learning and connectivity robustness learning.
\begin{figure*}[!h]
\centering
\subfloat[]{\includegraphics[width=0.2\textwidth]{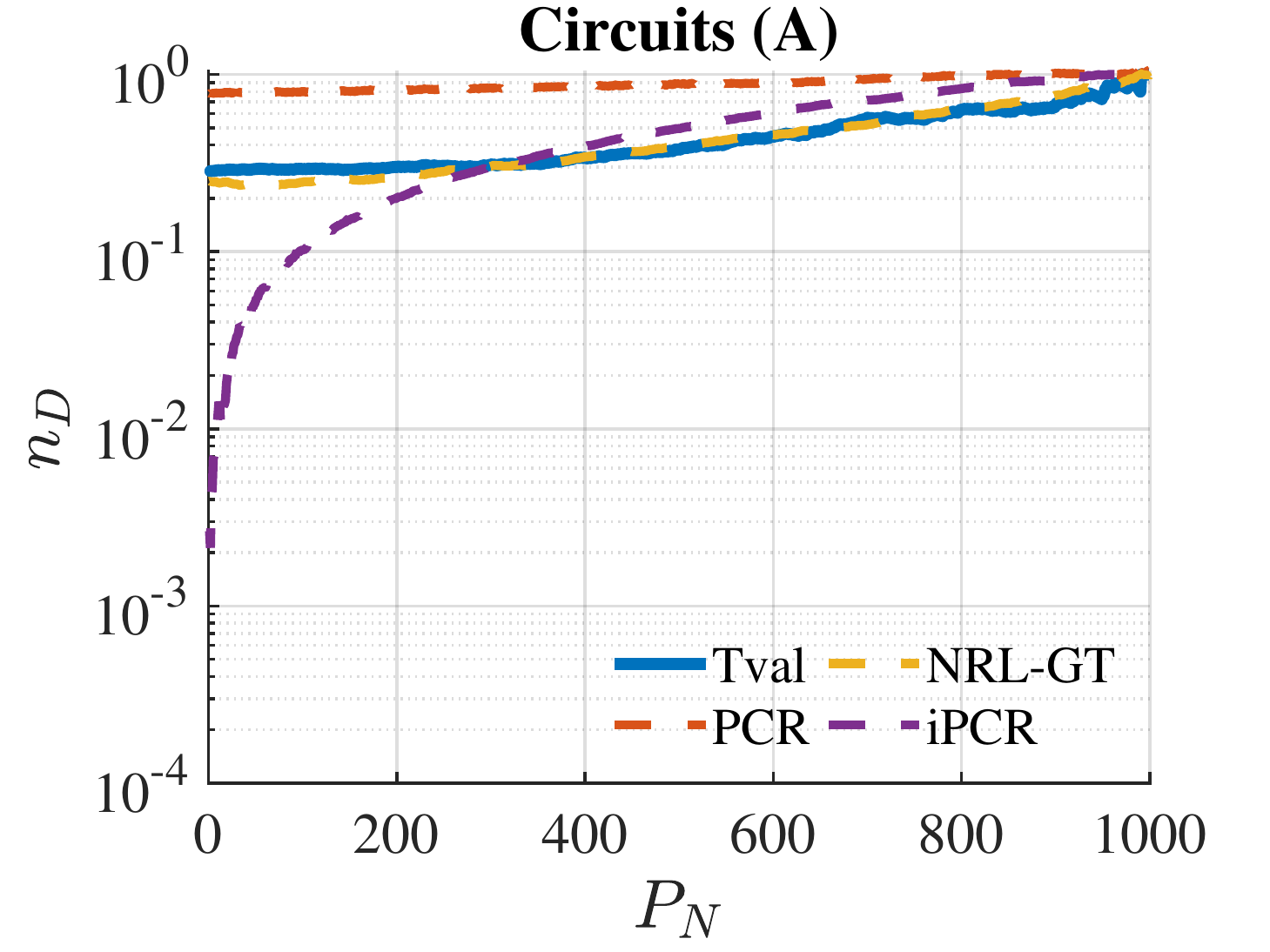}%
\label{a}}
\subfloat[]{\includegraphics[width=0.2\textwidth]{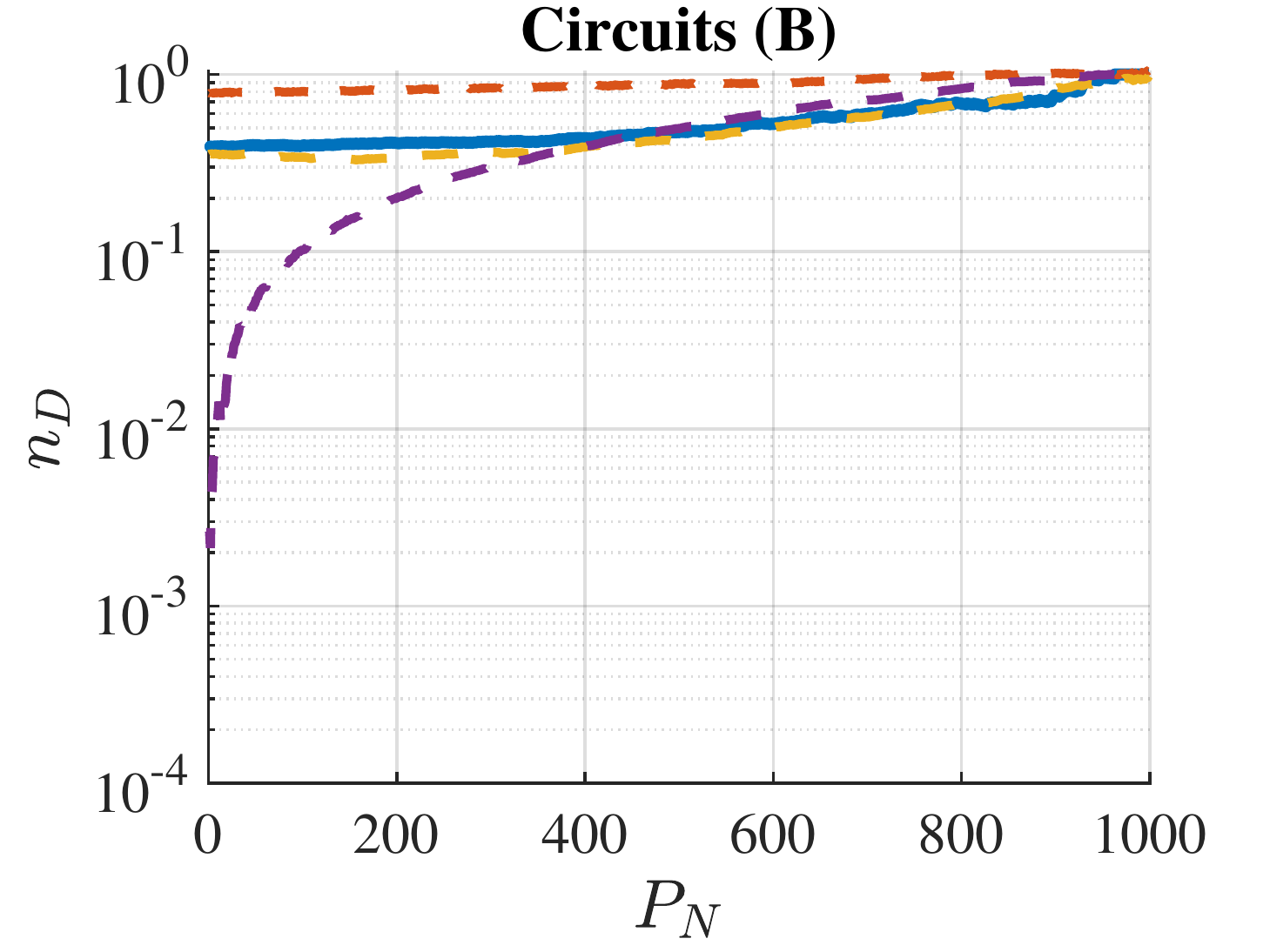}%
\label{b}}
\subfloat[]{\includegraphics[width=0.2\textwidth]{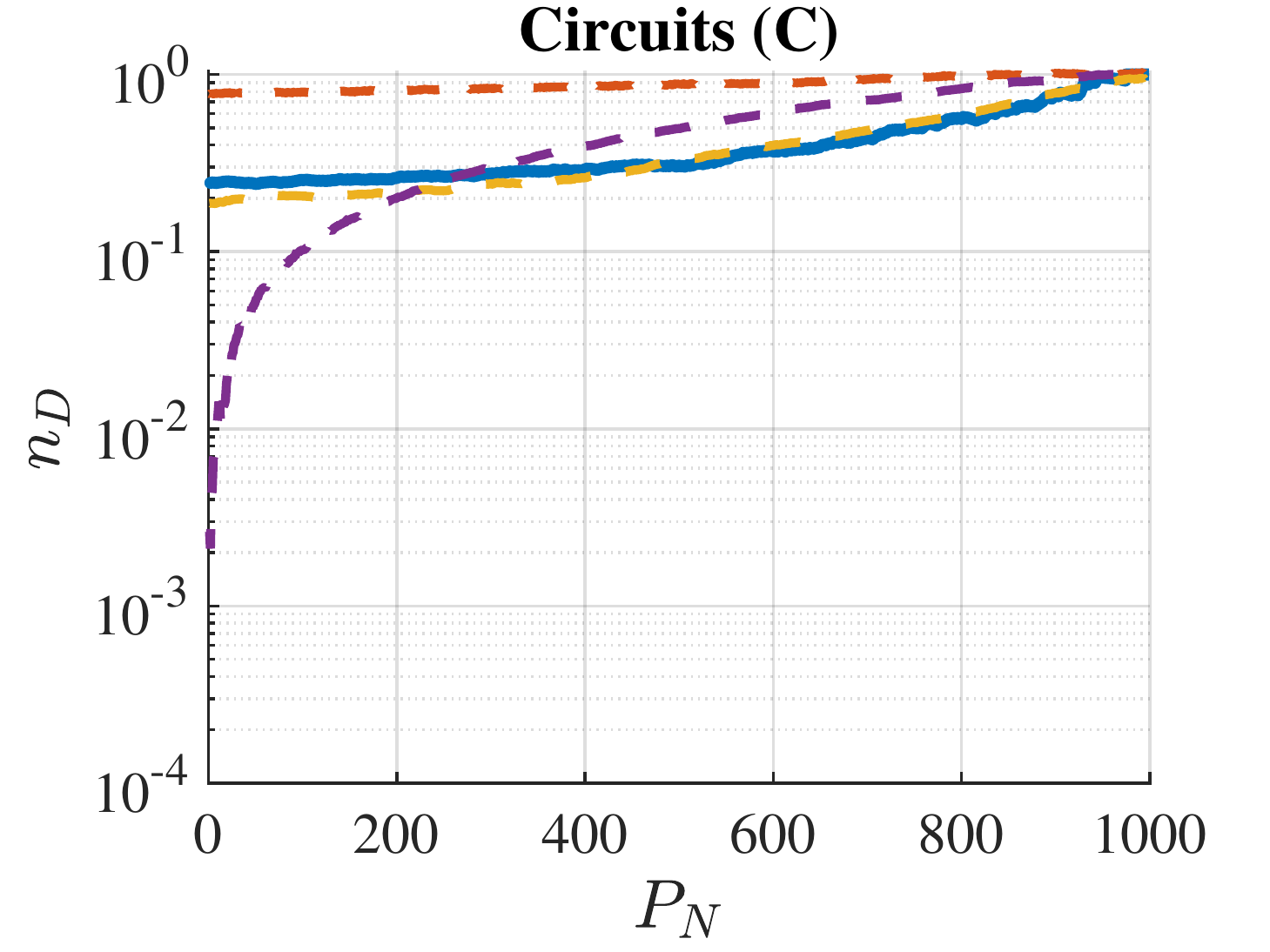}%
\label{c}}
\subfloat[]{\includegraphics[width=0.2\textwidth]{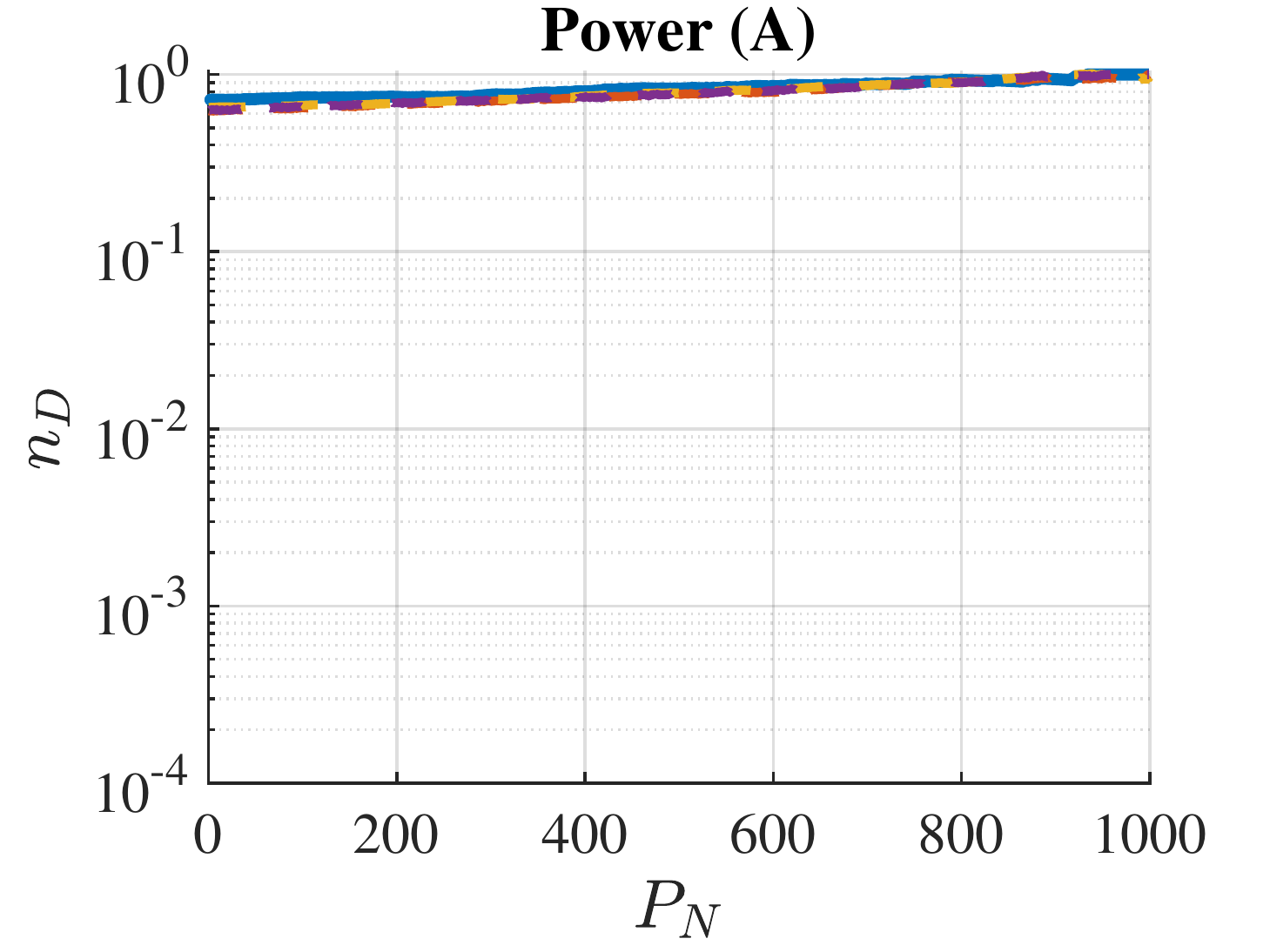}%
\label{d}}
\subfloat[]{\includegraphics[width=0.2\textwidth]{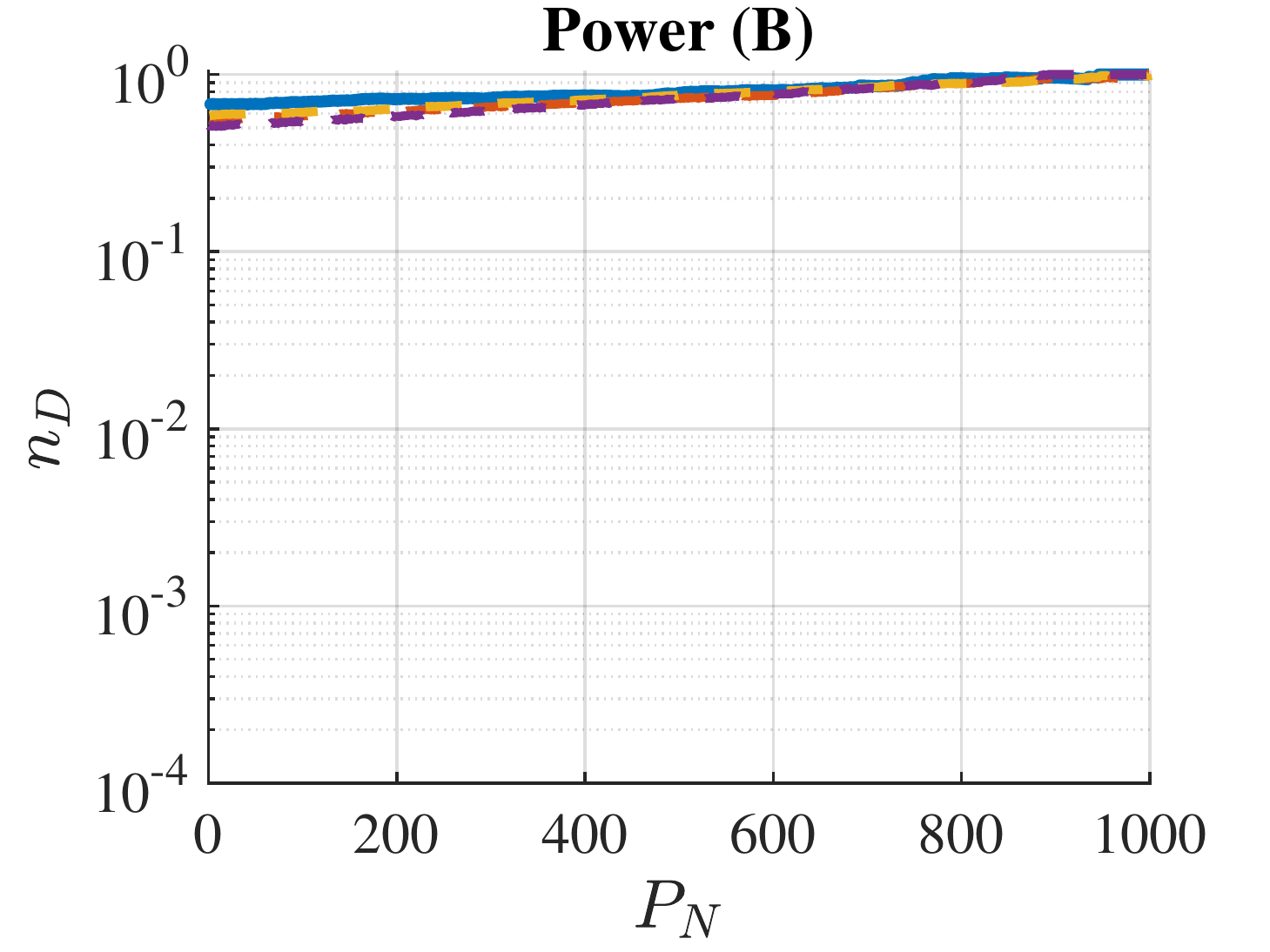}%
\label{d}}

\subfloat[]{\includegraphics[width=0.2\textwidth]{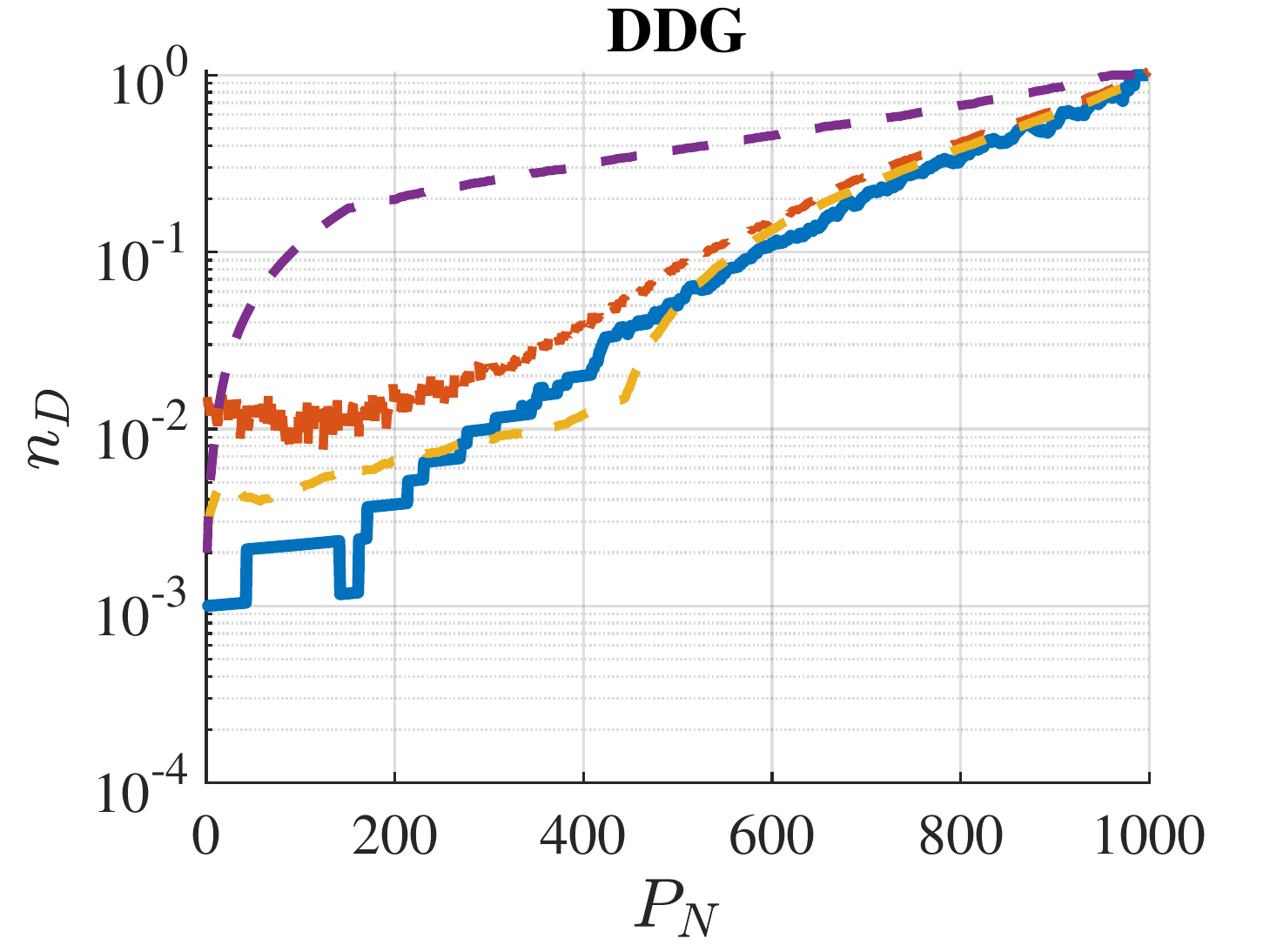}%
\label{a}}
\subfloat[]{\includegraphics[width=0.2\textwidth]{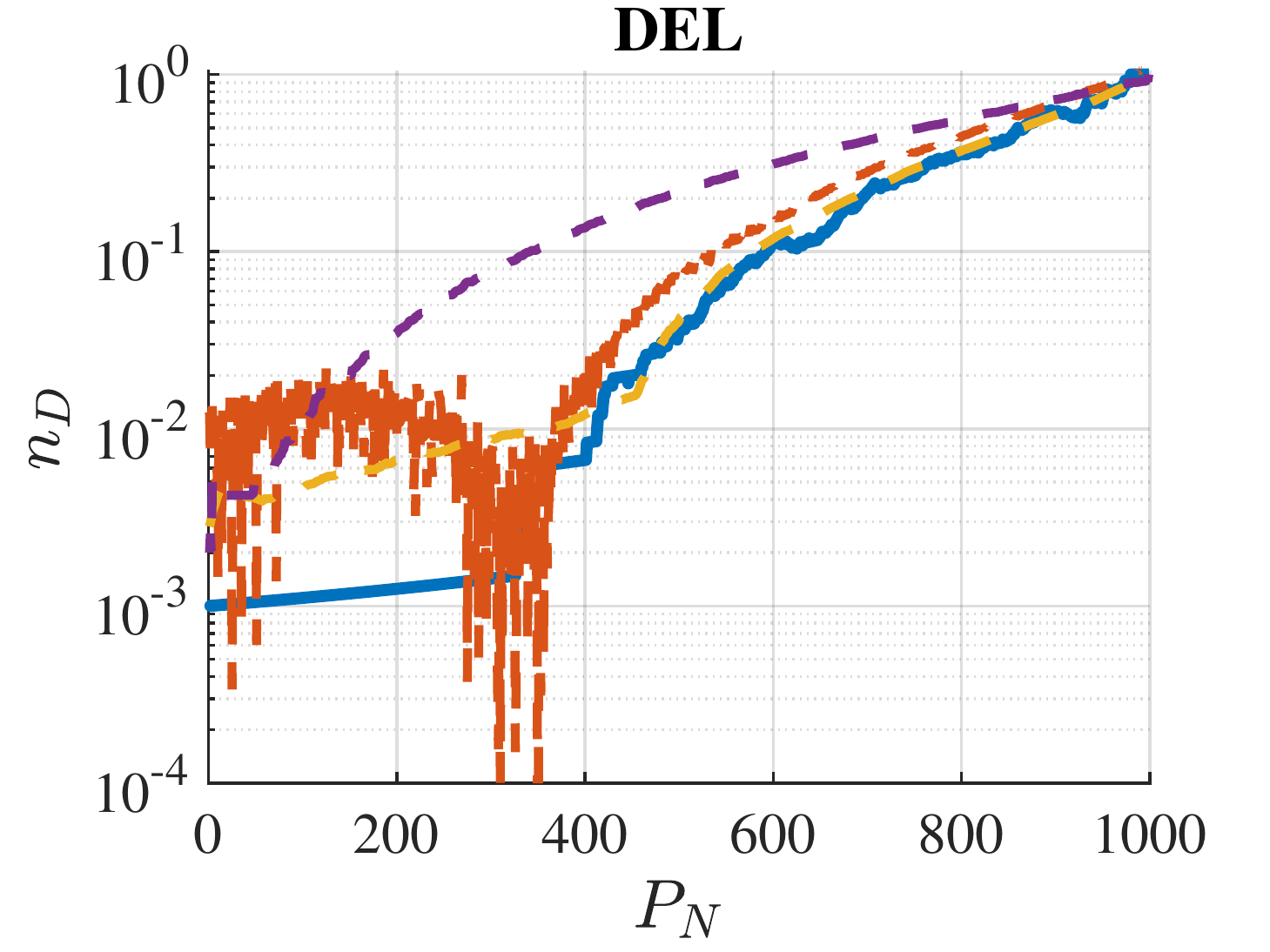}%
\label{b}}
\subfloat[]{\includegraphics[width=0.2\textwidth]{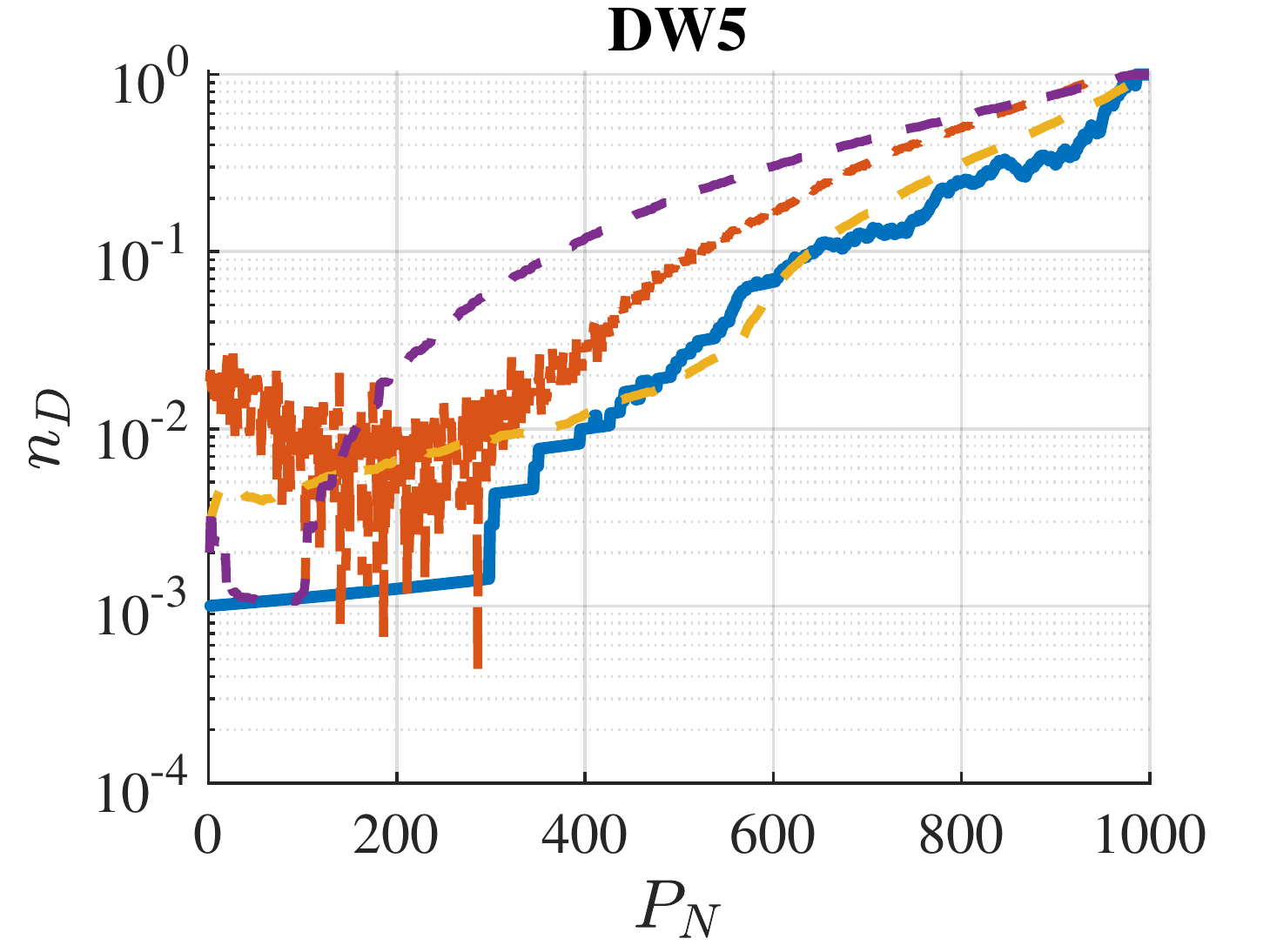}%
\label{c}}
\subfloat[]{\includegraphics[width=0.2\textwidth]{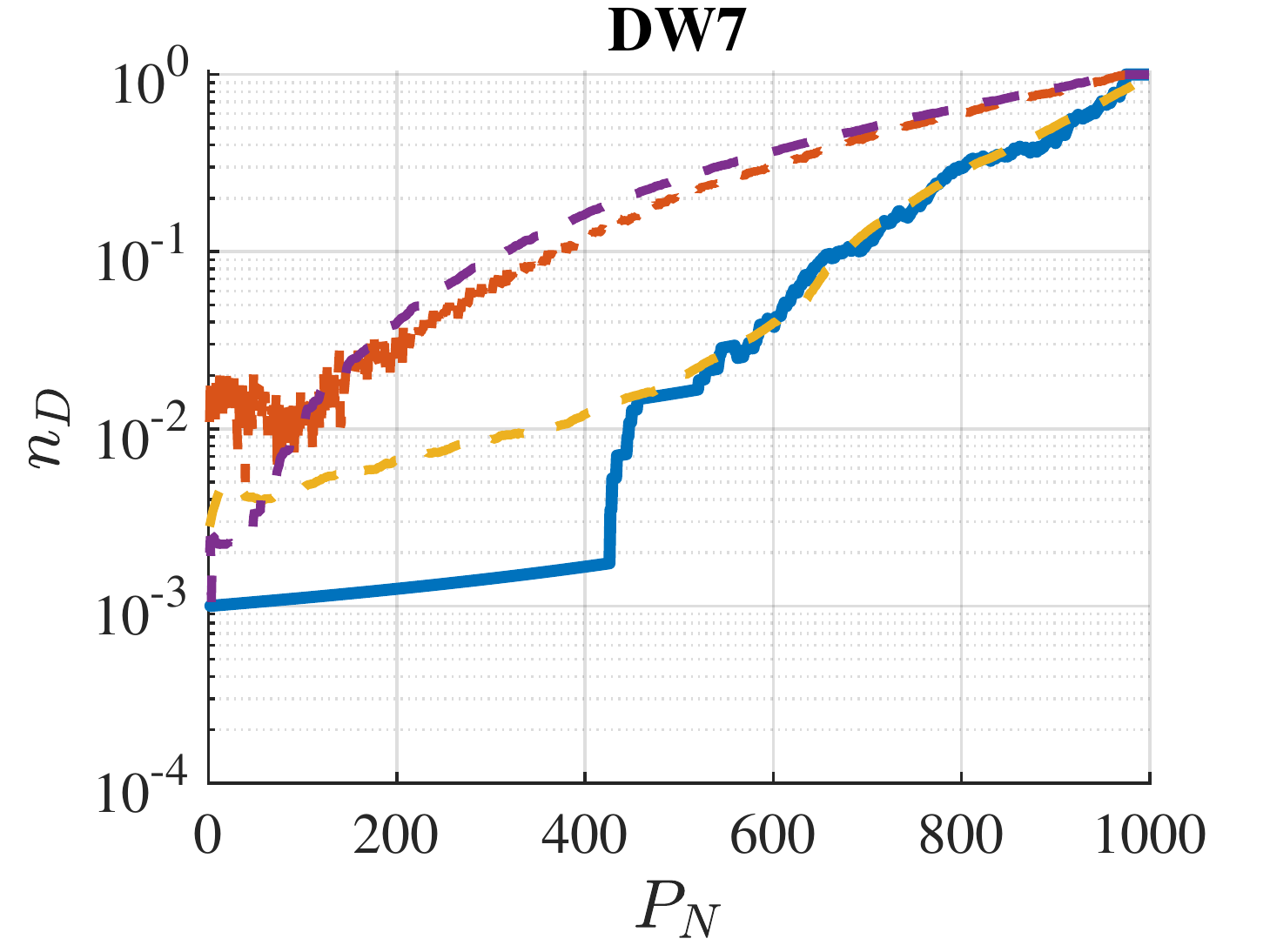}%
\label{d}}
\subfloat[]{\includegraphics[width=0.2\textwidth]{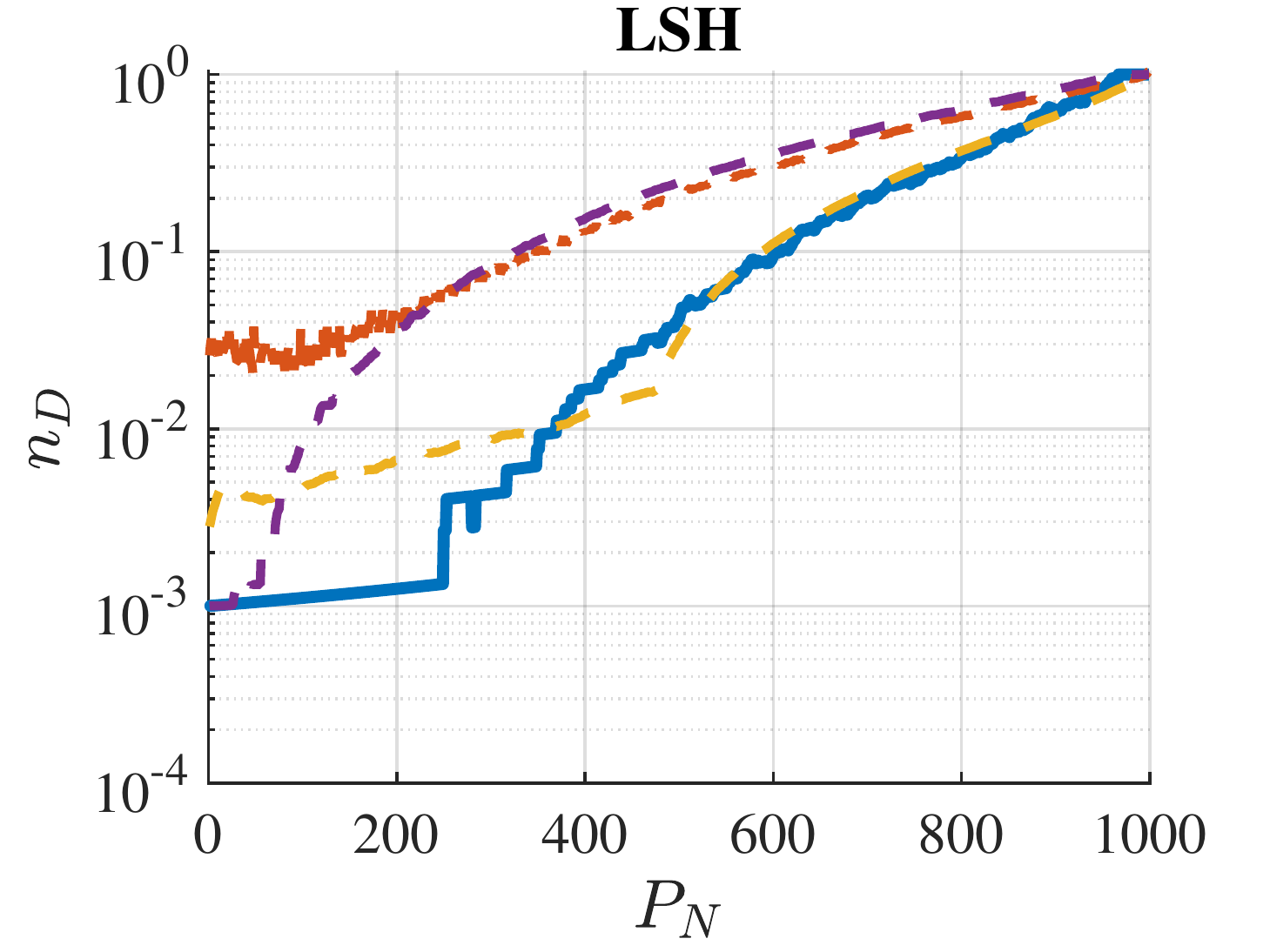}%
\label{d}}

\subfloat[]{\includegraphics[width=0.2\textwidth]{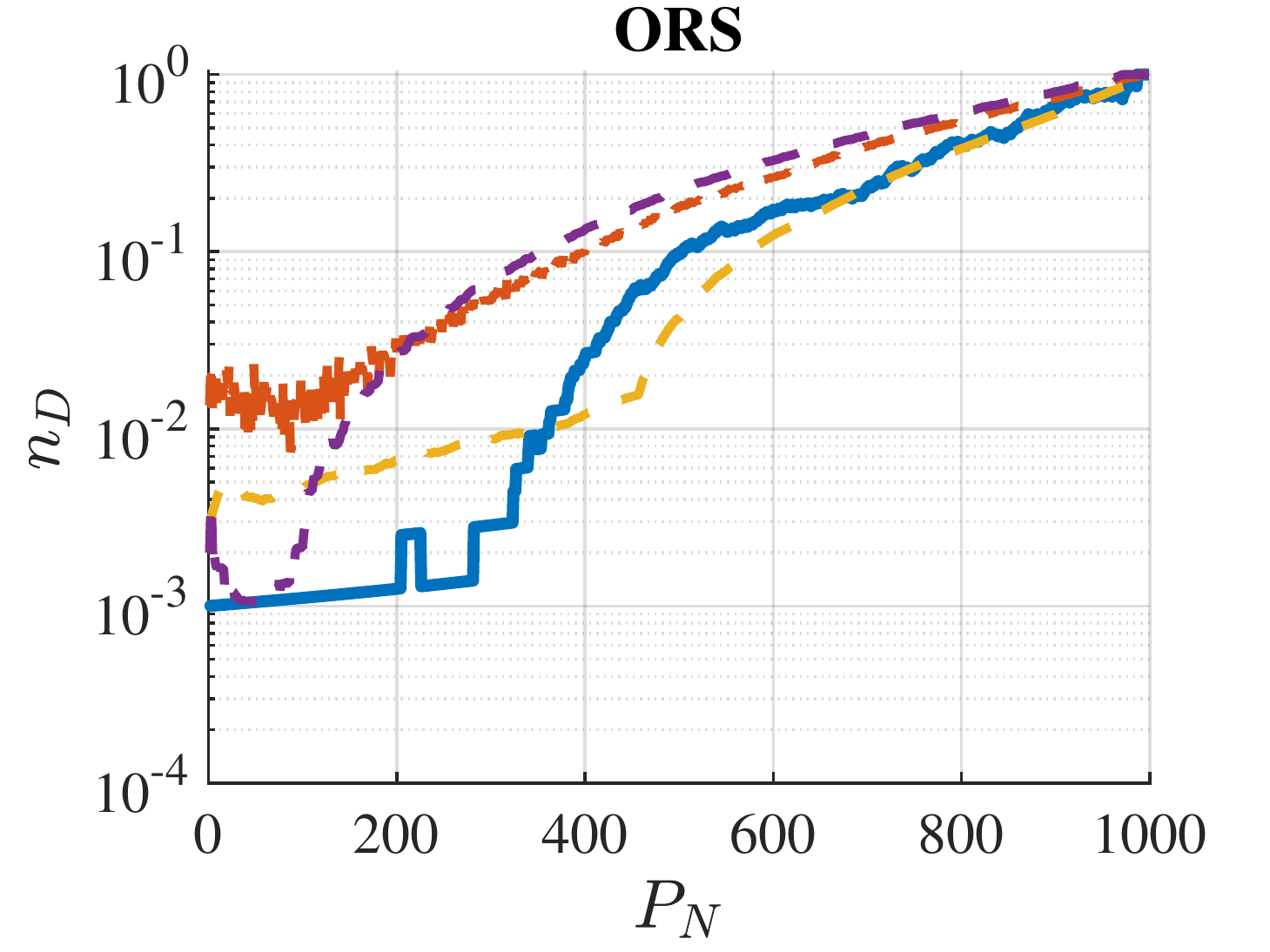}%
\label{a}}
\subfloat[]{\includegraphics[width=0.2\textwidth]{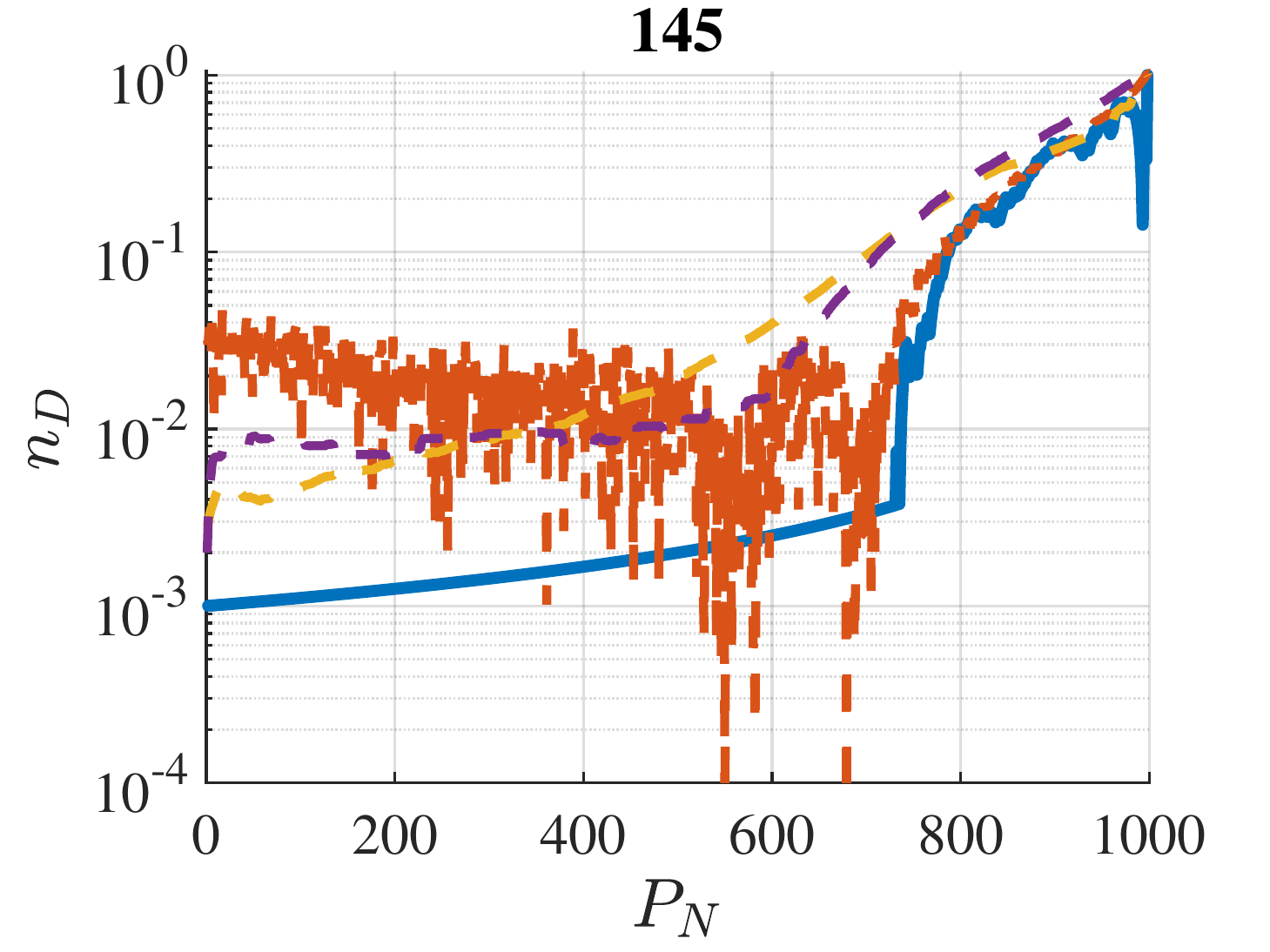}%
\label{b}}
\subfloat[]{\includegraphics[width=0.2\textwidth]{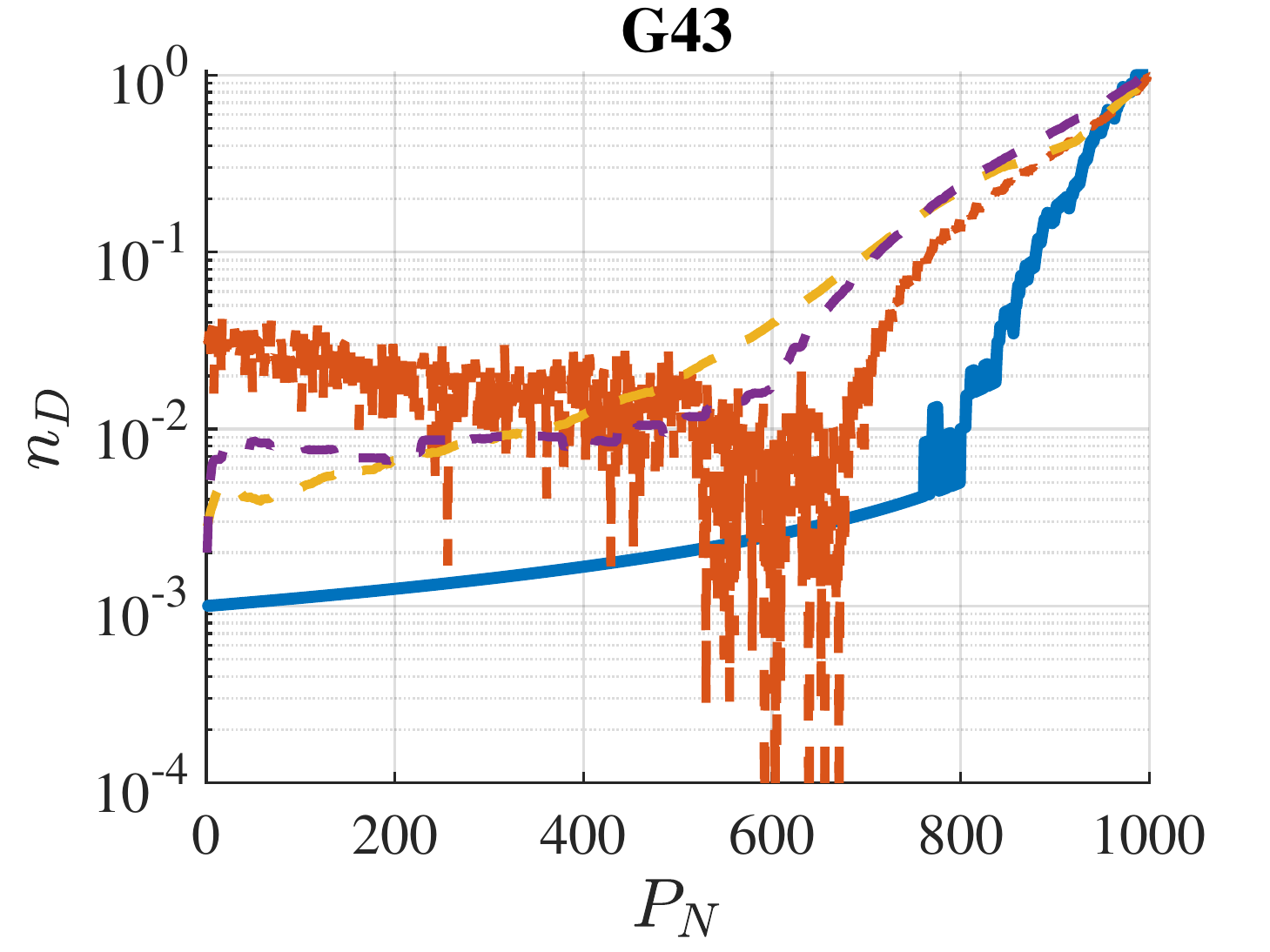}%
\label{c}}
\subfloat[]{\includegraphics[width=0.2\textwidth]{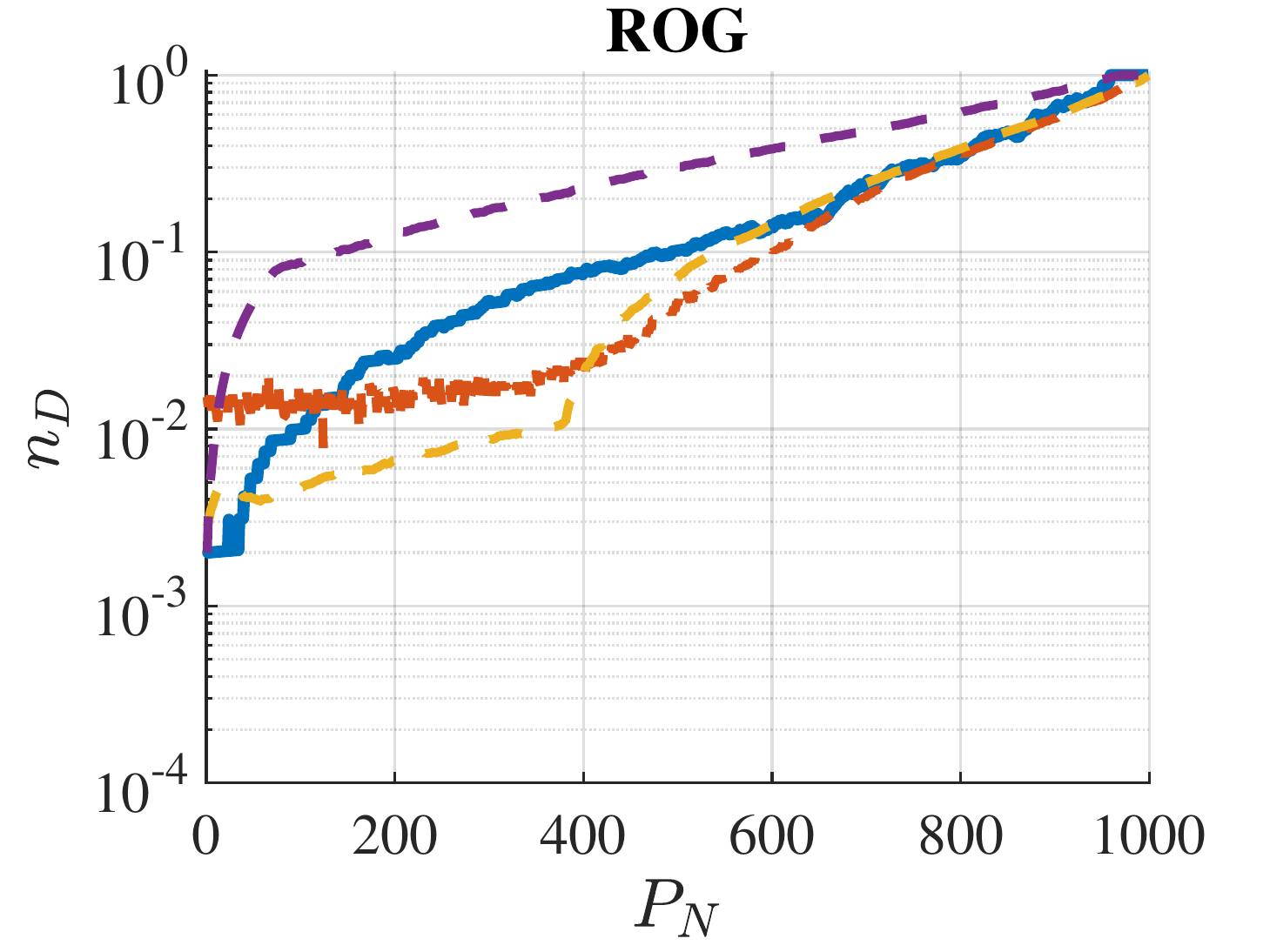}%
\label{d}}
\subfloat[]{\includegraphics[width=0.2\textwidth]{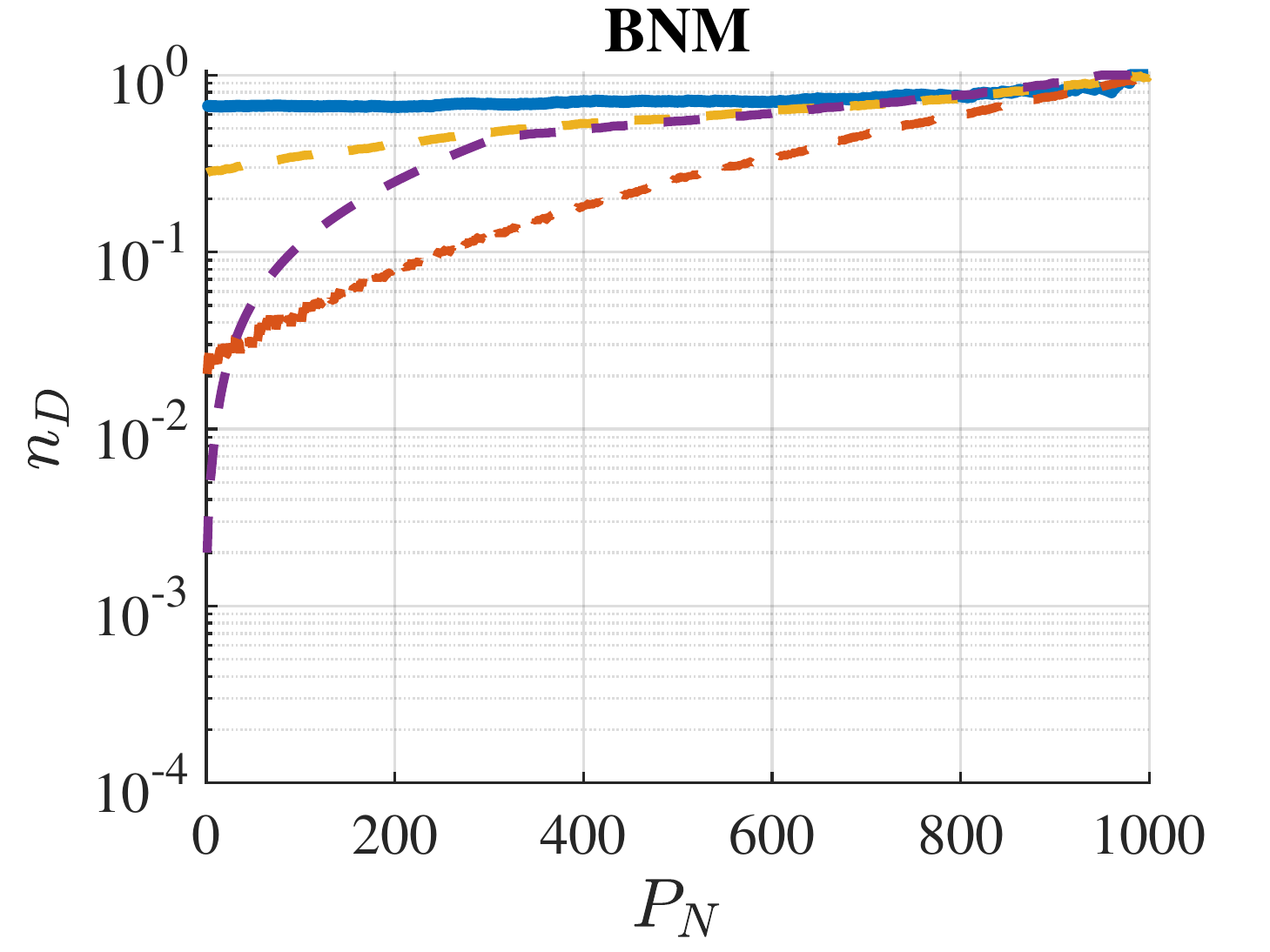}%
\label{d}}

\caption{Comparison of controllability robustness curve learning results of NRL-GT, PCR~\cite{31}, and iPCR~\cite{32} on real-world networks, including circuit networks, power networks, brain networks, and so on. The details of real-world networks are displayed in Table S2.}
\label{fig_sim}
\end{figure*}
\vspace{-0.4cm}
\subsection{Transfer Learning Applications Enabled by NRL-GT}
The backbone of NRL-GT has powerful transferability to enable different application scenarios such as transfer learning for complex networks of different sizes and transfer learning between different modules in NRL-GT. 

\subsubsection{Transfer learning for complex networks with different sizes}
\

Firstly, the backbone of NRL-GT pre-trained on fixed-size networks can be transferred to networks of different sizes for feature extraction, including large-scale networks that CNNs are difficult to deal with. It is worth noting that for networks with more than 3000 nodes, the generation of training and test sets through attack simulations will be very time-consuming. Therefore, we extract features for  3000-node and 600-node networks directly using the backbone trained on the 1000-node networks to
demonstrate the performance of NRL-GT. Given training and test data for larger networks, NRL-GT still can accomplish different tasks at high speed.

After adjusting the output dimension of the robustness curve learning module for networks of different sizes, we directly transfer the backbone trained on 1000-node networks for feature extraction and fine-tune the module with a small amount of data, which is only half of that in Section IV-B.
Fig. S3 and Fig. S4 of SI show the learning results of NRL-GT for network sizes $N=600$ and $N=3000$. The average test errors are summarized in Table V. In Fig. S3 and Fig. S4, the top row indicates the learning results of controllability robustness while the bottom row indicates the results of connectivity robustness. The learning results of NRL-GT fit well with the $Tvals$ for every configuration. Based on the experimental results in Table II and Table III and in ~\cite{31} and ~\cite{68}, we set 0.03 and 0.06 as the error thresholds for network controllability robustness learning and connectivity robustness learning, respectively, and any error above the threshold cannot be considered a low error. Combining Fig. S3 and Fig. S4 with Table V, it can be found that both mean test errors and the error curves are basically below the threshold value, demonstrating that NRL-GT provides a high-precision transferable feature extraction backbone, which can extract the robustness-related feature of complex networks of any size in a large range and perform high-precision learning for networks of different sizes.
\begin{table*}[h]
\caption{robustness learning comparison of NRL-GT, PCR~\cite{31}, iPCR~\cite{32}, CNN-RP~\cite{68}, and mCNN-RP~\cite{77} on real-world networks. The details of real-world networks are displayed in Table S2.\label{tab:table2}}
\resizebox{\linewidth}{!}{
\begin{tabular}{|cc|c|c|c|c|c|c|c|c|c|c|c|c|c|c|c|}
\hline
\multicolumn{2}{|c|}{Real-World Networks}                                                                              & Circuits(A)    & Circuits(B)    & Circuits(C)    & Power(A)       & Power(B)       & DDG            & DEL            & DW5            & DW7            & LSH            & ORS            & 145            & G43            & ROG            & BNM            \\ \hline
\multicolumn{1}{|c|}{\multirow{3}{*}{\begin{tabular}[c]{@{}c@{}}Controllability \\ Robustness\end{tabular}}} & PCR     & 0.444          & 0.351          & 0.476          & 0.054          & 0.064          & 0.037          & 0.046          & 0.122          & 0.171          & 0.126          & 0.081          & \textbf{0.028} & \textbf{0.048} & 0.036          & 0.398          \\ \cline{2-17} 
\multicolumn{1}{|c|}{}                                                                                       & iPCR    & 0.149          & 0.139          & 0.161          & 0.047          & 0.077          & 0.256          & 0.124          & 0.173             & 0.203          & 0.151          & 0.114          & 0.056          & 0.070          & 0.159          & 0.224          \\ \cline{2-17} 
\multicolumn{1}{|c|}{}                                                                                       & NRL-GT  & \textbf{0.030} & \textbf{0.046} & \textbf{0.038} & \textbf{0.041} & \textbf{0.045} & \textbf{0.020} & \textbf{0.013} & \textbf{0.037} & \textbf{0.014} & \textbf{0.019} & \textbf{0.022} & 0.041          & 0.063          & \textbf{0.030} & \textbf{0.153} \\ \hline
\multicolumn{1}{|c|}{\multirow{3}{*}{\begin{tabular}[c]{@{}c@{}}Connectivity \\ Robustness\end{tabular}}}    & CNN-RP  & 0.225          & 0.095          & 0.086          & 0.069          & 0.072          & 0.197          & 0.103          & \textbf{0.026} & \textbf{0.155} & 0.501          & 0.260          & 0.236          & 0.029          & 0.035          & 0.123          \\ \cline{2-17} 
\multicolumn{1}{|c|}{}                                                                                       & mCNN-RP & \textbf{0.175}          & 0.072          & 0.081          & 0.108          &0.123        & \textbf{0.114} & 0.298          & 0.040          & 0.200          & 0.600          & 0.358          & 0.286          &0.037          & 0.090          & 0.111          \\ \cline{2-17} 
\multicolumn{1}{|c|}{}                                                                                       & NRL-GT  & 0.190 & \textbf{0.046} & \textbf{0.071} & \textbf{0.033} & \textbf{0.028} & 0.198          & \textbf{0.090} & 0.133          & 0.194          & \textbf{0.249} & \textbf{0.206} & \textbf{0.078} & \textbf{0.024} & \textbf{0.031} & \textbf{0.042} \\ \hline
\end{tabular}}
\end{table*}
\vspace{-0.2cm}
\begin{table}[!h]
\centering
\caption{Average errors of robustness learning values of NRL-GT of BA, ER, NW, QSN, and SF networks with$N=3000$ and $N=600$\label{tab:table2}}
\begin{tabular}{|cc|c|c|c|c|c|}
\hline
\multicolumn{2}{|c|}{Average Learning Error $\overline{\xi} $}                               & BA    & ER    & NW    & QSN   & SF    \\ \hline
\multicolumn{1}{|c|}{\multirow{2}{*}{\begin{tabular}[c]{@{}c@{}}Controllability\\ robustness\end{tabular}}} & N=3000 & 0.014 & 0.014 & 0.012 & 0.011 & 0.016 \\ \cline{2-7} 
\multicolumn{1}{|c|}{}                                            & N=600  & 0.023 & 0.019 & 0.018 & 0.018 & 0.028 \\ \hline
\multicolumn{1}{|c|}{\multirow{2}{*}{\begin{tabular}[c]{@{}c@{}}Connectivity \\ robustness\end{tabular}}}    & N=3000 & 0.020 & 0.021 & 0.017 & 0.020 & 0.050 \\ \cline{2-7} 
\multicolumn{1}{|c|}{}                                            & N=600  & 0.037 & 0.037 & 0.033 & 0.037 & 0.061 \\ \hline
\end{tabular}
\vspace{-0.4cm}
\end{table}
\begin{table}[!h]
\centering
\caption{Comparison of the Overall Robustness learning errors of NRL-GT, PCR, iPCR, CNN-RP, and mCNN-RP under RA\label{tab:table2}}
\begin{tabular}{|cc|c|c|c|c|c|}
\hline
\multicolumn{2}{|c|}{$R_c$ error}                        & BA             & ER             & NW             & QSN            & SF             \\ \hline
\multicolumn{1}{|c|}{\multirow{3}{*}{Controllability}} & PCR    & 0.011          & 0.009          & 0.016          & 0.023          & 0.013          \\ \cline{2-7} 
\multicolumn{1}{|c|}{}                                 & iPCR   & 0.013          & 0.011          & 0.011          & 0.011          & 0.013          \\ \cline{2-7} 
\multicolumn{1}{|c|}{}                                 & NRL-GT & \textbf{0.010}  & \textbf{0.007} & \textbf{0.007} & \textbf{0.007} & \textbf{0.011} \\ \hline
\multicolumn{1}{|c|}{\multirow{3}{*}{Connectivity}}    & CNN-RP & 0.019          & 0.021          & 0.020           & 0.020           & 0.039          \\ \cline{2-7} 
\multicolumn{1}{|c|}{}                                 & mCNN-RP & 0.017          & 0.019          & 0.016           & 0.021           & 0.040          \\ \cline{2-7} 
\multicolumn{1}{|c|}{}                                 & NRL-GT & \textbf{0.016} & \textbf{0.012} & \textbf{0.014} & \textbf{0.015} & \textbf{0.038} \\ \hline
\end{tabular}
\end{table}
\vspace{0.2cm}
\subsubsection{Transfer learning between different downstream tasks in NRL-GT}
\

Due to the training strategy described in Section III-D, the features extracted by the backbone of NRL-GT in predicting robustness curves can be used to predict overall robustness and to classify synthetic networks. In other words, the backbone trained on the task of robustness curve learning can be directly transferred to the task of overall robustness learning and synthetic network classification without retraining. 

{\bf{Overall robustness learning}} In NRL-GT, we design a learning module dedicated to overall robustness prediction of networks with different sizes, which receives the topology information extracted from the backbone transferred from the robustness curve learning task and directly outputs the predicted value of overall robustness. In addition, the overall robustness learning module also enables transfer learning of different sizes, that is, the module trained on 1000-node complex networks can be generalized to overall robustness learning of networks of any size in a large range. 

In Table VI, we compare the $R_c$ learning precision of different methods under fixed network size. Each value in the table represents the learning error of $R_c$, calculated by $error_{R_c} = \left| {{{\hat R}_c} - {R_c}} \right|$, where  ${{{\hat R}_c}}$ represents the predicted value and ${{R_c}}$ is the true value by simulation. The training and test data are from Section IV-B and Section IV-C. It can be seen that NRL-GT has an absolute advantage over other methods in terms of overall network controllability robustness and connectivity robustness.

To demonstrate the transferability of the overall robustness learning module for networks of different sizes, random sampling experiments are carried out. Note that the model used here is trained on 1000-node networks. We set the range of network size in the test data as [600, 3000]. And we randomly sample 200 instances for each topology.

For controllability robustness, the variation of $error_{R_c}$ with  ${\left\langle k \right\rangle }$  and network size $N$ is illustrated in Fig. 6. For SF, there are a few points above the error limit of 0.03, mainly in the interval $ N \in [2500,3000]$ and ${\left\langle k \right\rangle } \in [8,10]$, while for the other topologies, the $error_{R_c}$ is less than the error limit in the whole range.

\begin{figure*}[h]
\centering
\subfloat[]{\includegraphics[width=0.2\textwidth]{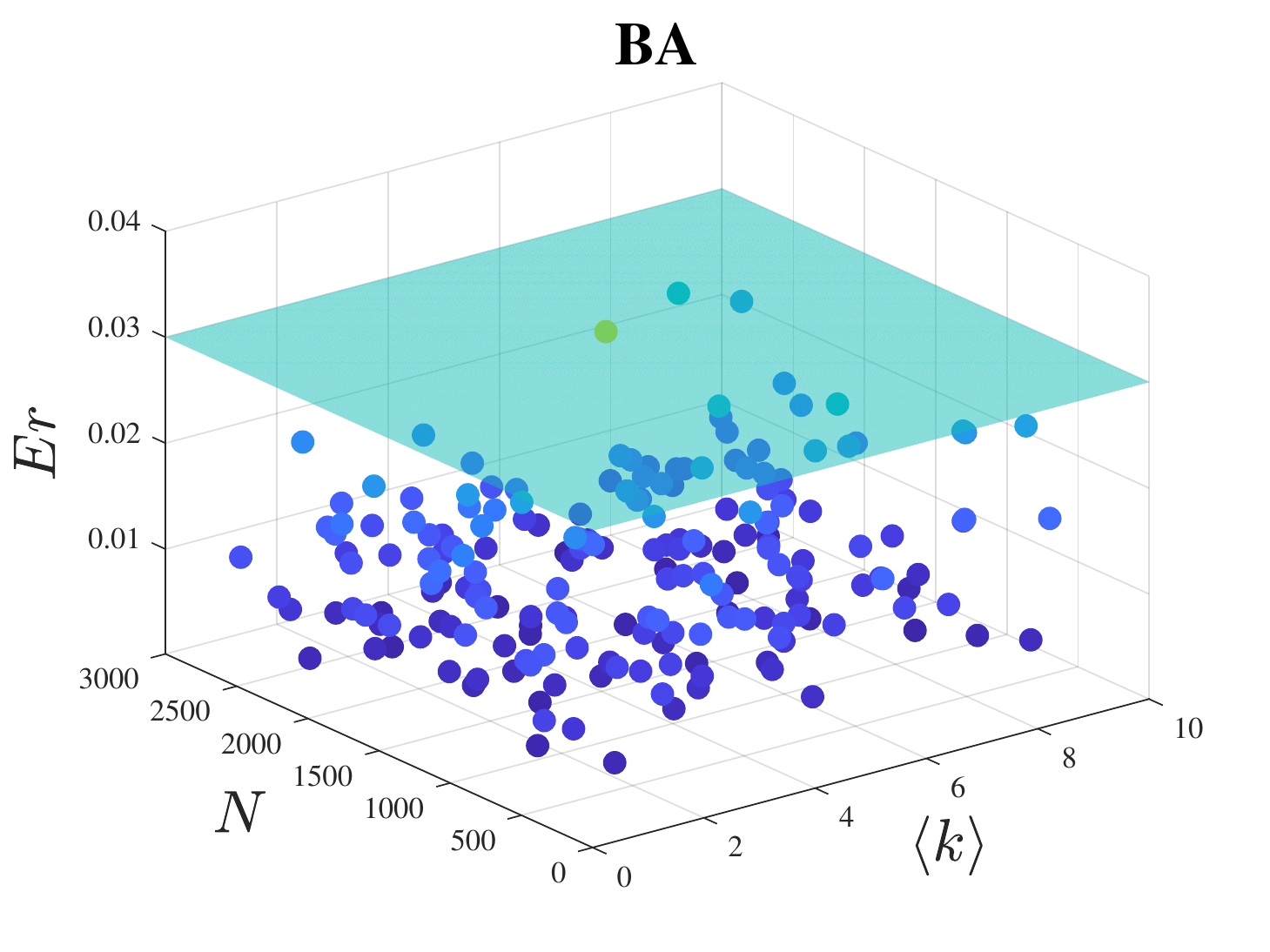}%
\label{a}}
\subfloat[]{\includegraphics[width=0.2\textwidth]{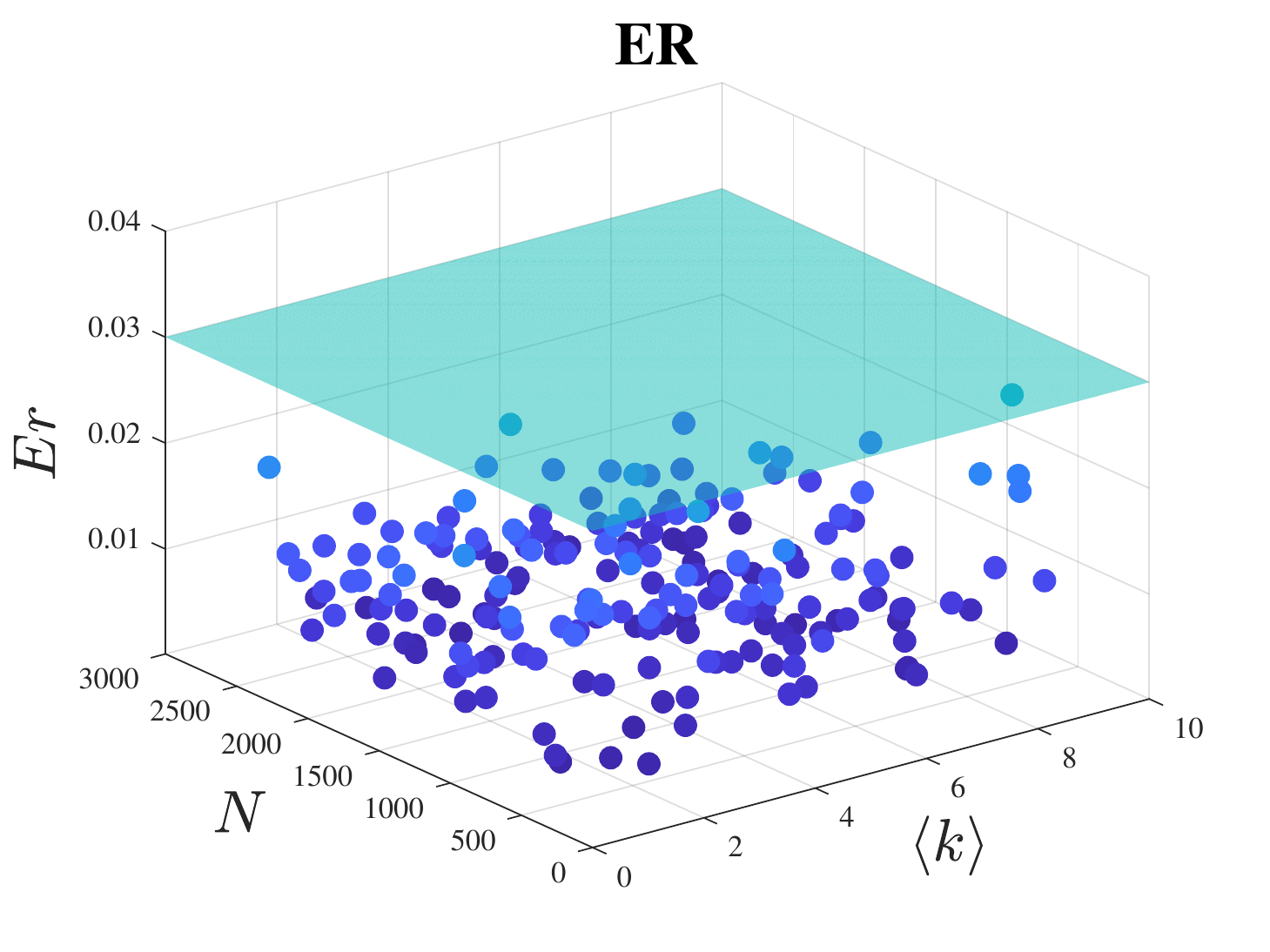}%
\label{b}}
\subfloat[]{\includegraphics[width=0.2\textwidth]{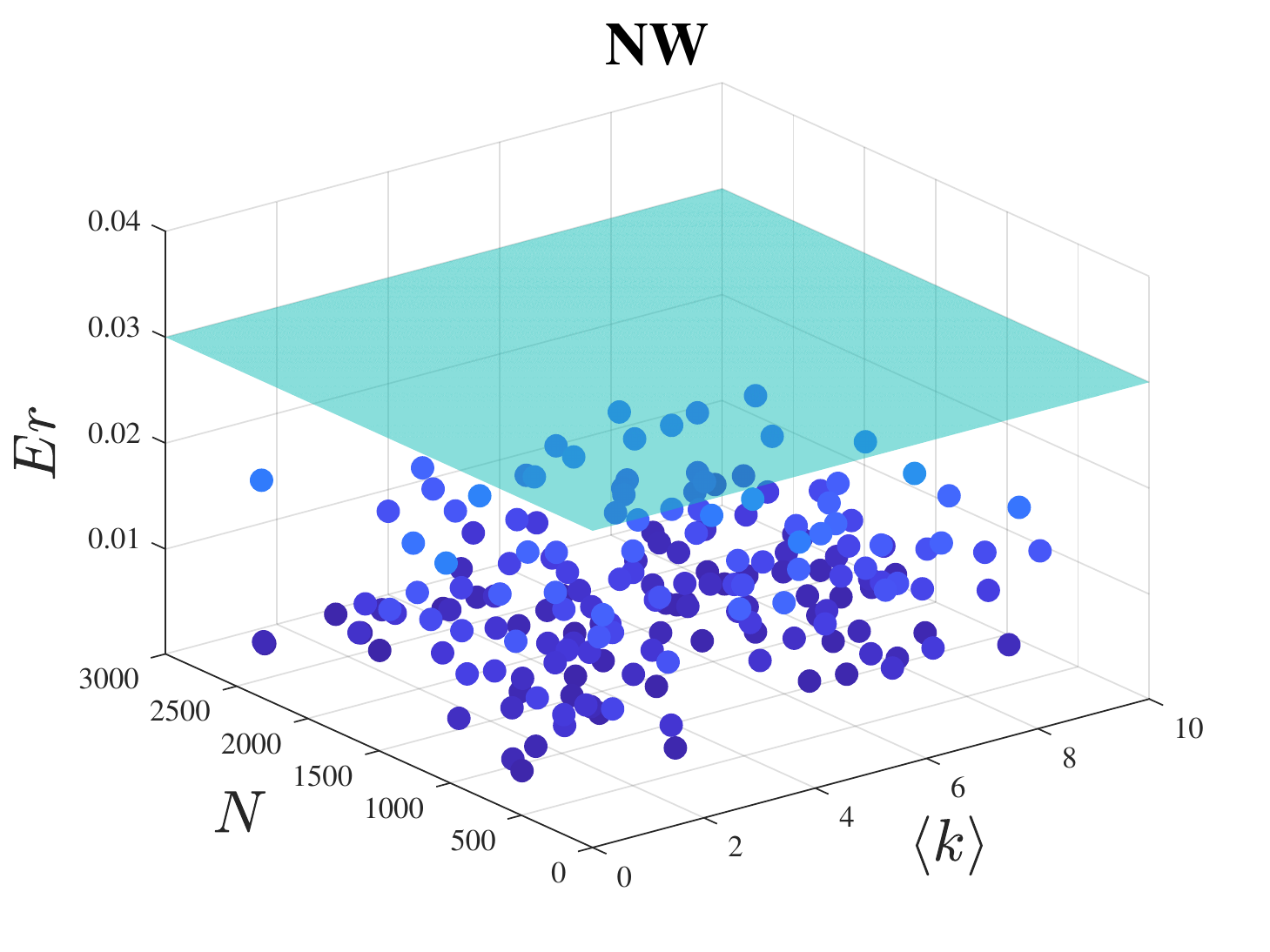}%
\label{c}}
\subfloat[]{\includegraphics[width=0.2\textwidth]{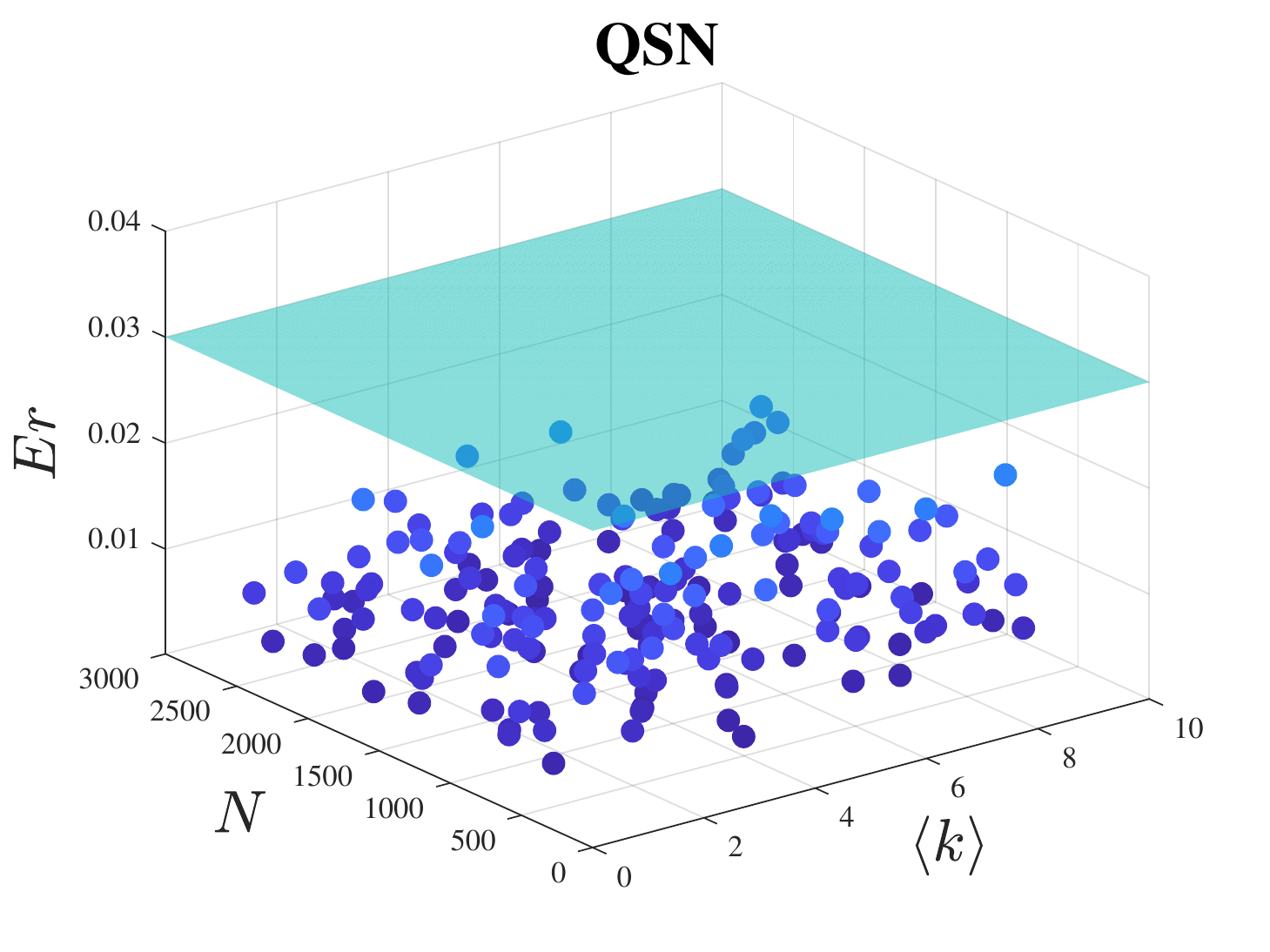}%
\label{d}}
\subfloat[]{\includegraphics[width=0.2\textwidth]{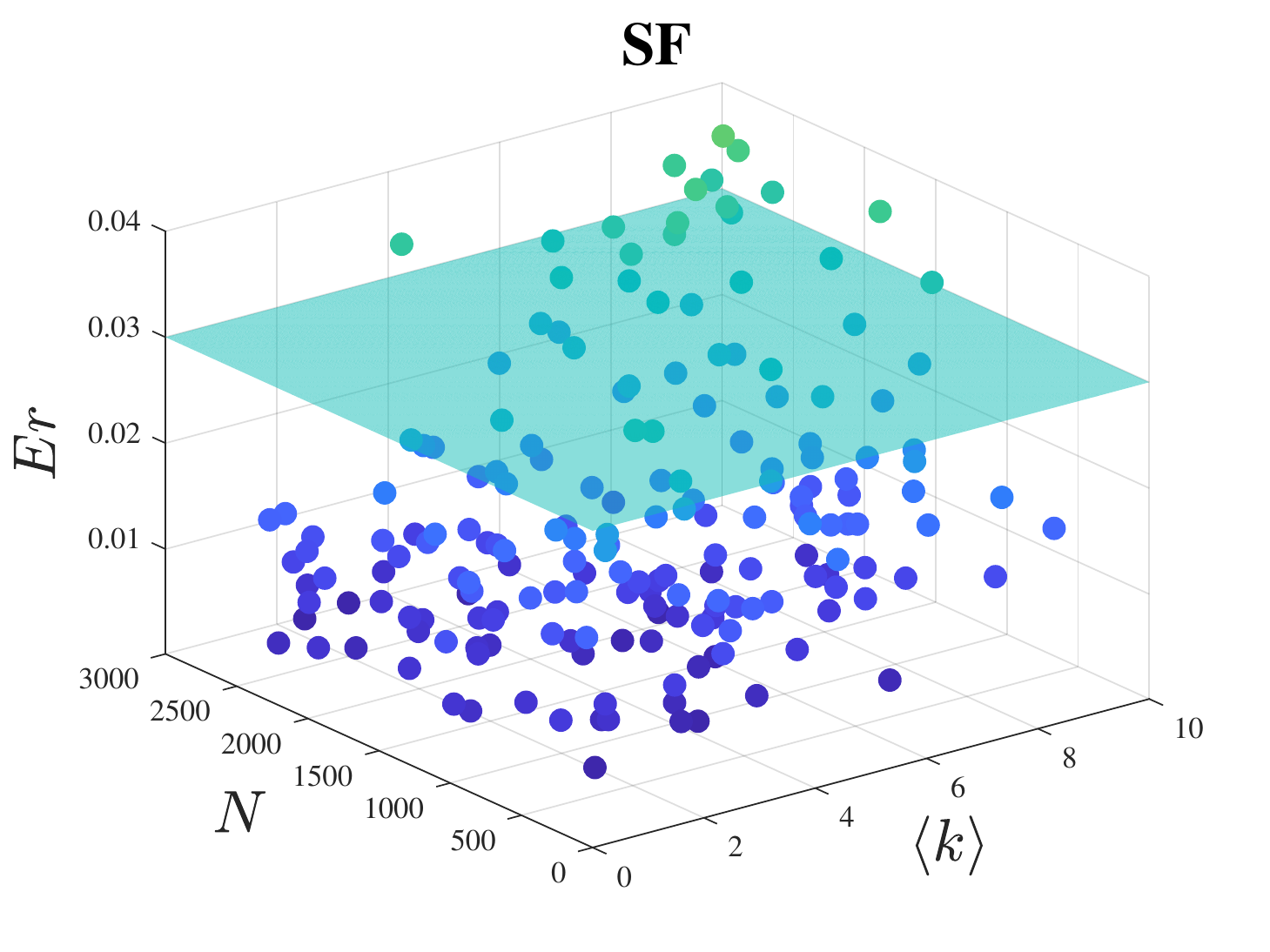}%
\label{d}}
\caption{Overall controllability robustness learning errors of NRL-GT on networks with average degree ${\left\langle k \right\rangle } \in [1,10]$ and network size $ N \in [600,3000]$. Cyan plane denotes the error threshold, and the blue dots denote the overall robustness learning errors of NRL-GT for different topologies.}
\label{fig_sim}
\end{figure*}

For connectivity robustness, we compare NRL-GT with spectral measures, which are general estimators of connectivity robustness for undirected networks~\cite{88}, including algebraic connectivity (AC), natural connectivity (NC), spectral gap (SG), and spectral radius (SR). Performance comparison with spectral measures is based on the predicted rank error of network robustness~\cite{68}. Specifically, we obtain five predicted rank lists by NRL-GT and four spectral measures and compare them with the real rank list generated by simulation. The rank errors can be obtained by averaging the absolute values of the differences between the predicted and true values for different methods. The rank errors of test data for the different methods are summarized in Table VII. Overall, NRL-GT can obtain the smallest average rank error, demonstrating that it provides the best evaluation for connectivity robustness compared to spectral measures.
\begin{table}[h]
\centering
\caption{Learning rank errors of the four spectral measures and NRL-GT\label{tab:table2}}
\begin{tabular}{|c|c|c|c|c|c|c|c|}
\hline
\begin{tabular}[c]{@{}c@{}}Average \\ Rank Error\end{tabular} & BA             & ER            & NW            & QSN          & SF             & overall        & rank \\ \hline
SR                 & 14.01          & 7.69          & 6.39          & 6.54         & 29.04          & 12.73         & 4    \\ \hline
SG                 & 11.21          & 7.65          & 7.41          & 6.65         & 24.68          & 11.52          & 2    \\ \hline
NC                 & 12.70           & \textbf{7.54} & \textbf{6.08} & \textbf{6.40} & 28.26          & 12.20         & 3    \\ \hline
AC                 & 30.52          & 17.19         & 7.15          & 9.63         & 56.75          & 24.25         & 5    \\ \hline
NRL-GT      & \textbf{11.12} & 7.64          & 7.39          & 6.81         & \textbf{16.87} & \textbf{9.97} & 1    \\ \hline
\end{tabular}
\vspace{-0.1cm}
\end{table}

{\bf{Synthetic network classification}} Similar to the module of overall robustness learning, the topological information captured by the backbone can also be used for the classification of synthetic networks. Here we directly transfer the backbone trained by the robustness curve learning task for initial feature extraction. We only need to train the classification module using the dataset with $N=1000$ from Section IV-B.  The classification test set contains both weighted and unweighted networks with network size $N \in [600,3000]$. The confusion matrix for the classification is shown in Table VIII. In Table VIII, (pre) and (act) represent the predicted and the actual type of network, respectively. It can be seen that 100\% classification accuracy is achieved for both unweighted and weighted synthetic networks of different sizes.
\begin{table}[h]
\centering
\caption{ Confusion matrix of NRL-GT for classifying unweighted and weighted networks with network size $N \in [600,3000]$. \label{tab:table2}}
\begin{tabular}{|c|c|c|c|c|c|}
\hline
\begin{tabular}[c]{@{}c@{}}Unweighted \\Networks\end{tabular} & BA(pre) & ER(pre) & NW(pre) & QSN(pre) & SF(pre) \\ \hline
BA(act)                                                                                               & 1       & 0       & 0       & 0        & 0       \\ \hline
ER(act)                                                                                               & 0       & 1       & 0       & 0        & 0       \\ \hline
NW(act)                                                                                               & 0       & 0       & 1       & 0        & 0       \\ \hline
QSN(act)                                                                                              & 0       & 0       & 0       & 1        & 0       \\ \hline
SF(act)                                                                                               & 0       & 0       & 0       & 0        & 1       \\ \hline
\begin{tabular}[c]{@{}c@{}}Weighted \\Networks\end{tabular}   & BA(pre) & ER(pre) & NW(pre) & QSN(pre) & SF(pre) \\ \hline
BA(act)                                                                                               & 1       & 0       & 0       & 0        & 0       \\ \hline
ER(act)                                                                                               & 0       & 1       & 0       & 0        & 0       \\ \hline
NW(act)                                                                                               & 0       & 0       & 1       & 0        & 0       \\ \hline
QSN(act)                                                                                              & 0       & 0       & 0       & 1        & 0       \\ \hline
SF(act)                                                                                               & 0       & 0       & 0       & 0        & 1       \\ \hline
\end{tabular}
\vspace{-0.2cm}
\end{table}
\vspace{-0.4cm}
\subsection{Running Time Comparison}
In addition to the precision, the speed improvement compared to attack simulations is another indicator for various methods. The average run time over 1000 independent tests on 1000-node networks of attack simulation (SIM), NRL-GT, PCR, iPCR, CNN-RP, and mCNN-RP are summarized in Table IX. For controllability robustness under RA, NRL-GT takes only 0.017 seconds to complete an inference, which is about 1/14 of PCR, 1/40 of iPCR, and 1/1000 of attack simulation. For connectivity robustness under TBA, one inference requires only 0.018 seconds for NRL-GT, which is about 1/13 of CNN-RP, 1/40 of mCNN-RP, and 1/4000 of attack simulation. Moreover, under RA, for a 3000-node network, NRL-GT accomplishes controllability robustness learning only in 0.04 seconds and connectivity robustness learning in 0.14 seconds, which are both more than 1000 times faster than attack simulation. Overall, NRL-GT has a definite speed advantage over leading-edge methods.
\begin{table}[!h]
\centering
\caption{Running time comparison of attack simulations, PCR, iPCR, CNN-RP, mCNN-RP, and NRL-GT\label{tab:table2}}
\begin{tabular}{|c|c|}
\hline
Controllability Robustness & Unit:Second    \\ \hline
SIM                        & 16.174         \\ \hline
PCR                        & 0.241          \\ \hline
iPCR                       & 0.692          \\ \hline
NRL-GT                     & \textbf{0.017} \\ \hline
Connectivity Robustness    & Unit:Second    \\ \hline
SIM                        & 75.543         \\ \hline
CNN-RP                     & 0.235
\\ \hline
mCNN-RP                    & 0.729
\\ \hline
NRL-GT                     & \textbf{0.018} \\ \hline
\end{tabular}
\vspace{-0.4cm}
\end{table}
\vspace{-0.4cm}
\subsection{Ablation Experiments Comparing to Classical GNN and Transformer-Based Models}
As stated in the Introduction and Section III-B, it theoretically shows that the proposed Graph transformer layer outperforms the classical GNN and transformer-based models in network robustness learning. In this subsection, we provide experimental validation for this statement. We compare the classical GNN methods: GCN~\cite{46}, GAT~\cite{47}, and GraphSAGE~\cite{45}, and the classical Transformer-based model: UGformer~\cite{102}. Specifically, we replace the core node embedding unit in NRL-GT with GCN, GAT, GraphSAGE, and UGformer for the ablation study. The training and testing datasets are from Section IV-B and Section IV-C. The average test errors are shown in Table X. For controllability robustness learning, the proposed GT algorithm has an absolute advantage over other models. For connectivity robustness learning, GCN on ER networks and GAT on the SF network can also get a comparable result with GT. In contrast, in the remaining cases, GT gets the best performance. It can be clearly seen that UGformer is disabled in our task due to the self-attention mechanism applied to all nodes of the graph ignoring the structure of the graph itself. The classical Transformer~\cite{80} also suffers from the same limitation.
\begin{table}[!h]
\centering
\caption{Results of ablation studies compared with classical GNN and Transformer-Based algorithms\label{tab:table2}}
\begin{tabular}{|c|c|c|c|c|c|}
\hline
\begin{tabular}[c]{@{}c@{}}Controllability \\ Robustness\end{tabular} & BA             & ER             & NW             & QSN            & SF             \\ \hline
GCN                                                                   & 0.029          & 0.022          & 0.024          & 0.020           & 0.034          \\ \hline
GAT                                                                   & 0.030           & 0.024          & 0.024          & 0.027          & 0.035          \\ \hline
GraphSAGE                                                             & 0.028          & 0.022          & 0.023          & 0.024          & 0.029          \\ \hline
UGformer                                                             & 0.109          & 0.131          & 0.135          & 0.097          & 0.332          \\ \hline
GT                                                                    & \textbf{0.019} & \textbf{0.015} & \textbf{0.014} & \textbf{0.014} & \textbf{0.020} \\ \hline
\begin{tabular}[c]{@{}c@{}}Connectivity \\ Robustness\end{tabular}    & BA             & ER             & NW             & QSN            & SF             \\ \hline
GCN                                                                   & 0.031          & \textbf{0.021} & 0.024          & 0.030          & 0.062          \\ \hline
GAT                                                                   & 0.031          & 0.024          & 0.029          & 0.029          & \textbf{0.056} \\ \hline
GraphSAGE                                                             & 0.033          & 0.031          & 0.035          & 0.029          & 0.062          \\ \hline
UGformer                                                             & 0.257          & 0.311          & 0.391          & 0.324          & 0.317          \\ \hline
GT                                                                    & \textbf{0.028} & \textbf{0.021} & \textbf{0.021} & \textbf{0.023} & \textbf{0.056} \\ \hline
\end{tabular}
\vspace{-0.4cm}
\end{table}
\vspace{-0.4cm}
\subsection{Learning Connectivity Robustness of Networks Against Batch Attacks}
As discussed in the introduction, LFR-CNN is a generalized network robustness learning framework, with which we did not compare in our previous experiments. The main reason is that LFR-CNN is dedicated to learning the robustness of complex networks against batch attacks. Specifically, for networks of different sizes, batch attack refers to attacking the same fixed proportion of nodes each time and eventually generating the same length of robustness sequences, which is different from the nature of our learning. In previous experiments, we mainly 
consider attacking one node at a time, where the length of the sequence is 
matched to the size of the network. However, this does not demonstrate that 
NRL-GT has limited ability to lean network robustness against batch attacks.

In this subsection, we specialize in studying the precision of NRL-GT in learning robustness of networks against batch attacks. We compare LFR-CNN by using the publicly available dataset and results$\footnote{https://fylou.github.io/sourcecode.html}$ from the paper for LFR-CNN. The overall average degree and range of network size in the publicly available dataset are set as 4.33 and [350,650]. Given that the publicly available well-trained model of LFR-CNN is only for network connectivity robustness, we experimentally focus on network connectivity robustness learning in this subsection. The experimental results are displayed in Table XI. It is clear to show the superiority of NRL-GT. When performing feature extraction for networks of different sizes, LFR-CNN loses information because of a node selection process, making it inferior to the NRL-GT with no information loss.
\begin{table}[!h]
\centering
\caption{Results of learning connectivity robustness of networks against batch attacks\label{tab:table2}}
\resizebox{\linewidth}{!}{\begin{tabular}{|cc|c|c|c|c|c|}
\hline
\multicolumn{2}{|c|}{Average Learning Error $\overline{\xi} $}                                                                                                                   & BA     &ER     & NW     & QSN    & SF     \\ \hline
\multicolumn{1}{|c|}{\multirow{2}{*}{\begin{tabular}[c]{@{}c@{}} Connectivity Roubustness \\ under batch RAs \end{tabular}}} & LFR-CNN & 0.0665 & 0.0362 & 0.0338 & 0.0350  & 0.0868 \\ \cline{2-7} 
\multicolumn{1}{|c|}{}                                                                                                                               & NRL-G  & \textbf{0.0339} & \textbf{0.0240}  & \textbf{0.0256} & \textbf{0.0241} & \textbf{0.0581} \\ \hline
\end{tabular}}
\end{table}
\vspace{-0.4cm}
\section{Conclusion}
In this paper, a novel versatile robustness and unified learning approach based on the customized graph transformer (NRL-GT) has been proposed. It accomplishes the task of network controllability robustness learning and connectivity robustness learning from three aspects, including robustness curve learning, overall robustness learning, and synthetic network classification. A large number of experiments prove that, compared with cutting-edge methods, NRL-GT has not only the superiority of speed and precision but also powerful generalization ability, which can guarantee high precision performance when the distribution of the training and test sets are inconsistent, as well as enabling the superior performance of real-world networks, especially circuit networks and power networks. Meanwhile, the NRL-GT provides a better indicator of connectivity robustness than spectral measures. In addition, it is verified from both experimental and theoretical aspects that the proposed GT is superior to the classical GNN and transformer-based algorithms. The backbone of NRL-GT can serve as a transferable feature learning module that can be directly applied to different downstream tasks and complex networks with different sizes without retraining. 

NRL-GT provides a relatively complete robustness learning framework for complex 
networks of pairwise interactions. However, it may not perform well on higher-order 
networks, such as hypergraphs, since the high-order structural information 
are not fully explored. Extending NRL-GT to a hypergraph neural 
networks-based~\cite{104} model to learn robustness of hypergraphs is 
a  promising direction for future work. Besides, it should be noted that degree-related features, which have been shown effective, are set as the initial input to NRL-GT. Exploring other useful features for network robustness learning is also a  promising but challenging topic for future research.

\bibliographystyle{IEEEtran}

\bibliography{main.bib}
%
%
%
%
%
%
%
%


\end{document}